\documentclass{article}

 \PassOptionsToPackage{numbers, compress}{natbib}

     \usepackage[preprint]{neurips_2022}

\usepackage[utf8]{inputenc} 
\usepackage[T1]{fontenc}    
\usepackage{hyperref}       
\usepackage{url}            
\usepackage{booktabs}       
\usepackage{amsfonts}       
\usepackage{nicefrac}       
\usepackage{microtype}      
\usepackage{xcolor}
\usepackage{amsthm}
\newtheorem{theorem}{Theorem}
\newtheorem{corollary}{Corollary}
\newtheorem{lemma}{Lemma}
\usepackage{amsmath}
\usepackage{graphicx}
\usepackage{multirow}
\usepackage{makecell,rotating}
\newtheorem{assumption}{Assumption}

\title{Why Adversarial Training of ReLU Networks \\Is Difficult?}

\author{%
  Xu Cheng\footnotemark[1]\thanks{Contribute equally to this paper.}\\
  Shanghai Jiao Tong University\\
  \texttt{xcheng8@sjtu.edu.cn}\\
  \And
  Hao Zhang\footnotemark[1]\\
  Shanghai Jiao Tong University\\
  \texttt{1603023-zh@sjtu.edu.cn}\\
   \And
  Yue Xin \\
  Shanghai Jiao Tong University\\
  \texttt{xinyuexiong@sjtu.edu.cn}\\
     \And
  Wen Shen\\
  Shanghai Jiao Tong University\\
    \texttt{wen\_shen@sjtu.edu.cn}\\
      \And
  Quanshi Zhang\thanks{Quanshi Zhang is the corresponding author. He is with the Department of Computer Science and Engineering,
the John Hopcroft Center and the MoE Key Lab of Artificial Intelligence, AI Institute, at the Shanghai Jiao Tong University, China.} \\
  Shanghai Jiao Tong University\\
  \texttt{zqs1022@sjtu.edu.cn} \\
}

\begin{document}

\maketitle

\begin{abstract}
This paper mathematically derives an analytic solution of the adversarial perturbation on a ReLU network, and theoretically explains the difficulty of adversarial training.
Specifically, we formulate the dynamics of the adversarial perturbation generated by the multi-step attack, which shows that the adversarial perturbation tends to strengthen eigenvectors corresponding to a few top-ranked eigenvalues of the Hessian matrix of the loss \textit{w.r.t.} the input.
We also prove that adversarial training tends to strengthen the influence of unconfident input samples with large gradient norms in an exponential manner.
Besides, we find that adversarial training strengthens the influence of the Hessian matrix of the loss~\textit{w.r.t.} network parameters, which makes the adversarial training more likely to oscillate along directions of a few samples, and boosts the difficulty of adversarial training.
Crucially, our proofs provide a unified explanation for previous findings in understanding adversarial training~\cite{liu2020loss,kanai2021smoothness,wu2020adversarial,yamada2021adversarial,athalye2018obfuscated,tsipras2018robustness,ilyas2019adversarial,liu2021impact,chen2020robust,rice2020overfitting}.
\end{abstract}

\section{Introduction}
Although deep neural networks (DNNs) have shown promise in different tasks, the DNN was usually fooled by specific imperceptible perturbations of the input data~\cite{goodfellow2014explaining,lecun2015deep}, which were termed  \textit{adversarial examples}.
To defend against adversarial examples, the most widely-used strategy is adversarial training~\cite{kurakin2016adversarial,madry2018towards}.
Despite the effectiveness of adversarial training, extensive experiments have shown that adversarial training is much more difficult to optimize than vanilla training.
Previous studies explained this phenomenon from different perspectives,
such as the sharp loss landscape~\cite{liu2020loss,kanai2021smoothness,wu2020adversarial,yamada2021adversarial},
obfuscated gradients~\cite{athalye2018obfuscated}, and inhomogeneous data distribution~\cite{sinha2017certifying,zhang2019defense,miyato2018virtual}.

Unlike previous research, this paper aims to derive an analytic solution to adversarial perturbations on a ReLU network, and further theoretically proves why adversarial training is difficult.
Specifically, we formulate the dynamics of the adversarial perturbation of the multi-step attack with an analytic solution, and further explain the effects of adversarial perturbations on adversarial training.
Hence, based on the above proofs, we obtain the following three conclusions.

(1) The adversarial perturbation strengthens eigenvectors corresponding to a few top-ranked eigenvalues of the Hessian matrix of the loss~\textit{w.r.t.} the input.

(2) Adversarial training mainly focuses on a few unconfident input samples with large gradient norms.
Furthermore, we prove that the normalization/regularization of perturbations in $\ell_2$ attacks and $\ell_\infty$ attacks alleviate such an imbalance.

(3) Adversarial training strengthens the influence of the Hessian matrix of the loss \textit{w.r.t.} network parameters.
Hence, adversarial training is more likely to make network parameters oscillate, which explains the difficulty of adversarial training, as well.

Note that there is a common trick in adversarial training that people usually learn a robust network on relatively weak adversarial perturbations~\cite{wong2020fast}, which often have not reached the constraint for perturbations.
In this way, we can ignore the constraint of adversarial perturbations.

More crucially, our theoretical proof also provides a theoretical foundation, which may explain various previous findings/understandings of adversarial training~\cite{liu2020loss,kanai2021smoothness,wu2020adversarial,yamada2021adversarial,athalye2018obfuscated,tsipras2018robustness,ilyas2019adversarial,liu2021impact,chen2020robust,rice2020overfitting}.

Contributions of this paper are summarized as follows.
(1) We derive an analytic solution that explains the dynamics of the adversarial perturbation.
(2) We prove that adversarial training strengthens the influence of a few input samples, and increases the likelihood of the oscillation of network parameters, which boosts the difficulty of adversarial training.
(3) Our proofs can explain the benefit of the normalization/regularization of perturbations in $\ell_2$ attacks and $\ell_\infty$ attacks, and can provide a unified view to understand a total of ten previous studies in adversarial training.

\section{Related work}

Some previous studies~\cite{liu2020loss,kanai2021smoothness,wu2020adversarial,yamada2021adversarial,yu2018interpreting} considered that the sharp loss landscape~\textit{w.r.t.} network parameters resulted in the difficulty of adversarial training.
\citet{kurakin2016adversarial} demonstrated that label leaking hindered adversarial training.
\citet{tsipras2018robustness} had proven that compared to vanilla training, adversarial training relied on robust features and did not use non-robust features for inference, which resulted in the inferior classification performance.
The gradient-masking phenomenon~\cite{papernot2017practical,athalye2018obfuscated,tramer2018ensemble} led to a false sense of security in defenses against adversarial examples.
\citet{ilyas2019adversarial} had proven that adversarial examples were attributed to the presence of highly predictive but non-robust features.
Some works~\cite{sinha2017certifying,zhang2019defense,miyato2018virtual} demonstrated adversarial examples generated in the supervised way usually corrupted the underlying data structure, which hindered adversarial training~\cite{qian2022perturbation}.

Crucially, it has been discovered that adversarial training usually has a more significant overfitting problem than vanilla training~\cite{rice2020overfitting}.
\citet{liu2021impact} had proven that the overfitting in adversarial training was caused by the model’s attempt to fit hard adversarial examples.
\citet{chen2020robust} considered that the model overfitted the attacks generated in the early stage of adversarial training, and failed to generalize to the attacks in the late stage.
\citet{stutz2020confidence} demonstrated that the overfitting in adversarial training was a result of enforcing high-confidence predictions on adversarial examples.
\citet{schmidt2018adversarially} and~\citet{zhai2019adversarially} considered that the significantly high adversarial data complexity made adversarial training difficult to achieve good generalization capacity.
\citet{rice2020overfitting} used early stopping to reduce overfitting in adversarial training.

Unlike previous studies, this paper analyzes the dynamics of adversarial perturbations, and theoretically explains the difficulty of adversarial training, based on the derived analytic solution.
More crucially, our proof can also provide a theoretical explanation for previous findings/understandings of adversarial training~\cite{liu2020loss,kanai2021smoothness,wu2020adversarial,yamada2021adversarial,athalye2018obfuscated,tsipras2018robustness,ilyas2019adversarial,liu2021impact,chen2020robust,rice2020overfitting} in Sec.~\ref{sec:4}.


\section{Explaining adversarial perturbations and adversarial training}
\label{sec:3}

Let us first revisit adversarial training.
Given a DNN $f_{\theta}$ parametrized by $\theta$ and an input sample $x \in \mathbb{R}^{n}$ with its true label $y$, the adversarial attack adds a human-imperceptible perturbation $\delta$ to fool the DNN with the adversarial example $x+\delta$, whose objective is usually formulated as follows.
\begin{equation}
\label{main_eqn:adv_attack}
\max_{\delta}   L(f_{\theta}(x+\delta), y) , \quad \text{ s.t. } \quad   \Vert \delta \Vert_{p} \leq \epsilon,
\end{equation}
where $f_{\theta}(x+\delta) $ denotes the network output, and $L(f_{\theta}(x+\delta ), y)$ represents the loss function.
$\epsilon$ is the constraint of the $ \ell_p$ norm of the adversarial perturbation.
To defend against adversarial attacks, adversarial training is often formulated as a min-max game~\cite{madry2018towards}.
\begin{equation}
\label{eqn:adv_train}
\min_{\theta} \; \mathbb{E}_{\{x,y\}} \big[ \max_{ \delta}   L(f_{\theta}(x+\delta), y) \big], \quad \text{ s.t. } \quad  \Vert \delta \Vert_{p} \leq \epsilon,
\end{equation}

\subsection{Analysis of adversarial perturbations}
\label{sec:3.1}

To analyze the dynamics of adversarial perturbations, let us consider the multi-step attack as follows, where $\delta^{(t)}$ is referred to as the perturbation generated after attacking for $t$ steps; $m$ represents the total number of steps; $\alpha$ denotes the step size.
\begin{equation}
\label{main_eqn:adv_per}
\delta^{(m)} =\sum\nolimits_{t=0}^{m-1}  \alpha \cdot  g_{x+ \delta^{(t)}}.
\end{equation}
To simplify the story, we first analyze the most straightforward solution to the multi-step adversarial attack, {\small$g_{x+ \delta^{(t)}} =\frac{\partial}{\partial x} L(f(x+ \delta^{(t)}), y)$}.
Then, we will extend the analysis to the widely-used $\ell_2$ attack and the $\ell_\infty$ attack~\cite{dong2018boosting,goodfellow2014explaining,madry2018towards}, where they regularize or normalize the gradient as
{\small$g_{x+ \delta^{(t)}}^{(\ell_2)}=\frac{\partial}{\partial x} L(f(x+ \delta^{(t)}), y)/\Vert \frac{\partial}{\partial x} L(f(x+ \delta^{(t)}), y) \Vert$}, and
{\small$g_{x+ \delta^{(t)}}^{(\ell_\infty)}=\text{sign} (\frac{\partial}{\partial x} L(f(x+ \delta^{(t)}), y))$}, respectively.

Without loss of generality, let us consider a ReLU network $f$ and an input sample $x$.
$z(x)$ denotes the input feature of the top layer (\textit{e.g.} a softmax layer
{\small$f(x)=\textrm{softmax}(z(x))$}, or a sigmoid layer {\small$f(x)=\textrm{sigmoid}(z(x))$}).
The following equation formulates how the network uses the feature $h$ of the $j$-th linear layer to compute $z(x)$.
\begin{equation}
z(x) = W_{l}^{T}( \dots \Sigma_{j+1}(W_{j+1}^{T}\Sigma_j h +b_{j+1}) \dots) +b_{l},
\end{equation}
where {\small$h=W^{T}_{j} x' +b_{j}$} denotes the linear transformation in the $j$-th layer, subject to {\small$x'=\Sigma_{j-1}(W_{j-1}^{T}(\dots \Sigma_1(W^{T}_{1}x+b_1)\dots)+b_{j-1})$}.
$W_{j}$ and $b_j$ denote the weight and bias of the $j$-th linear layer, respectively.
The matrix {\small$\Sigma_j=diag(\sigma_{j,1},\sigma_{j,2},\dots,\sigma_{j,D}) \in\mathbb{R}^{D \times D}$} represents gating states of the $j$-th
gating layer (\textit{e.g.} a ReLU layer, or a MaxPooling layer), {\small$\sigma_{j,d} \in\{0,1\}$}.
For simplicity, we use  {\small$W= W_{j}$} and  {\small$b= b_{j}$} in the following manuscript without causing ambiguity, thereby {\small$h=W^{T} x' +b$}.

\begin{assumption}
\label{assumption:relu}

Because the change of gating states in multi-step attacks is usually chaotic and unpredictable for analysis, we simplify our research into an idealized adversarial attack, whose adversarial perturbation does not significantly change gating states in gating layers.
In this scenario, we approximate the ReLU network $f$ to a linear model,~\textit{i.e.,} {\small$z(x) \approx \boldsymbol{w} x + \boldsymbol{b}$}, {\small$\boldsymbol{w}=W_{l}^{T}\Sigma_{l-1}\cdots\Sigma_{2}W_2^{T}\Sigma_{1}W_1^{T}$}, to simplify the proof.
\end{assumption}

We have conducted experiments to verify the above assumption.
Experimental results in supplementary material shows that the analytic solution to adversarial perturbations in Theorem~\ref{theorem:adv_per} and Theorem~\ref{theorem:adv_per_2} derived on Assumption~\ref{assumption:relu} well match real perturbations.
It is because the change of gating states is usually unpredictable.
Thus, the analysis under Assumption~\ref{assumption:relu} can still reflect the properties of a ReLU network, to some extent.

Before the later analysis of the $\ell_2$ attack and the $\ell_\infty$ attack, we first focus on the original form of the multi-step attack,~\textit{i.e.} perturbation generated via
{\small$g_{x+ \delta^{(t)}} =\frac{\partial}{\partial x} L(f(x+ \delta^{(t)}), y)$}.
Note that there exists a common trick in adversarial training,~\textit{i.e.,} people usually generate weak adversarial examples for training~\cite{wong2020fast}, which often have not reached the constraint
{\small$\Vert \delta \Vert_{p} < \epsilon$}.
In this way, we can ignore the constraint of perturbations, which does not significantly hurt the trustworthiness of the analysis.

\begin{theorem}
\label{theorem:adv_per}
(Dynamics of perturbations of the m-step attack).
Based on Assumption~\ref{assumption:relu},
the adversarial perturbation {\small$\delta^{(m)}$} is given as follows,
where {\small$\lambda_{i}$} and {\small$v_{i}$} denote the $i$-th largest eigenvalue of the Hessian matrix {\small$H_x= \frac{\partial^{2}}{\partial x \partial x^{T}} L(f(x),y)$} and its corresponding eigenvector, respectively.
{\small$\gamma_i = g_x^{T} v_{i} \in \mathbb{R}$} represents the projection of the gradient {\small$g_x=\frac{\partial}{\partial x} L(f(x),y)$} on the eigenvector {\small$v_{i}$}.
{\small$\mathcal{R}_2(\alpha)$} denote the sum of terms no less than the second order in Taylor expansion~\textit{w.r.t.} {\small$\delta^{(m)}$}.
\vspace{-8pt}

\begin{small}
\begin{equation}
\label{main_eqn:solution_adv_per}
\delta^{(m)} =  \sum_{i=1}^{n}  \frac{(1+\alpha \lambda_{i} ) ^{m}-1}{\lambda_{i}} \gamma_{i} v_{i} + \mathcal{R}_2(\alpha),\quad
g_{x+\delta^{(m)}} = \sum_{i=1}^{n} (1+\alpha \lambda_{i} ) ^{m} \gamma_{i} v_{i}.
\end{equation}\end{small}
\end{theorem}

Additionally, we can extend the $m$-step attack in Theorem~\ref{theorem:adv_per} to a more idealized case,~\textit{i.e.,} the infinite-step attack with the infinitesimal step size, as follows.

\begin{theorem}
\label{theorem:adv_per_2}
(Perturbations of the infinite-step attack).
{\small$\beta = \alpha m$} reflects the overall adversarial strength of the infinite-step attack with the step number {\small$m \to + \infty$} and the step size {\small$\alpha=\beta/m \to 0$}.
Then, this infinite-step adversarial perturbation {\small$\hat\delta = \lim_{m\to+ \infty }\alpha \sum\nolimits_{t=0}^{m-1} \frac{\partial}{\partial x} L(f(x+ \delta^{(t)}), y)$} can be re-written as follows, where {\small$\hat{\mathcal{R}}_2(\alpha)$} denote the sum of terms no less than the second order in Taylor expansion~\textit{w.r.t.} {\small$\hat\delta$}.
\vspace{-7pt}

\begin{small}
\begin{equation}
\label{main_eqn:solution_adv_per_2}
\hat\delta =  \sum_{i=1}^{n}  \frac{\exp(\beta \lambda_{i})-1}{\lambda_{i}} \gamma_{i} v_{i} +\hat{\mathcal{R}}_2(\alpha),\quad
g_{x+\hat\delta} = \sum_{i=1}^{n} \exp(\beta \lambda_{i}) \gamma_{i} v_{i}.
\end{equation}\end{small}
\end{theorem}


Theorem~\ref{theorem:adv_per} and Theorem~\ref{theorem:adv_per_2} are proven in the supplementary with the following two conclusions.\\
\textbf{(C. 1)} The adversarial perturbation strengthens gradient components in $g_x$ along eigenvectors corresponding to a few top-ranked eigenvalues {\small$\lambda_{i}$} of the Hessian matrix {\small$H_x$} exponentially.
Furthermore, a larger adversarial strength $\beta$, such as attacking for more steps, is more likely to force the perturbation to change along fewer top-ranked eigenvectors.\\
\textbf{(C. 2)}
Both the gradient norm  {\small$\Vert g_{x+\hat\delta} \Vert$}~\textit{w.r.t.} the adversarial perturbation,
and the perturbation norm {\small$\Vert\hat\delta\Vert$} increase along with the step number $m$ exponentially.

Note that different parameter settings of multi-step attacks (such as the step size or the step number) may make slight differences on the generation of adversarial perturbations.
Thus, to remove side effects of such settings and simplify the story, in the following manuscript, we use the idealized case of the infinite-step attack in Theorem~\ref{theorem:adv_per_2}
to analyze adversarial training, without loss of generality.


$\bullet\quad$\textit{Experimental verification 1 of Theorem~\ref{theorem:adv_per_2}}.
Here, we examined whether the norm of the gradient {\small$\Vert g_{x+\hat\delta} \Vert$}~\textit{w.r.t.} the adversarial perturbation, and the norm of the adversarial perturbation {\small$\Vert\hat\delta\Vert$} increased with the step number $m$ in an approximately exponential manner.
Specifically, we generated perturbations {\small$\hat\delta$} in Theorem~\ref{theorem:adv_per_2} based on
VGG-11~\cite{simonyan2014very}, AlexNet~\cite{krizhevsky2012imagenet}, and ResNet-18~\cite{he2016deep}, which were learned on the MNIST dataset~\cite{lecun1998gradient}, respectively.
Then, the perturbation {\small$\hat\delta$} was crafted by the gradient {\small$g_{x+ \hat\delta^{(t)}} =\frac{\partial}{\partial x} L(f(x+ \hat\delta^{(t)}), y)$}.
Besides, we also generated two baseline perturbations via the $\ell_2$ attack and the $\ell_\infty$ attack for comparison,~\textit{i.e.,} applying {\small$g_{x+ \delta^{(t)}}^{(\ell_2)}$}, and {\small$g_{x+ \delta^{(t)}}^{(\ell_\infty)}$} defined under Eq.~\eqref{main_eqn:adv_per}.
Because the step number {\small$m_\text{success}$} that successfully attacked was different for each sample, in Fig.~\ref{delta}, we used {\small$m/m_\text{success}$} as the horizontal axis to measure the rate of progress in the adversarial attack.
Fig.~\ref{delta} shows that both the gradient norm {\small$\Vert g_{x+\hat\delta} \Vert$}, and the perturbation norm {\small$\Vert\hat\delta\Vert$} increased exponentially with the step number, which verified Theorem~\ref{theorem:adv_per_2}.
Please see supplementary materials for more details.

$\bullet\quad$\textit{Experimental verification 2 of Theorem~\ref{theorem:adv_per_2}}.
Unlike the above experiment, in this experiment, we checked whether the solution {\small$\hat\delta$} derived in Theorem~\ref{theorem:adv_per_2} well fit the real perturbation measured in practice.
Specifically, we generated adversarial perturbations on sixteen different ReLU networks, including various MLPs, CNNs, and ResNets.
Experimental results show that for each network, the theoretically derived perturbation {\small$\hat\delta$} well fit the real one, which successfully verified Theorem~\ref{theorem:adv_per_2}.
Please see supplementary materials for details.

\begin{figure}[t]
	\centering
	\includegraphics[width=\linewidth]{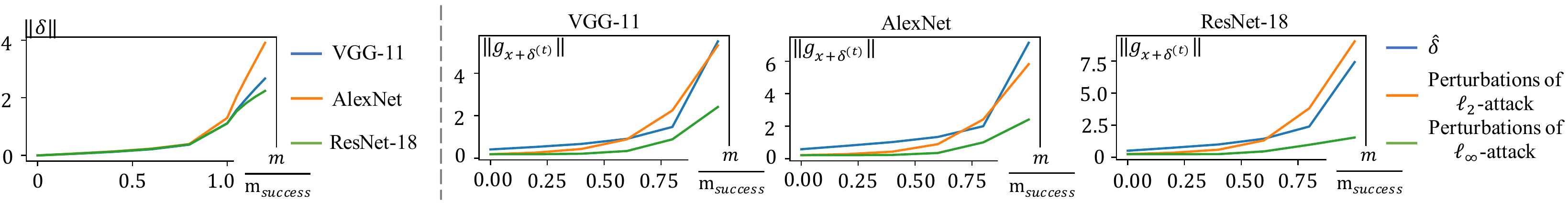}
	\vspace{-15pt}
	\caption{Exponential increases of perturbation norms {\small$\Vert\hat\delta\Vert$} and gradient norms {\small$\Vert g_{x+ \hat\delta^{(t)}}\Vert$} with the step number {\small$m$} of the adversarial attack.}
	\label{delta}
	\vspace{-8pt}
\end{figure}


\textbf{For $\ell_2$ attacks and $\ell_\infty$ attacks.}
As two typical attacking methods, the $\ell_2$ attack and the $\ell_\infty$ attack usually regularize/normalize the adversarial strength in each step by applying {\small$g_{x+ \delta^{(t)}}^{(\ell_2)}=\frac{\partial}{\partial x} L(f(x+ \delta^{(t)}), y)/\Vert \frac{\partial}{\partial x} L(f(x+ \delta^{(t)}), y) \Vert$},
and {\small$g_{x+ \delta^{(t)}}^{(\ell_\infty)}=\text{sign} (\frac{\partial}{\partial x} L(f(x+ \delta^{(t)}), y))$}, respectively.
In fact, for the $\ell_\infty$ attack, we can roughly consider that only the gradient component {\small$o_{x}^{T}g_{x+ \delta^{(t)}}^{(\ell_\infty)} o_{x}$} disentangled from {\small$g_{x+ \delta^{(t)}}^{(\ell_\infty)}$} along {\small$\frac{\partial}{\partial x} L(f(x), y)$} is effective,
where {\small$o_{x}=\frac{\partial}{\partial x} L(f(x), y)/ \Vert \frac{\partial}{\partial x} L(f(x), y) \Vert$} is the unit vector in the direction of {\small$\frac{\partial}{\partial x} L(f(x), y)$}.
However, it is quite complex to analyze the exact attacking behavior.
Therefore, in Corollary~\ref{theorem:adv_per_norm}, we just normalize the perturbation in Theorem~\ref{theorem:adv_per_2} to roughly approximate the regularization/normalization of perturbations in  $\ell_2$ attacks and $\ell_\infty$ attacks.


\begin{corollary}
\label{theorem:adv_per_norm}
(Normalized perturbation of the infinite-step attack).
Based on Theorem~\ref{theorem:adv_per_2}, if we ignore residual terms {\small$R_2(\hat\delta)$},
the perturbation of the infinite-step $\ell_2$ attack generated via  {\small$g_{x+ \delta^{(t)}}^{(\ell_2)}$}, and the perturbation of the infinite-step $\ell_\infty$ attack generated via  {\small$g_{x+ \delta^{(t)}}^{(\ell_\infty)}$} can be approximated as follows.
Here, $C \in \mathbb{R}$ is roughly considered to be proportional to the step number $m$ of the adversarial attack, which is explained in supplementary materials.
\vspace{-6pt}

\begin{small}
\begin{equation}
\label{main_eqn:solution_adv_per_norm}
\hat\delta^{\text{(norm)}} \approx  C \cdot \hat\delta/{\Vert \hat\delta \Vert}
=  C \cdot \sum\nolimits_{i=1}^{n}  \frac{\exp(\beta \lambda_{i})-1}{\lambda_{i}} \gamma_{i} v_{i} \Bigg/
{ \sqrt {\sum\nolimits_{i=1}^{n}  (\frac{\exp(\beta \lambda_{i})-1}{\lambda_{i}}\gamma_{i})^{2}}}.
\end{equation}\end{small}
\end{corollary}

\textbf{(C. 3)}
Corollary~\ref{theorem:adv_per_norm} reveals that a weak adversarial strength {\small$\beta$} makes the normalized perturbation {\small$\hat\delta^{\text{(norm)}}$} approximately parallel to the gradient {\small$g_x$}.
Whereas, a large adversarial strength makes the normalized perturbation {\small$\hat\delta^{\text{(norm)}}$} approximately parallel to the eigenvector {\small$v_1$}~\textit{w.r.t.} the largest eigenvalue.

\subsection{Explaining the difficulty of adversarial training}
\label{sec:3.2}

In this subsection, we explain the effects of adversarial perturbations on weight optimization in adversarial training.
Without loss of generality, we analyze the learning dynamics of the $j$-th linear layer of the ReLU network $f$.
Specifically, if we use vanilla training to fine-tune the network on the original input sample $x$ for a single step, then the gradient of the loss~\textit{w.r.t.} the weight of the $j$-th layer {\small$W=W_j$} is given as {\small$g_{W}= \frac{\partial}{\partial W} L(f(x),y)$}.
In comparison, if we train the network on the adversarial example {\small$x+\hat\delta$} for a single step, then we will get the gradient {\small$g^{\text{(adv)}}_{W} =\frac{\partial}{\partial W}  L(f(x+\hat\delta),y)$}.
In this way, {\small $\Delta g_{W} = g^{\text{(adv)}}_{W}- g_{W}$} denotes additional effects of adversarial training on the gradient.
\begin{equation}
\Delta g_{W} = g^{\text{(adv)}}_{W}- g_{W}= \frac{\partial}{\partial W} L(f(x+\hat\delta),y) -\frac{\partial}{\partial W}  L(f(x),y).
\end{equation}
Similarly, {\small $ \Delta g_{W}^{\text{(norm)}} =  g^{\text{(adv,norm)}}_{W}- g_{W}$} represents additional effects on the gradient brought by adversarial training, when we use the normalized perturbation {\small$\hat\delta^{\text{(norm)}}$} in Corollary~\ref{theorem:adv_per_norm} (related to the $\ell_2$ attack and the $\ell_\infty$ attack).
\begin{equation}
\Delta g_{W}^{\text{(norm)}} = g^{\text{(adv,norm)}}_{W}- g_{W}= \frac{\partial}{\partial W} L(f(x+\hat\delta^{\text{(norm)}}),y) - \frac{\partial}{\partial W}  L(f(x),y).
\end{equation}

\begin{assumption}
(proven in the supplementary materials).
\label{assumption:sigmoid}
The analysis of binary classification based on a sigmoid function, {\small$f(x)= \frac{1}{1+exp(-z(x))}, z(x) \in \mathbb{R}$},
can also explain the multi-category classification with a softmax function, {\small$f(x) = \frac{exp(z’_{1})}{\sum_{i=1}^{c} exp(z’_{i})}, z’ \in \mathbb{R}^{c}$}, if the second-best category is much stronger than other categories.
In this case, attacks on the multi-category classification can be approximated by attacks on the binary classification between the best and the second-best categories,~\textit{i.e.,} {\small$f(x) \approx \frac{1}{1+exp(-z)}$}, subject to {\small$z=z'_1-z'_2\in \mathbb{R}$}.
$z'_1$ and $z'_{2}$ are referred to as network outputs corresponding to the best category and the second-best category, respectively.
\end{assumption}

\begin{lemma}
(proven in the supplementary materials).
\label{lemma:hessian}
Let us focus on the cross-entropy loss {\small$L(f(x),y)$}.
If the classification is based on a softmax operation, then the Hessian matrix {\small$H_z =\frac{\partial^{2}}{\partial z \partial z^{T}}   L(f(x),y)$} is positive semi-definite.
If the classification is based on a sigmoid operation, the scalar {\small$H_z \geq g_z^{2} \geq 0$}, as long as the attacking has not finished (still {\small$z(x) \cdot y>0, y \in \{-1,+1\}$}).
Here, {\small$g_z = \frac{\partial}{\partial z} L(f(x),y)  \in \mathbb{R}$}.
\end{lemma}


Theorems~\ref{theorem2} and~\ref{theorem3} explain training effects of the perturbation {\small$\hat\delta$} in Theorem~\ref{theorem:adv_per_2} on adversarial training.

\begin{theorem}
\label{theorem2}
(proven in the supplementary materials).
Based on Assumptions~\ref{assumption:relu} and \ref{assumption:sigmoid}, the effect of the adversarial perturbation {\small$\hat\delta$} in Eq.~\eqref{main_eqn:solution_adv_per_2} on the change of the gradient {\small$\tilde{g}_x = \frac{\partial z(x)}{\partial x}$} is formulated as follows.
{\small$\Delta \tilde{g}_x = - \eta \Delta g_{W} \tilde{g}_h$} represents the additional effects of adversarial training on changing {\small$\tilde{g}_x$}, because
adversarial training makes an additional change {\small$- \eta \Delta g_{W}$} on {\small${W}$}\footnote[1]{It is because adversarial training changes {\small${W}$} by {\small$- \eta g_{W}^{\text{(adv)}}$}, and vanilla training changes {\small${W}$} by  {\small$- \eta g_{W}$}, {\small$\eta>0$}.}.
In this way, {\small$\tilde{g}_x ^{T} \Delta \tilde{g}_x$} measures the significance of such additional changes along the direction of the gradient {\small$\tilde{g}_x$}.

\begin{small}
\begin{equation}
\label{main_eqn:adv_train}
\tilde{g}_x ^{T} \Delta \tilde{g}_x
=-\eta\tilde{g}_x ^{T} \Delta g_{W} \tilde{g}_h
= (e^{\mathcal{A}} -1) \tilde{g}_x ^{T}  \Delta \tilde{g}_x^{\text{(ori)}}
- \frac{ \eta g_z^{2} \;\Vert \tilde{g}_h \Vert ^{2} }{H_z}  (e^{2\mathcal{A}} -e^{\mathcal{A}}),
\end{equation}\end{small}
where {\small$\tilde{g}_h = \frac{\partial z(x)}{\partial h}$}, {\small$\mathcal{A} = \beta H_z \Vert \tilde{g}_x \Vert ^{2} \in \mathbb{R}$}, and $\eta$ denotes the learning rate to update the weight.
Considering the footnote\footnotemark[1], {\small$\Delta \tilde{g}_x^{\text{(ori)}} =- \eta g_{W} \tilde{g}_h$} measures the effects of vanilla training on changing {\small$\tilde{g}_x$} in the current back-propagation.
\end{theorem}

\begin{theorem}
\label{theorem3}
(proven in the supplementary materials).
Based on Assumptions~\ref{assumption:relu} and~\ref{assumption:sigmoid}, we derived the following equation~\textit{w.r.t.} adversarial training based on perturbations {\small$\hat\delta$} in Theorem~\ref{theorem:adv_per_2}.
Considering the footnote\footnotemark[1], {\small$\Delta \tilde{g}_x^{\text{(adv)}} =-\eta g_{W}^{\text{(adv)}} \tilde{g}_h$} reflects effects of adversarial training on changing the gradient {\small$\tilde{g}_x$}.
In this way, {\small$\tilde{g}_x ^{T}  \Delta \tilde{g}_x^{\text{(adv)}}$} represents the significance of such effects along the direction of the gradient {\small$\tilde{g}_x$}.

\begin{small}
\begin{equation}
\label{main_eqn:adv_train_1}
\tilde{g}_x ^{T}  \Delta \tilde{g}_x^{\text{(adv)}}
= -\eta\tilde{g}_x ^{T} g_{W}^{\text{(adv)}} \tilde{g}_h
= e^{ \mathcal{A}} \tilde{g}_x ^{T} \Delta \tilde{g}_x^{\text{(ori)}}
- \frac{ \eta g_z^{2}  (e^{2 \mathcal{A}} -e^{ \mathcal{A}})}{H_z}  \Vert \tilde{g}_h \Vert ^{2}.
\end{equation}
\end{small}
\end{theorem}

A common understanding of adversarial training is to alleviate the current gradient {\small$g_x$},
\textit{i.e.,} having a trend towards {\small${g}_x ^{T} \Delta \tilde{g}_x <0$}, so as to boost the adversarial robustness.
In this scenario, Theorem~\ref{theorem2} and Theorem~\ref{theorem3} reveal the following two conclusions.\\
\textbf{(C. 4)}
Adversarial training usually has a potential of decreasing the significance of the current gradient,~\textit{i.e.,} pushing {\small$ \tilde{g}_x ^{T} \Delta \tilde{g}_x $} and {\small$ \tilde{g}_x ^{T}  \Delta \tilde{g}_x^{\text{(adv)}}$} towards negative values.
It is because the second term in Eq.~\eqref{main_eqn:adv_train} and Eq.~\eqref{main_eqn:adv_train_1} is non-positive, due to {\small$H_z > 0$} in Lemma~\ref{lemma:hessian}.
More crucially, if vanilla training has already alleviated the current gradient {\small$g_x$}
(\textit{i.e.,} {\small$ \tilde{g}_x ^{T} \Delta \tilde{g}_x^{\text{(ori)}}<0$}),
then adversarial training will further strengthen such an alleviation in an exponential manner.
\\
\textbf{(C. 5)}
Adversarial training exponentially strengthens the influence of a few unconfident input samples with large values of {\small$H_z \in \mathbb{R}$} and large gradient norms {\small$ \Vert \tilde{g}_x \Vert $}.
Such mechanisms make the adversarial training more likely to oscillate in directions of a few samples (cf. Theorem~\ref{theorem:oscillation}), which boosts the difficulty of adversarial training, as well.


$\bullet\quad$\textit{Experimental verification 1 of Theorem~\ref{theorem2}.}
For verification, we conducted experiments to examine whether the derived training effect well fit the real effect.
Specifically, we calculated the metric {\small$\kappa = {\mathbb{E}_{x}[ \| \phi^{*} - \hat\phi \|}/{ \|\phi^{*}\|}]$} to evaluate the fitting between the theoretical derivation {\small$\hat\phi$} computed using the right side of Eq.~\eqref{main_eqn:adv_train} and {\small$\phi^{*}=\tilde{g}_x ^{T} \Delta \tilde{g}_x$} measured in experiments,
where {\small$\phi^{*}$} was computed using real measurements of {\small$\tilde{g}_x, \eta, g_{W}^{\text{(adv)}}, g_{W}$}, and {\small$\tilde{g}_h$} on a ReLU network.
To this end, we learned four types of ReLU networks on the MNIST dataset via adversarial training.
We followed settings in~\cite{ren2022towards} to construct five MLPs, five CNNs, three MLPs with skip connections, and three CNNs with skip connections, respectively.
Experimental results show that the average {\small$\kappa$} over all sixteen networks was {\small$0.097$}, which verified the correctness of Theorem~\ref{theorem2}.
Please see supplementary materials for detailed experimental settings and {\small$\kappa$} values for different networks.


$\bullet\quad$\textit{Experimental verification 2 of Theorem~\ref{theorem2}.}
Unlike \textit{Experimental verification 1 of Theorem~\ref{theorem2}}, in this experiment, we examined whether input samples with large {\small$H_z$}, large  {\small$H_z\Vert\tilde{g}_{x}\Vert^2$} values, and large {\small$\mathcal{A}$} values had large impacts {\small$|\tilde{g}_x ^{T} \Delta \tilde{g}_x|$} and {\small$ \Vert \Delta g_{W}^{\text{(adv)}} \Vert$},
\textit{i.e.,} whether adversarial training boosted the influence of such samples.
Note that in real applications, the {\small$\mathcal{A}$} value changed in each step of the adversarial attack, because the step-wise perturbation sometimes changed
the Hessian matrix {\small$H_z$} and the gradient {\small$\tilde{g}_{x}$}.
Thus, to be precise, we estimated the real {\small$\mathcal{A}$} value in Theorem~\ref{theorem2} as
{\small $\hat{\mathcal{A}}=\sum\nolimits_{t=1}^{m}\alpha H_z^{(t)}\Vert\tilde{g}_{x+\hat\delta^{(t)}}\Vert^2$}, subject to {\small$\tilde{g}_{x+\hat\delta^{(t)}}= \frac{\partial}{\partial x} z(x+\hat\delta^{(t)})$} and {\small$H_z^{(t)}= \frac{\partial^{2}}{\partial z \partial z^{T}} L(f(x+\hat\delta^{(t)}),y)$}.
To this end, we learned AlexNet and VGG-11 on the MNIST dataset via adversarial training on PGD, respectively.
Fig.~\ref{adv_gw} shows that input samples with larger values of {\small$H_z$},  {\small$H_z\Vert\tilde{g}_{x}\Vert^2$}, and {\small $\hat{\mathcal{A}}$} usually yielded larger {\small$\Vert \tilde{g}_x ^{T} \Delta \tilde{g}_x \Vert$} and {\small$\Vert \Delta g_{W}^{\text{(adv)}} \Vert$} values, which indicated that adversarial training strengthened the influence of these samples.
Thus, the conclusion \textbf{C. 5} was verified.
Please see supplementary materials for more details.

\begin{figure}[t]
	\centering
	\includegraphics[width=\linewidth]{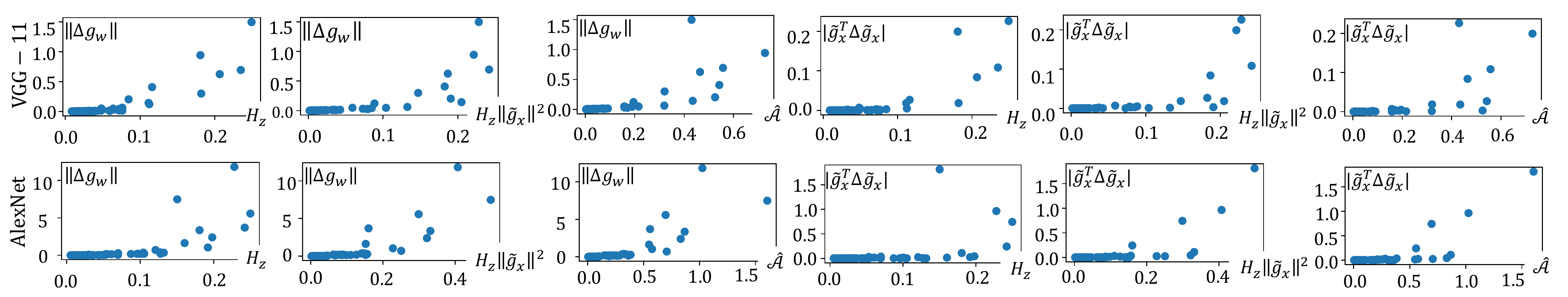}
	\vspace{-15pt}
	\caption{Impacts {\small$\Vert \Delta g_{W} \Vert$} and {\small$|\tilde{g}_x ^{T} \Delta \tilde{g}_x|$} of different input samples on adversarial training. Adversarial training boosts the influence of input samples with large  {\small$H_z$}, {\small$H_z \Vert \tilde{g}_x \Vert ^{2}$}, and {\small$\hat{\mathcal{A}}$} values.}
	\label{adv_gw}
	\vspace{-3pt}
\end{figure}

$\bullet\quad$\textit{Experimental verification 3 of Theorem~\ref{theorem2}.}
Here, we examined whether the optimization direction of adversarial training was dominated by a few input samples with large  {\small$\mathcal{A} = \beta H_z \Vert \tilde{g}_x \Vert ^{2}$} values.
Specifically,
let {\small$ \Delta g_{W} =g_{W}^{\text{(adv)}}-g_{W} $} denote the additional effect of adversarial training on a specific sample $x$ beyond vanilla training.
Then, based on the adversarially trained networks in \textit{experimental verification 2 of Theorem~\ref{theorem2}},
we measured the cosine similarity {\small$\cos( \Delta g_{W},  \Delta\overline{ g}_{W})$} between the training effect {\small$ \Delta g_{W} $} on a single adversarial example and the average effect {\small$\Delta\overline{g}_{W} =\mathbb{E}_{x+\hat\delta}[\Delta g_{W}]$} over different adversarial examples.
Fig.~\ref{cos} shows that the direction of the average effect  {\small$ \Delta\overline{g}_{W} $} was similar to (dominated by) training effects of a few input samples with large {\small$\hat{\mathcal{A}}$} values (the real {\small$\mathcal{A}$} calculated in experiments), which verified Theorem~\ref{theorem2}.
Please see supplementary materials for more details.

\begin{figure}[t]
	\centering
	\includegraphics[width=0.8\linewidth]{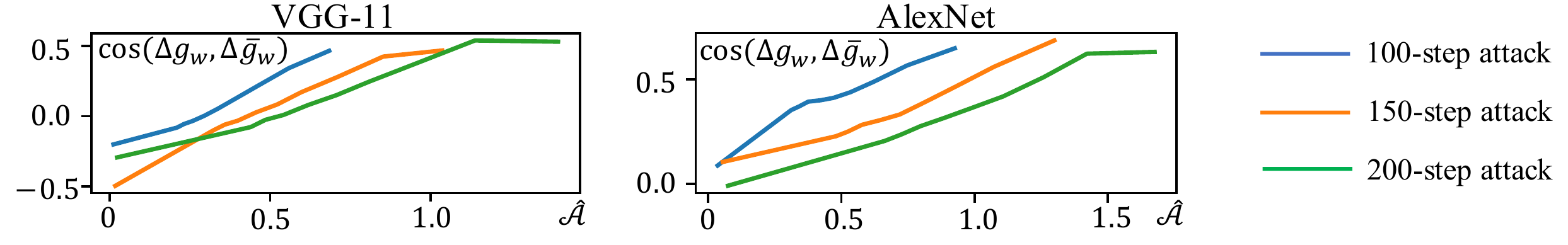}
	\vspace{-5pt}
	\caption{Average cosine similarity {\small$\mathbb{E}_{x}  [\cos(\Delta g_W|_x,  \Delta\overline{ g}_{W})]$} between {\small$ \Delta\overline{g}_{W} $} and each sample $x$ with a specific {\small$\hat{\mathcal{A}}$} value. {\small$ \Delta\overline{g}_{W} $} is similar to the direction of {\small$ \Delta g_{W}$}~\textit{w.r.t.} samples with large {\small$\hat{\mathcal{A}}$} values.}
	\label{cos}
	\vspace{-8pt}
\end{figure}


\textbf{Effects of normalized perturbations.}
As aforementioned, the $\ell_2$ attack and the $\ell_\infty$ attack can be roughly considered as the regularization/normalization of adversarial perturbations.
In this way, we analyze the effects of the normalized perturbation {\small$\hat\delta^{\text{(norm)}}$} on adversarial training,
which approximately explains adversarial training based on perturbations of the $\ell_2$ attack and the $\ell_\infty$ attack.

\begin{theorem}
\label{theorem4}
(proven in the supplementary materials).
Based on Assumptions~\ref{assumption:relu} and~\ref{assumption:sigmoid}, we derived the following equation~\textit{w.r.t.} adversarial training based on normalized perturbations {\small$\hat\delta^{\text{(norm)}}$}  in Corollary~\ref{theorem:adv_per_norm}.
Considering the footnote\footnotemark[1], {\small$\Delta \tilde{g}_x^{\text{(norm)}}=- \eta \Delta g_{W}^{\text{(norm)}} \tilde{g}_h = - \eta( g_{W}^{\text{(adv, norm)}}- g_{W}) \tilde{g}_h$}
represents additional effects of adversarial training on changing {\small$\tilde{g}_x$}.
In this way, {\small $\tilde{g}_x ^{T} \Delta \tilde{g}_x^{\text{(norm)}} =- \eta \tilde{g}_x ^{T}\Delta g_{W}^{\text{(norm)}} \tilde{g}_h$} reflects the significance of such additional effects along the direction of the gradient {\small$ \tilde{g}_x$}.
\vspace{-8pt}

\begin{small}
\begin{equation}
\label{main_eqn:norm_adv_train}
\tilde{g}_x ^{T} \Delta \tilde{g}_x^{\text{(norm)}}
= C \cdot \big(\frac{e^{\mathcal{A}}}{\Vert \hat\delta \Vert} -\frac{1}{\Vert \hat\delta \Vert}\big) \tilde{g}_x ^{T}  \Delta \tilde{g}_x^{\text{(ori)}}
- C \cdot \frac{   \eta g_z^{2}\; \Vert \tilde{g}_h \Vert ^{2} }{H_z }
\bigg(\frac{e^{\mathcal{A}}}{\Vert \hat\delta \Vert} -\frac{1}{\Vert \hat\delta \Vert} + C \cdot (\frac{e^{\mathcal{A}}}{\Vert \hat\delta \Vert} -\frac{1}{\Vert \hat\delta \Vert})^{2} \bigg).
\end{equation}
\end{small}
\end{theorem}

It is because Theorem~\ref{theorem:adv_per_2} shows that an extremely weak adversarial strength {\small$\beta \to 0$} usually yields {\small$\Vert \hat\delta \Vert \to \Vert g_x \Vert$}, and a
relatively strong adversarial strength {\small$\beta$} usually makes {\small$\Vert \hat\delta \Vert \to \exp(\beta  \Vert \tilde{g}_x \Vert ^{2} g_z^{2} )/{ \Vert {g}_x \Vert}$} with an exponential strength.
In this way, given a relatively strong attack, we can ignore the term {\small$1/ \Vert \hat\delta \Vert \to 0$} in Eq.~\eqref{main_eqn:norm_adv_train}, and prove that the strength of
the training effect {\small$\tilde{g}_x ^{T} \Delta \tilde{g}_x^{\text{(norm)}}$} is mainly determined by the term {\small$\exp(\mathcal{A})/\Vert \hat\delta \Vert \approx \Vert {g}_x \Vert \cdot \exp( \beta  \Vert \tilde{g}_x \Vert ^{2} (H_z-g_z^{2}))$}.
Please see supplementary materials for the proof.
Besides, according to Lemma~\ref{lemma:hessian}, as long as the attack has not succeeded yet, we have {\small$H_z-g_z^{2} > 0$},
but for too confident samples {\small$z(x) \to \infty$} or too unconfident samples {\small$z(x) = 0$}, we get {\small$H_z-g_z^{2} = 0$}.
Hence, we obtain the following two conclusions.\\
\textbf{(C. 6)}
Adversarial training on the normalized perturbations strengthens the influence of a few input samples with large gradient norms {\small$ \Vert \tilde{g}_x \Vert $}, which are neither too confident nor too unconfident.\\
\textbf{(C. 7)}
Compared to Theorems~\ref{theorem2} and~\ref{theorem3}, the normalized perturbation {\small$\hat\delta^{\text{(norm)}}$} in Eq.~\eqref{main_eqn:solution_adv_per_norm} alleviates the imbalance between different samples, which proves the benefits of $\ell_2$ attacks and $\ell_\infty$ attacks.

\textbf{Oscillation of network parameters.}
Above proofs can explain that adversarial training makes network parameters oscillate in very few directions, which is considered as a common phenomenon in adversarial training.
Such an explanation is based on a typical claim in optimization~\cite{cohen2021gradient,wu2018sgd} that if the largest eigenvalue of the Hessian matrix of the loss~\textit{w.r.t} network parameters is large enough, network parameters will oscillate along the eigenvector corresponding to the largest eigenvalue.

Here, although we do not directly prove that adversarial training can boost the largest eigenvalue of the Hessian matrix of the loss~\textit{w.r.t} network parameters,
Theorems~\ref{theorem:adv_per} and~\ref{theorem:adv_per_2} show that training on adversarial examples is somewhat equivalent to
boosting the influence of the Hessian matrix.
Specifically, given a ReLU network $f$ and an adversarial example $x+\hat\delta$ for adversarial training,
let us temporarily consider the Hessian matrix {\small$H_{h}\overset{\text{def}}{=}\frac{\partial^{2} }{\partial h \partial h^{T}} L(f(x),y)$}~\textit{w.r.t} the output $h$ of the  $j$-th linear layer.
Then, the loss function on adversarial examples {\small$L(f(x+\hat\delta),y)$} can be represented as follows.

\begin{theorem}
\label{theorem:oscillation}
Let {\small$\Delta h = W^{T} \hat\delta \in \mathbb{R}^{1\times D}$} denote the change of the intermediate-layer feature $h$ caused by the perturbation {\small$\hat\delta$}, and
{\small$\text{Loss}(h+\Delta h)=L(f(x+\hat\delta),y)$} represents the loss function on the adversarial example {\small$x+\hat\delta$}.
Then, we use the second-order Taylor expansion to decompose the loss,~\textit{i.e.,}
{\small$\text{Loss}(h+\Delta h)=\text{Loss}(h)+ g_h^{T} \Delta h + \frac{1}{2!}\Delta h ^{T} H_h  \Delta h+R_3(\Delta h) = \text{Loss}(h) + g_h^{T} (W^{T} \hat\delta) + \frac{1}{2!} (W^{T} \hat\delta) ^{T} H_h  (W^{T} \hat\delta)+R_3(\Delta h)$}, where {\small$R_3(\Delta h)$} indicates terms no less than the third order.
In this way, if we focus on the $i$-th dimension of {\small$\hat\delta$}, {\small$\hat\delta_i \in \mathbb{R}$}, the loss can be re-written as follows,
where {\small$W_{i}$} denotes a row vector corresponding to the $i$-th row of the weight matrix $W$, and $\tau$ is a constant~\textit{w.r.t} the change of {\small$W_{i}$}.
\vspace{-5pt}

\begin{small}
\begin{equation}
\label{eqn:oscillation}
\text{Loss}(h+\Delta h)= \tau+ [\hat\delta_i \,g_{h,i}^{T}] W_{i}^{T} +  W_{i} [\frac{1}{2!}\hat\delta_{i}^{2} H_h] W_{i}^{T}.
\end{equation}\end{small}
\end{theorem}
\textbf{(C. 8)}
Theorem~\ref{theorem:oscillation} reveals that adversarial training can be roughly considered to boost the influence of the Hessian matrix~\textit{w.r.t.} network parameters {\small$W_{i}$},~\textit{i.e.,} proportional to {\small$ \hat\delta_{i}^{2} H_h$},
because the perturbation $\hat\delta$ increases exponentially along with the step number, according to Theorems~\ref{theorem:adv_per} and~\ref{theorem:adv_per_2}.
In this way, adversarial training is more likely to make network parameters oscillate than vanilla training.

$\bullet\quad$\textit{Experimental verification of Theorem~\ref{theorem:oscillation}}.
For verification, we conducted experiments to check whether adversarial training boosted the influence of Hessian matrix~\textit{w.r.t.} the network parameters.
Specifically, we learned AlexNet and VGG-11 on the MNIST dataset, and measured effects of adversarial examples on the optimization of network parameters.
To this end, we used an original input sample $x$ and its corresponding adversarial example {\small$x+\delta$} to update the weight {\small$W \in \mathbb{R}^{D\times D}$} in each layer by the length
{\small$\Vert \Delta W\Vert$} and {\small$\Vert \Delta W^{(\text{adv})} \Vert$}, respectively.
Thus, vanilla training's influence and adversarial training's influence of such weight changes on the gradient could be estimated as
{\small$\Delta^{(\text{ori})} =\frac{1}{ D \Vert \Delta W\Vert} \cdot \Vert \frac{\partial L(f(x|W+\Delta W),y)}{\partial (W+\Delta W)} -   \frac{\partial L(f(x|W),y)}{\partial W}\Vert $}, and
{\small$\Delta^{(\text{adv})} = \frac{1} {D \Vert \Delta W^{(\text{adv})} \Vert} \cdot \Vert \frac{\partial L(f(x+\delta|W+\Delta W^{(\text{adv})}),y)}{\partial (W+\Delta W^{(\text{adv})})} -   \frac{\partial L(f(x+\delta|W),y)}{\partial W}\Vert $}, respectively. 
Fig.~\ref{oscillation} compares the influence of weight changes on gradients~\textit{w.r.t.} network parameters.
We discovered that compared to vanilla training, the weight change with a fixed strength in adversarial training usually affected the gradient much more significantly.
Such phenomenon demonstrated that adversarial training boosted the influence of Hessian matrix~\textit{w.r.t.} the network parameters, which verified Theorem~\ref{theorem:oscillation}.
Please see supplementary materials for more details.

\begin{figure}[t]
	\centering
	\includegraphics[width=0.9\linewidth]{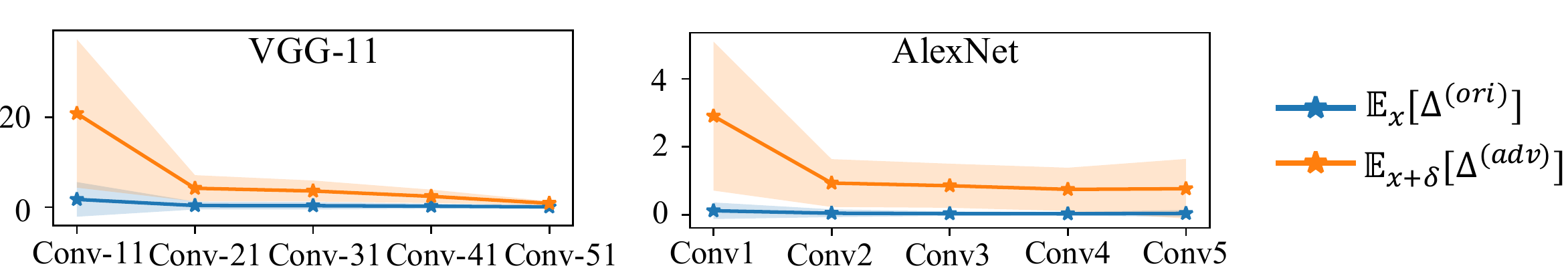}
	\vspace{-8pt}
	\caption{Influence of weight changes on gradients of the loss function~\textit{w.r.t.} network parameters (weights). The weight change in adversarial training makes more significant impacts {\small$\Delta^{(\text{adv})}$} on gradients than that in vanilla training {\small$\Delta^{(\text{ori})}$}, which verifies Theorem~\ref{theorem:oscillation}.}
	\label{oscillation}
	\vspace{-12pt}
\end{figure}

\section{A unified analysis of previous findings in adversarial training}
\label{sec:4}
In this section, we use our proofs to theoretically explain previous findings in adversarial training.

\textbullet\; Many previous studies~\cite{liu2020loss,kanai2021smoothness,wu2020adversarial,yamada2021adversarial} considered that the difficulty of adversarial training was caused by the sharp loss landscape
~\textit{w.r.t} network parameters.
To this end, Theorem~\ref{theorem:oscillation} verifies such an explanation.
Specifically, we have proven that adversarial training can be considered to strengthen the influence of the Hessian matrix of the loss~\textit{w.r.t.} network parameters,
which is equivalent to sharpening the loss landscape.

\textbullet\;\citet{athalye2018obfuscated} discovered that obfuscated gradients led to a false sense of security in defenses against adversarial examples, which hindered adversarial training~\cite{NEURIPS2019_d8700cbd}.
To this end, Theorem~\ref{theorem2} and Theorem~\ref{theorem3} explain the third type of obfuscated gradients in~\cite{athalye2018obfuscated},~\textit{i.e.,} vanishing gradients.
Specifically, we have proven that adversarial training significantly strengthens the influence of a few unconfident samples, and neglects the influence of many confident samples,
which makes the training process more likely to oscillate in directions of a few unconfident samples.
Such oscillation along optimization directions of a few hard samples usually significantly increases norms of weights along such directions, and causes over-confident predictions on some easy samples.
These over-confident predictions on easy samples may lead to vanishing gradients.

\textbullet\;\citet{tsipras2018robustness} clarified that compared to vanilla training, adversarial training mainly relied on robust features and did not use non-robust features for inference, which caused inferior classification performance.
To this end, Theorems~\ref{theorem2} and~\ref{theorem3} verify such a finding.
Specifically, we have proven that adversarial training is mainly dominated by a few samples, which easily makes network parameters oscillate in very few directions.
In other words, the training of non-robust features, or more precisely, training on samples with significant {\small $\hat{\mathcal{A}}$} values that are easily attacked,
is hard to converge.

\textbullet\;\citet{ilyas2019adversarial} demonstrated that adversarial examples were attributed to the presence of highly predictive but non-robust features.
To this end, Theorems~\ref{theorem:adv_per} and~\ref{theorem:adv_per_2} verify such a finding, which reveals that
in the multi-category classification, the direction of the largest eigenvalue of the Hessian matrix {\small$H_x$} suppresses features related to the target category, and promotes features related to the second-best category.
Here, the eigenvector~\textit{w.r.t.} the largest eigenvalue corresponds to non-robust features.

\textbullet\;\citet{liu2021impact} considered that the robust overfitting was caused by the fitting of hard samples, under the assumption that all training samples followed a Gaussian mixture distribution in a logistic regression problem.
To this end, Theorem~\ref{theorem2} and Theorem~\ref{theorem3} explain such a finding in a more generic classification task without assuming the data distribution.
Specifically, we have proven that compared to vanilla training, the adversarially trained network is more likely to be over-fitted to a few unconfident samples, which correspond to hard samples in adversarial training.

\textbullet\;\citet{chen2020robust} discovered that the overfitting in adversarial training was because the network overfitted to adversarial examples generated in the early stage of adversarial training, and failed to generalize to adversarial examples generated in the late stage.
To this end, we provide a deeper insight into such a phenomenon.
Specifically, according to Theorem~\ref{theorem2} and Theorem~\ref{theorem3}, only a few unconfident samples with large {\small$H_z$} and large gradient norms {\small$\Vert\tilde{g}_x\Vert$} influence the adversarial training.
In fact, the imbalance of the sample influence can easily make unconfident samples with large {\small$H_z$} values and large gradient norms {\small$\Vert\tilde{g}_x\Vert$} in the early stage of adversarial training significantly different from those in the late stage.
Such mechanisms lead to the overfitting in adversarial training.

\textbullet\;\citet{rice2020overfitting} demonstrated that early stopping could effectively reduce overfitting in adversarial training.
To this end, Theorem~\ref{theorem2} and Theorem~\ref{theorem3} also explain the effectiveness of the early stopping.
Specifically, during adversarial training process, the network becomes robust, and the number of unconfident samples decreases.
Because adversarially trained networks mainly focus on unconfident samples, the decreasing number of unconfident samples boosts the significance of overfitting.
In this way, early stopping can effectively reduce overfitting.


\section{Conclusion and discussion}
\label{sec:conclusion}
This paper theoretically analyzes the dynamics of adversarial perturbations via an analytic solution.
We also prove that adversarial training strengthens the influence of a few input samples, which boosts the difficulty of adversarial training.
Crucially, our proofs provide a theoretical explanation for previous studies in understanding adversarial training.
Note that our analysis is all based on the assumption that adversarial perturbations cannot significantly change the gating states of the ReLU network.
Despite this, experimental results show that our analysis can still explain most adversarial perturbations generated in real cases, when gating states change.
Besides, in this paper, we use the normalized perturbations to approximate adversarial perturbations of the $\ell_2$ attack and the $\ell_\infty$ attack, instead of deriving an exact formulation for these perturbations.
Nevertheless, experimental results show that our analysis can well explain the $\ell_2$ attack and the $\ell_\infty$ attack, to some extent.



\bibliographystyle{plainnat}
\bibliography{adv_train}

\begin{thebibliography}{34}
\providecommand{\natexlab}[1]{#1}
\providecommand{\url}[1]{\texttt{#1}}
\expandafter\ifx\csname urlstyle\endcsname\relax
  \providecommand{\doi}[1]{doi: #1}\else
  \providecommand{\doi}{doi: \begingroup \urlstyle{rm}\Url}\fi

\bibitem[Athalye et~al.(2018)Athalye, Carlini, and
  Wagner]{athalye2018obfuscated}
Anish Athalye, Nicholas Carlini, and David Wagner.
\newblock Obfuscated gradients give a false sense of security: Circumventing
  defenses to adversarial examples.
\newblock In \emph{International conference on machine learning}, pages
  274--283. PMLR, 2018.

\bibitem[Chen et~al.(2020)Chen, Zhang, Liu, Chang, and Wang]{chen2020robust}
Tianlong Chen, Zhenyu Zhang, Sijia Liu, Shiyu Chang, and Zhangyang Wang.
\newblock Robust overfitting may be mitigated by properly learned smoothening.
\newblock In \emph{International Conference on Learning Representations}, 2020.

\bibitem[Cohen et~al.(2021)Cohen, Kaur, Li, Kolter, and
  Talwalkar]{cohen2021gradient}
Jeremy Cohen, Simran Kaur, Yuanzhi Li, J~Zico Kolter, and Ameet Talwalkar.
\newblock Gradient descent on neural networks typically occurs at the edge of
  stability.
\newblock In \emph{International Conference on Learning Representations}, 2021.
\newblock URL \url{https://openreview.net/forum?id=jh-rTtvkGeM}.

\bibitem[Dong et~al.(2018)Dong, Liao, Pang, Su, Zhu, Hu, and
  Li]{dong2018boosting}
Yinpeng Dong, Fangzhou Liao, Tianyu Pang, Hang Su, Jun Zhu, Xiaolin Hu, and
  Jianguo Li.
\newblock Boosting adversarial attacks with momentum.
\newblock In \emph{Proceedings of the IEEE conference on computer vision and
  pattern recognition}, pages 9185--9193, 2018.

\bibitem[Goodfellow et~al.(2014)Goodfellow, Shlens, and
  Szegedy]{goodfellow2014explaining}
Ian~J Goodfellow, Jonathon Shlens, and Christian Szegedy.
\newblock Explaining and harnessing adversarial examples.
\newblock \emph{arXiv preprint arXiv:1412.6572}, 2014.

\bibitem[He et~al.(2016)He, Zhang, Ren, and Sun]{he2016deep}
Kaiming He, Xiangyu Zhang, Shaoqing Ren, and Jian Sun.
\newblock Deep residual learning for image recognition.
\newblock In \emph{Proceedings of the IEEE conference on computer vision and
  pattern recognition}, pages 770--778, 2016.

\bibitem[Ilyas et~al.(2019)Ilyas, Santurkar, Tsipras, Engstrom, Tran, and
  Madry]{ilyas2019adversarial}
Andrew Ilyas, Shibani Santurkar, Dimitris Tsipras, Logan Engstrom, Brandon
  Tran, and Aleksander Madry.
\newblock Adversarial examples are not bugs, they are features.
\newblock \emph{Advances in neural information processing systems}, 32, 2019.

\bibitem[Kanai et~al.(2021)Kanai, Yamada, Takahashi, Yamanaka, and
  Ida]{kanai2021smoothness}
Sekitoshi Kanai, Masanori Yamada, Hiroshi Takahashi, Yuki Yamanaka, and
  Yasutoshi Ida.
\newblock Smoothness analysis of adversarial training.
\newblock \emph{arXiv preprint arXiv:2103.01400}, 2021.

\bibitem[Krizhevsky et~al.(2012)Krizhevsky, Sutskever, and
  Hinton]{krizhevsky2012imagenet}
Alex Krizhevsky, Ilya Sutskever, and Geoffrey~E Hinton.
\newblock Imagenet classification with deep convolutional neural networks.
\newblock \emph{Advances in neural information processing systems}, 25, 2012.

\bibitem[Kurakin et~al.(2016)Kurakin, Goodfellow, and
  Bengio]{kurakin2016adversarial}
Alexey Kurakin, Ian Goodfellow, and Samy Bengio.
\newblock Adversarial machine learning at scale.
\newblock \emph{arXiv preprint arXiv:1611.01236}, 2016.

\bibitem[LeCun et~al.(1998)LeCun, Bottou, Bengio, and
  Haffner]{lecun1998gradient}
Yann LeCun, L{\'e}on Bottou, Yoshua Bengio, and Patrick Haffner.
\newblock Gradient-based learning applied to document recognition.
\newblock \emph{Proceedings of the IEEE}, 86\penalty0 (11):\penalty0
  2278--2324, 1998.

\bibitem[LeCun et~al.(2015)LeCun, Bengio, and Hinton]{lecun2015deep}
Yann LeCun, Yoshua Bengio, and Geoffrey Hinton.
\newblock Deep learning.
\newblock \emph{nature}, 521\penalty0 (7553):\penalty0 436--444, 2015.

\bibitem[Liu et~al.(2020)Liu, Salzmann, Lin, Tomioka, and
  S{\"u}sstrunk]{liu2020loss}
Chen Liu, Mathieu Salzmann, Tao Lin, Ryota Tomioka, and Sabine S{\"u}sstrunk.
\newblock On the loss landscape of adversarial training: Identifying challenges
  and how to overcome them.
\newblock \emph{Advances in Neural Information Processing Systems},
  33:\penalty0 21476--21487, 2020.

\bibitem[Liu et~al.(2021)Liu, Huang, Salzmann, Zhang, and
  S{\"u}sstrunk]{liu2021impact}
Chen Liu, Zhichao Huang, Mathieu Salzmann, Tong Zhang, and Sabine
  S{\"u}sstrunk.
\newblock On the impact of hard adversarial instances on overfitting in
  adversarial training.
\newblock \emph{arXiv preprint arXiv:2112.07324}, 2021.

\bibitem[Madry et~al.(2018)Madry, Makelov, Schmidt, Tsipras, and
  Vladu]{madry2018towards}
Aleksander Madry, Aleksandar Makelov, Ludwig Schmidt, Dimitris Tsipras, and
  Adrian Vladu.
\newblock Towards deep learning models resistant to adversarial attacks.
\newblock In \emph{International Conference on Learning Representations}, 2018.
\newblock URL \url{https://openreview.net/forum?id=rJzIBfZAb}.

\bibitem[Miyato et~al.(2018)Miyato, Maeda, Koyama, and
  Ishii]{miyato2018virtual}
Takeru Miyato, Shin-ichi Maeda, Masanori Koyama, and Shin Ishii.
\newblock Virtual adversarial training: a regularization method for supervised
  and semi-supervised learning.
\newblock \emph{IEEE transactions on pattern analysis and machine
  intelligence}, 41\penalty0 (8):\penalty0 1979--1993, 2018.

\bibitem[Papernot et~al.(2017)Papernot, McDaniel, Goodfellow, Jha, Celik, and
  Swami]{papernot2017practical}
Nicolas Papernot, Patrick McDaniel, Ian Goodfellow, Somesh Jha, Z~Berkay Celik,
  and Ananthram Swami.
\newblock Practical black-box attacks against machine learning.
\newblock In \emph{Proceedings of the 2017 ACM on Asia conference on computer
  and communications security}, pages 506--519, 2017.

\bibitem[QIAN et~al.(2022)QIAN, Zhang, Huang, Wang, Gu, Xiong, and
  Yi]{qian2022perturbation}
Zhuang QIAN, Shufei Zhang, Kaizhu Huang, Qiufeng Wang, Bin Gu, Huan Xiong, and
  Xinping Yi.
\newblock Perturbation diversity certificates robust generalisation, 2022.
\newblock URL \url{https://openreview.net/forum?id=jm1RxJFQdDN}.

\bibitem[Ren et~al.(2022)Ren, Li, Zhou, Chan, and Zhang]{ren2022towards}
Jie Ren, Mingjie Li, Meng Zhou, Shih-Han Chan, and Quanshi Zhang.
\newblock Towards theoretical analysis of transformation complexity of relu
  dnns.
\newblock \emph{arXiv preprint arXiv:2205.01940}, 2022.

\bibitem[Rice et~al.(2020)Rice, Wong, and Kolter]{rice2020overfitting}
Leslie Rice, Eric Wong, and Zico Kolter.
\newblock Overfitting in adversarially robust deep learning.
\newblock In \emph{International Conference on Machine Learning}, pages
  8093--8104. PMLR, 2020.

\bibitem[Schmidt et~al.(2018)Schmidt, Santurkar, Tsipras, Talwar, and
  Madry]{schmidt2018adversarially}
Ludwig Schmidt, Shibani Santurkar, Dimitris Tsipras, Kunal Talwar, and
  Aleksander Madry.
\newblock Adversarially robust generalization requires more data.
\newblock \emph{Advances in neural information processing systems}, 31, 2018.

\bibitem[Simonyan and Zisserman(2014)]{simonyan2014very}
Karen Simonyan and Andrew Zisserman.
\newblock Very deep convolutional networks for large-scale image recognition.
\newblock \emph{arXiv preprint arXiv:1409.1556}, 2014.

\bibitem[Sinha et~al.(2017)Sinha, Namkoong, Volpi, and
  Duchi]{sinha2017certifying}
Aman Sinha, Hongseok Namkoong, Riccardo Volpi, and John Duchi.
\newblock Certifying some distributional robustness with principled adversarial
  training.
\newblock \emph{arXiv preprint arXiv:1710.10571}, 2017.

\bibitem[Stutz et~al.(2020)Stutz, Hein, and Schiele]{stutz2020confidence}
David Stutz, Matthias Hein, and Bernt Schiele.
\newblock Confidence-calibrated adversarial training: Generalizing to unseen
  attacks.
\newblock In \emph{International Conference on Machine Learning}, pages
  9155--9166. PMLR, 2020.

\bibitem[Tramèr et~al.(2018)Tramèr, Kurakin, Papernot, Goodfellow, Boneh, and
  McDaniel]{tramer2018ensemble}
Florian Tramèr, Alexey Kurakin, Nicolas Papernot, Ian Goodfellow, Dan Boneh,
  and Patrick McDaniel.
\newblock Ensemble adversarial training: Attacks and defenses.
\newblock In \emph{International Conference on Learning Representations}, 2018.
\newblock URL \url{https://openreview.net/forum?id=rkZvSe-RZ}.

\bibitem[Tsipras et~al.(2019)Tsipras, Santurkar, Engstrom, Turner, and
  Madry]{tsipras2018robustness}
Dimitris Tsipras, Shibani Santurkar, Logan Engstrom, Alexander Turner, and
  Aleksander Madry.
\newblock Robustness may be at odds with accuracy.
\newblock In \emph{International Conference on Learning Representations}, 2019.
\newblock URL \url{https://openreview.net/forum?id=SyxAb30cY7}.

\bibitem[Wong et~al.(2020)Wong, Rice, and Kolter]{wong2020fast}
Eric Wong, Leslie Rice, and J~Zico Kolter.
\newblock Fast is better than free: Revisiting adversarial training.
\newblock \emph{arXiv preprint arXiv:2001.03994}, 2020.

\bibitem[Wu et~al.(2020)Wu, Xia, and Wang]{wu2020adversarial}
Dongxian Wu, Shu-Tao Xia, and Yisen Wang.
\newblock Adversarial weight perturbation helps robust generalization.
\newblock \emph{Advances in Neural Information Processing Systems},
  33:\penalty0 2958--2969, 2020.

\bibitem[Wu et~al.(2018)Wu, Ma, et~al.]{wu2018sgd}
Lei Wu, Chao Ma, et~al.
\newblock How sgd selects the global minima in over-parameterized learning: A
  dynamical stability perspective.
\newblock \emph{Advances in Neural Information Processing Systems}, 31, 2018.

\bibitem[Yamada et~al.(2021)Yamada, Kanai, Iwata, Takahashi, Yamanaka,
  Takahashi, and Kumagai]{yamada2021adversarial}
Masanori Yamada, Sekitoshi Kanai, Tomoharu Iwata, Tomokatsu Takahashi, Yuki
  Yamanaka, Hiroshi Takahashi, and Atsutoshi Kumagai.
\newblock Adversarial training makes weight loss landscape sharper in logistic
  regression.
\newblock \emph{arXiv preprint arXiv:2102.02950}, 2021.

\bibitem[Yu et~al.(2018)Yu, Liu, Wang, Zhao, and Chen]{yu2018interpreting}
Fuxun Yu, Chenchen Liu, Yanzhi Wang, Liang Zhao, and Xiang Chen.
\newblock Interpreting adversarial robustness: A view from decision surface in
  input space.
\newblock \emph{arXiv preprint arXiv:1810.00144}, 2018.

\bibitem[Zhai et~al.(2019)Zhai, Cai, He, Dan, He, Hopcroft, and
  Wang]{zhai2019adversarially}
Runtian Zhai, Tianle Cai, Di~He, Chen Dan, Kun He, John Hopcroft, and Liwei
  Wang.
\newblock Adversarially robust generalization just requires more unlabeled
  data.
\newblock \emph{arXiv preprint arXiv:1906.00555}, 2019.

\bibitem[Zhang and Wang(2019{\natexlab{a}})]{NEURIPS2019_d8700cbd}
Haichao Zhang and Jianyu Wang.
\newblock Defense against adversarial attacks using feature scattering-based
  adversarial training.
\newblock In H.~Wallach, H.~Larochelle, A.~Beygelzimer, F.~d\textquotesingle
  Alch\'{e}-Buc, E.~Fox, and R.~Garnett, editors, \emph{Advances in Neural
  Information Processing Systems}, volume~32. Curran Associates, Inc.,
  2019{\natexlab{a}}.
\newblock URL
  \url{https://proceedings.neurips.cc/paper/2019/file/d8700cbd38cc9f30cecb34f0c195b137-Paper.pdf}.

\bibitem[Zhang and Wang(2019{\natexlab{b}})]{zhang2019defense}
Haichao Zhang and Jianyu Wang.
\newblock Defense against adversarial attacks using feature scattering-based
  adversarial training.
\newblock \emph{Advances in Neural Information Processing Systems}, 32,
  2019{\natexlab{b}}.

\end{thebibliography}

\newpage
\appendix


\section{Proof of Theorem~\ref{theorem:adv_per}}
\label{supp_sec:theorem1}
In this section, we prove Theorem~\ref{theorem:adv_per} in Section~\ref{sec:3.1} of the main paper, which analyzes the dynamics of perturbations of the $m$-step attack.

Let us focus on the most straightforward solution to the multi-step adversarial attack.
In this scenario, given a ReLU network $f$ and an input sample $x \in \mathbb{R}^{n}$, the perturbation generated after attacking for $m$ steps is formulated as follows.
\begin{equation}
\label{eqn:adv_per}
\delta^{(m)} =\sum\nolimits_{t=0}^{m-1}  \alpha \cdot  g_{x+ \delta^{(t)}},
\end{equation}
where {\small$g_{x+ \delta^{(t)}} =\frac{\partial}{\partial x} L(f(x+ \delta^{(t)}), y)$} represents the gradient of the loss~\textit{w.r.t.} the input sample $x$, and $m$ denotes the step size.
Furthermore, we define the update of the perturbation at each step $t$ as follows.
\begin{equation}
\label{eqn:adv_per_step}
\Delta {x}^{(t)} \overset{\text{def}}{=} \alpha \cdot  g_{x+ \delta^{(t-1)}},
\end{equation}
In this way, the perturbation {\small$\delta^{(m)}$} generated after the $m$-step attack can be re-written as
\begin{equation}
\label{eqn:adv_per_sum}
\delta^{(m)} = \Delta {x}^{(1)} + \Delta {x}^{(2)} + \cdots + \Delta {x}^{(m)}.
\end{equation}

\begin{lemma}[in Appendix]
\label{lemma:adv_per_step}
The update of the perturbation with the multi-step attack at step $t$ can be represented as
{\small$\Delta {x}^{(t)} =\alpha(I+\alpha H_x)^{t-1}g_x + \alpha R_2( \delta^{(t-1)})$}, where
{\small$g_x = \frac{\partial}{\partial x} L(f(x), y)$};
{\small$H_x = \frac{\partial^{2}}{\partial x \partial x^{T}} L(f(x),y)$};
{\small$\delta^{(t-1)}$} represents the perturbation generated after the {\small$(t-1)$}-step attack,
and {\small $  R_2( \delta^{(t-1)})$} denotes terms of {\small$\delta^{(t-1)}$} no less than the second order in Taylor expansion.
\end{lemma}


\begin{proof}
If the step {\small$t=1$}, according to Eq.~\eqref{eqn:adv_per_step}, we have {\small$\Delta {x}^{(1)}=\alpha g_x$}.

For $\forall t>1$, the perturbation of the $t$-th step attack is defined as {\small$\Delta {x}^{(t)} = \alpha \cdot  g_{x+ \delta^{(t-1)}}$} in Eq.~\eqref{eqn:adv_per_step}.
In order to simplify the perturbation {\small$\Delta {x}^{(t)}$}, we use the Taylor expansion to decompose the gradient {\small$g_{x+ \delta^{(t-1)}}$} of the loss~\textit{w.r.t.} the adversarial example {\small$x+ \delta^{(t-1)}$}.
Here, {\small$R_2( \delta^{(t-1)})$} denotes the term of the perturbation $\delta^{(t-1)}$ no less than the second order.
\begin{align}
\label{eqn:adv_per_step_taylor}
g_{x+ \delta^{(t-1)}} = g_x + H_x \delta^{(t-1)} + R_2( \delta^{(t-1)}).
\end{align}

Substituting Eq.~\eqref{eqn:adv_per_step_taylor} back to Eq.~\eqref{eqn:adv_per_step}, the perturbation {\small$\Delta {x}^{(t)}$} can be re-written as
\begin{equation}
\begin{aligned}
\label{eqn:adv_per_step_taylor_new}
\Delta {x}^{(t)} &= \alpha \cdot  g_{x+ \delta^{(t-1)}} \\
& = \alpha \cdot g_x + \alpha \cdot  H_x \delta^{(t-1)} + \alpha \cdot  R_2( \delta^{(t-1)}).
\end{aligned}
\end{equation}

In this way, we apply the mathematical induction to prove {\small$\forall 1\le t\le m, \Delta {x}^{(t)}=\alpha \cdot (I+\alpha H_x)^{t-1}g_x + \alpha \cdot  R_2( \delta^{(t-1)})$} of Lemma~\ref{lemma:adv_per_step} in Appendix.

\emph{Base case:} When $t=1$, we have {\small$\Delta {x}^{(1)}=\alpha \cdot  g_x=\alpha \cdot  (I+\alpha H_x)^{0}g_x$}.

\emph{Inductive step:}
For $t>1$, assuming {\small$\Delta {x}^{(t-1)}=\alpha \cdot  (I+\alpha H_x)^{t-2}g_x + \alpha \cdot  R_2( \delta^{(t-2)})$}, we have
\begin{equation}
\label{eqn:delta_final}
\begin{aligned}
\Delta {x}^{(t)}&=\alpha  \cdot \bigg(g_x+H_x\,\sum\nolimits_{i=1}^{t-1}\Delta {x}^{(i)} + R_2( \delta^{(t-1)})\bigg)
\quad // \quad \text{According to Eq.~\eqref{eqn:adv_per_step_taylor_new}}
\\
&=\alpha   \cdot \bigg[g_x+H_x\,\sum\nolimits_{i=1}^{t-2}\Delta {x}^{(i)} + R_2( \delta^{(t-2)})\bigg]
+ \alpha \cdot H_x \Delta {x}^{(t-1)} + \alpha \cdot R_2( \Delta {x}^{(t-1)})
\\
&=\Delta {x}^{(t-1)}+ \alpha \cdot H_x \Delta {x}^{(t-1)} + \alpha \cdot R_2( \Delta {x}^{(t-1)})
\quad // \quad \text{According to Eq.~\eqref{eqn:adv_per_step_taylor_new}}
\\
&=(I+\alpha \cdot H_x)\Delta {x}^{(t-1)} + \alpha \cdot R_2( \Delta {x}^{(t-1)})
\\
&=(I+\alpha \cdot H_x)\, \alpha \cdot \bigg[(I+\alpha H_x)^{(t-2)}g_x + R_2( \delta^{(t-2)})  \bigg]
+ \alpha \cdot R_2( \Delta {x}^{(t-1)})
\\
&=\alpha  \cdot (I+\alpha H_x)^{t-1}g_x + \alpha  \cdot  R_2( \delta^{(t-1)}),
\end{aligned}
\end{equation}
where {\small$R_2( \Delta {x}^{(t-1)})$} is referred to as the term of the perturbation {\small$\Delta {x}^{(t-1)}$} no less than the second order.

\emph{Conclusion:} Since both the base case and the inductive step have been proven to be true, we have {\small$\Delta {x}^{(t)}=\alpha(I+\alpha H_x)^{t-1}g_x + \alpha R_2( \delta^{(t-1)})$}.

Thus, Lemma~\ref{lemma:adv_per_step} in Appendix is proven.
\end{proof}

\textbf{Theorem 1.}
\textit{(Dynamics of perturbations of the m-step attack).
Based on Assumption 1,
the adversarial perturbation {\small$\delta^{(m)}$} is given as follows,
where {\small$\lambda_{i}$} and {\small$v_{i}$} denote the $i$-th largest eigenvalue of the Hessian matrix {\small$H_x= \frac{\partial^{2}}{\partial x \partial x^{T}} L(f(x),y)$} and its corresponding eigenvector, respectively.
{\small$\gamma_i = g_x^{T} v_{i} \in \mathbb{R}$} represents the projection of the gradient {\small$g_x=\frac{\partial}{\partial x} L(f(x),y)$} on the eigenvector {\small$v_{i}$}.
{\small$\mathcal{R}_2(\alpha)$} denote the sum of terms no less than the second order in Taylor expansion~\textit{w.r.t.} {\small$\delta^{(m)}$}.}

\begin{small}
\begin{equation}
\label{eqn:solution_adv_per}
\delta^{(m)} =  \sum_{i=1}^{n}  \frac{(1+\alpha \lambda_{i} ) ^{m}-1}{\lambda_{i}} \gamma_{i} v_{i} + \mathcal{R}_2(\alpha),\quad
g_{x+\delta^{(m)}} = \sum_{i=1}^{n} (1+\alpha \lambda_{i} ) ^{m} \gamma_{i} v_{i}.
\end{equation}\end{small}

\begin{proof}
According to Eq.~\eqref{eqn:adv_per_sum} and Lemma~\ref{lemma:adv_per_step} in Appendix, the perturbation {\small$\delta^{(m)}$} generated after the $m$-step attack can be re-written as
\begin{equation}
\label{supp_eqn:perturbation_m_2}
\begin{aligned}
\delta^{(m)}&=\Delta {x}^{(1)} + \Delta {x}^{(2)} + \cdots + \Delta {x}^{(m)}
\\
\\
&=\alpha [I+(I+\alpha H_x)+\cdots +(I+\alpha H_x)^{m-1}]g_x
\\
&\quad + \alpha R_2(\delta^{(1)}) + \alpha R_2(\delta^{(2)})+\cdots + \alpha R_2(\delta^{(m-1)})
\\
\\
&= \alpha [I+(I+\alpha H_x)+\cdots +(I+\alpha H_x)^{m-1}]g_x + \mathcal{R}_2(\alpha).
\end{aligned}
\end{equation}
Here, considering that {\small$\forall1\leq t \leq m, R_2(\delta^{(t)})$} is proportional to {\small$\alpha^{2}$}, the sum {\small$\alpha R_2(\delta^{(1)}) + \alpha R_2(\delta^{(2)})+\cdots + \alpha R_2(\delta^{(m-1)})$} is proportional to {\small$\alpha^{2}$}.
In this way, for simplicity, we use {\small$\mathcal{R}_2(\alpha) \propto \alpha^{2}$} to denote the sum of terms no less than the second order in Taylor expansion,~\textit{i.e.}, {\small$\mathcal{R}_2(\alpha) = \alpha R_2(\delta^{(1)}) + \alpha R_2(\delta^{(2)})+\cdots + \alpha R_2(\delta^{(m-1)})$} .
Moreover, the term {\small$\alpha [I+(I+\alpha H_x)+\cdots +(I+\alpha H_x)^{m-1}]$} is roughly considered to be proportional to $\alpha$.
Since the step size $\alpha$ is small enough, we can ignore the term {\small$\mathcal{R}_2(\alpha)$}, and the perturbation $\delta^{(m)}$ in Eq.~\eqref{supp_eqn:perturbation_m_2} can be roughly approximated as
\begin{equation}
\label{supp_eqn:perturbation_m_approx}
\begin{aligned}
\delta^{(m)}
&=\alpha [I+(I+\alpha H_x)+\cdots +(I+\alpha H_x)^{m-1}]g_x + \mathcal{R}_2(\alpha)
\\
&\approx \alpha [I+(I+\alpha H_x)+\cdots +(I+\alpha H_x)^{m-1}]g_x.
\end{aligned}
\end{equation}

Because the Hessian matrix {\small$H_x$} is a real-valued symmetric matrix, we can use the eigenvalue decomposition to decompose {\small$H_x$} as {\small$H_x=V\Lambda V^{-1}$}.
Here, {\small$\Lambda=diag[\lambda_1,\lambda_2,\cdots,\lambda_p]$} is a diagonal matrix, whose diagonal elements are the corresponding eigenvalues, {\small$\Lambda_{ii}=\lambda_i$}.
The square matrix $V=[v_1, v_2, \cdots, v_n] \in \mathbb{R}^{n\times n}$ contains $n$ linearly independent eigenvectors {\small$v_i$},~\textit{i.e.,} {\small$\forall i\neq k, v_i^Tv_k=0$},
where {\small$v_i$} is the eigenvector corresponding to the eigenvalue {\small$\lambda_i$}.
Without loss generality, we normalize these $n$ eigenvectors {\small$v_i$}, thereby {\small$V^TV=I$}.
In this scenario, $H_x$ can be decomposed as {\small$H_x=V\Lambda V^{T}$}, and the perturbation {\small$\delta^{(m)}$} can be represented as
\begin{equation}
\begin{small}
\begin{aligned}
\label{supp_eqn:perturbation_m_3}
\delta^{(m)}&=\alpha [I+(I+\alpha V\Lambda V^T)+\cdots +(I+\alpha V\Lambda V^T)^{m-1}]g_x +  \mathcal{R}_2(\alpha)
\\
&=\alpha [VV^T+(VV^T+\alpha V\Lambda V^T)+\cdots +(VV^T+\alpha V\Lambda V^T)^{m-1}]g_x +  \mathcal{R}_2(\alpha)
\\
&=\alpha [VIV^T+V(I+\alpha\Lambda)V^T+\cdots +[V(I+\alpha \Lambda )V^T]^{m-1}]g_x +  \mathcal{R}_2(\alpha)
\\
&=\alpha [VIV^T+V(I+\alpha\Lambda)V^T+\cdots +V(I+\alpha \Lambda )^{m-1}V^T]g_x +  \mathcal{R}_2(\alpha)
 \\
&=\alpha V[I+(I+\alpha\Lambda)+\cdots +(I+\alpha \Lambda )^{m-1}]V^Tg_x +  \mathcal{R}_2(\alpha).
\end{aligned}
\end{small}
\end{equation}

For simplicity, let {\small$D=\alpha(I+(I+\alpha\Lambda)+\cdots +(I+\alpha \Lambda )^{m-1})$}, which is a diagonal matrix, since {\small$I$}, {\small$I+\alpha\Lambda$}, ..., {\small$(I+\alpha \Lambda)^{m-1}$} are all diagonal matrices.
In this way, let us focus on the $k$-th diagonal element {\small$D_{kk} \in \mathbb{R}$}.
\begin{equation}
\begin{aligned}
\label{supp_eqn:diagonal_m_1}
D_{kk}&=\alpha(1+(1+\alpha \lambda_k)+\cdots+(1+\alpha \lambda_k)^{m-1})\\
&=\alpha(1\times\frac{1-(1+\alpha\lambda_k)^m}{1-(1+\alpha\lambda_k)})\\
&=\frac{(1+\alpha\lambda_k)^m-1}{\lambda_k}.
\end{aligned}
\end{equation}

Then, combining Eq.~\eqref{supp_eqn:perturbation_m_3} and Eq.~\eqref{supp_eqn:diagonal_m_1}, the perturbation {\small$\delta^{(m)}$} can be written as follows.
Here, considering that $n$ eigenvectors of the Hessian matrix form a set of unit orthogonal basis, the gradient {\small$g_x$} can be represented as {\small$g_x=\sum_{i=1}^p \gamma_iv_i$}, where {\small$\gamma_i$} is referred to as the projection length of {\small$g_x$} on {\small$v_i$}.
\begin{equation}
\label{supp_eqn:perturbation_m}
\begin{aligned}
\delta^{(m)}&=VDV^Tg_x +  \mathcal{R}_2(\alpha)\\
&=VDV^T(\sum_{i=1}^n\gamma_iv_i) +  \mathcal{R}_2(\alpha) \\
&=\sum_{i=1}^{n} D_{ii}v_iv_i^T\sum_{k=1}^n\gamma_k v_k +  \mathcal{R}_2(\alpha) \\
&=\sum_i \frac{(1+\alpha\lambda_i)^m-1}{\lambda_i}\gamma_i v_i +  \mathcal{R}_2(\alpha).
\end{aligned}
\end{equation}

Furthermore, we use the first-order Taylor expansion to decompose the gradient {\small$g_{x+\delta^{(m)}}$} of the loss~\textit{w.r.t.} the adversarial example {\small$x+\delta^{(m)}$}.
Here, {\small$R_2(\delta^{(m)})$} denotes terms of the perturbation {\small$\delta^{(m)}$} no less than the second order.
\begin{equation}
\label{eqn:grad_taylor_m}
g_{x+\delta^{(m)}} = g_x + H_x \; \delta^{(m)} + R_2(\delta^{(m)}).
\end{equation}

Substituting Eq.~\eqref{supp_eqn:perturbation_m} back to Eq.~\eqref{eqn:grad_taylor_m}, the gradient {\small$g_{x+\delta^{(m)}}$} can be written as
\begin{equation}
\label{eqn:grad_taylor_m_1}
\begin{aligned}
g_{x+\delta^{(m)}} &= g_x + H_x  \bigg( \sum_i \frac{(1+\alpha\lambda_i)^m-1}{\lambda_i}\gamma_iv_i + \mathcal{R}_2(\alpha) \bigg)
+ R_2(\delta^{(m)})
\\
& = \sum_{i=1}^n \gamma_iv_i  + H_x \sum_i \frac{(1+\alpha\lambda_i)^m-1}{\lambda_i}\gamma_iv_i
+ H_x \mathcal{R}_2(\alpha) + R_2(\delta^{(m)})
 \\
& = \sum_{i=1}^n \gamma_iv_i  + \sum_i \frac{(1+\alpha\lambda_i)^m-1}{\lambda_i}\gamma_i (H_x v_i)
+ H_x \mathcal{R}_2(\alpha) + R_2(\delta^{(m)})
\\
& = \sum_{i=1}^n \gamma_iv_i  + \sum_i \frac{(1+\alpha\lambda_i)^m-1}{\lambda_i}\gamma_i (\lambda_i v_i)
+H_x \mathcal{R}_2(\alpha) + R_2(\delta^{(m)})
\\
& \approx \sum_i (1+\alpha\lambda_i)^m \;\gamma_i  v_i .
\end{aligned}
\end{equation}

Hence, Theorem~\ref{theorem:adv_per} is proven.
\end{proof}

\section{Proof of Theorem~\ref{theorem:adv_per_2}}
\label{supp_sec:theorem2}
In this section, we prove Theorem~\ref{theorem:adv_per_2} in Section~\ref{sec:3.1} of the main paper, which analyzes the adversarial perturbation of the infinite-step attack.

\textbf{Theorem 2.}
\textit{(Perturbations of the infinite-step attack).
{\small$\beta = \alpha m$} reflects the overall adversarial strength of the infinite-step attack with the step number {\small$m \to + \infty$} and the step size {\small$\alpha=\beta/m \to 0$}.
Then, this infinite-step adversarial perturbation {\small$\hat\delta = \lim_{m\to+ \infty }\alpha \sum\nolimits_{t=0}^{m-1} \frac{\partial}{\partial x} L(f(x+ \delta^{(t)}), y)$} can be re-written as follows, where
{\small$\hat{\mathcal{R}}_2(\alpha)$} denote the sum of terms no less than the second order in Taylor expansion~\textit{w.r.t.} {\small$\hat\delta$}.
}

\begin{small}
\begin{equation}
\label{eqn:solution_adv_per_2}
\hat\delta =  \sum_{i=1}^{n}  \frac{\exp(\beta \lambda_{i})-1}{\lambda_{i}} \gamma_{i} v_{i} +\hat{\mathcal{R}}_2(\alpha),\quad
g_{x+\hat\delta} = \sum_{i=1}^{n} \exp(\beta \lambda_{i}) \gamma_{i} v_{i} .
\end{equation}\end{small}

\begin{proof}
According to Eq.~\eqref{eqn:adv_per_sum} and Lemma~\ref{lemma:adv_per_step} in Appendix, when the step number {\small$m \to + \infty$}, the infinite-step adversarial perturbation {\small$\hat\delta$} can be represented as
\begin{equation}
\begin{small}
\label{supp_eqn:perturbation_2}
\begin{aligned}
\hat\delta
&=\lim_{m\to+ \infty } \Delta {x}^{(1)} + \Delta {x}^{(2)} + \cdots + \Delta {x}^{(m)}
\\
\\
&= \lim_{m\to+ \infty } \alpha [I+(I+\alpha H_x)+\cdots +(I+\alpha H_x)^{m-1}]g_x +
\\
&\quad + \lim_{m\to+ \infty } \alpha R_2(\delta^{(1)}) + \alpha R_2(\delta^{(2)})+\cdots + \alpha R_2(\delta^{(m-1)})
\\
\\
&= \lim_{m\to+ \infty } \alpha [I+(I+\alpha H_x)+\cdots +(I+\alpha H_x)^{m-1}]g_x + \hat{\mathcal{R}}_2(\alpha).
\end{aligned}
\end{small}
\end{equation}

Here, considering that {\small$\forall1\leq t \leq m, R_2(\delta^{(t)})$} is proportional to $\alpha^{2}$, the sum {\small$ \lim_{m\to+ \infty } \alpha R_2(\delta^{(1)}) + \alpha R_2(\delta^{(2)})+\cdots + \alpha R_2(\delta^{(m-1)})$} is proportional to $\alpha^{2}$.
In this way, for simplicity, we use {\small$\hat{\mathcal{R}}_2(\alpha) \propto \alpha^{2}$} to denote the sum of terms no less than the second order in Taylor expansion,~\textit{i.e.}, {\small$\hat{\mathcal{R}}_2(\alpha) = \lim_{m\to+ \infty } \alpha R_2(\delta^{(1)}) + \alpha R_2(\delta^{(2)})+\cdots + \alpha R_2(\delta^{(m-1)})$}.
Moreover, the term {\small$\lim_{m\to+ \infty } \alpha [I+(I+\alpha H_x)+\cdots +(I+\alpha H_x)^{m-1}]$} can be roughly considered to proportional to $\alpha$.
Since the step size is infinitesimal {\small$\alpha=\beta/m \to 0$}, we can ignore the term {\small$\hat{\mathcal{R}}_2(\alpha) $}, and the perturbation $\hat\delta$ in Eq.~\eqref{supp_eqn:perturbation_2} can be roughly approximated as
\begin{equation}
\label{supp_eqn:perturbation_approx}
\begin{aligned}
\hat\delta
&= \lim_{m\to+ \infty } \alpha [I+(I+\alpha H_x)+\cdots +(I+\alpha H_x)^{m-1}]g_x + \hat{\mathcal{R}}_2(\alpha)
\\
&\approx  \lim_{m\to+ \infty } \alpha [I+(I+\alpha H_x)+\cdots +(I+\alpha H_x)^{m-1}]g_x.
\end{aligned}
\end{equation}

Because the Hessian matrix {\small$H_x$} is a real-valued symmetric matrix, we can use the eigenvalue decomposition to decompose {\small$H_x$} as {\small$H_x=V\Lambda V^{-1}=V\Lambda V^{T}$}.
In this scenario,  the perturbation {\small$\hat\delta$} can be further simplified as
\begin{equation}
\begin{small}
\label{supp_eqn:perturbation_3}
\begin{aligned}
\hat\delta
&= \lim_{m\to+ \infty } \alpha [I+(I+\alpha H_x)+\cdots +(I+\alpha H_x)^{m-1}]g_x + \hat{\mathcal{R}}_2(\alpha)
\\
&=  \lim_{m\to+ \infty } \alpha [I+(I+\alpha V\Lambda V^T)+\cdots +(I+\alpha V\Lambda V^T)^{m-1}]g_x + \hat{\mathcal{R}}_2(\alpha)\\
&=\lim_{m\to+ \infty }  \alpha V[I+(I+\alpha\Lambda)+\cdots +(I+\alpha \Lambda )^{m-1}]V^Tg_x + \hat{\mathcal{R}}_2(\alpha).
\end{aligned}
\end{small}
\end{equation}

For simplicity, let  {\small$D=\alpha(I+(I+\alpha\Lambda)+\cdots +(I+\alpha \Lambda )^{m-1})$}.
Then, when the step number {\small$m \to + \infty$}, the $k$-th diagonal element {\small$\lim_{m\to+ \infty }  D_{kk}$} can be written as
\begin{equation}
\begin{aligned}
\label{supp_eqn:diagonal_2}
\lim_{m\rightarrow +\infty}D_{kk}
&=\lim_{m\rightarrow +\infty} \big[\alpha(1+(1+\alpha \lambda_k)+\cdots+(1+\alpha \lambda_k)^{m-1})\big]\\
&=\frac{\lim_{m\rightarrow +\infty} (1+\alpha\lambda_k)^m -1}{\lambda_k}\\
&=\frac{\lim_{m\rightarrow +\infty} (1+\frac{\alpha m\lambda_k}{m})^m-1}{\lambda_k}\\
&=\frac{\exp(\alpha m\lambda_k)-1}{\lambda_k}\\
&= \frac{\exp(\beta\lambda_k)-1}{\lambda_k}.
\end{aligned}
\end{equation}

Then, combining Eq.~\eqref{supp_eqn:diagonal_2} and Eq.~\eqref{supp_eqn:perturbation_2}, the perturbation {\small$\hat\delta$} can be written as
\begin{equation}
\label{supp_eqn:perturbation_5}
\begin{aligned}
\hat\delta&=\lim_{m\rightarrow +\infty} VDV^Tg_x + \hat{\mathcal{R}}_2(\alpha)\\
&=\lim_{m\rightarrow +\infty} VDV^T(\sum_{i=1}^n\gamma_iv_i) +  \hat{\mathcal{R}}_2(\alpha) \\
&=\sum_{i=1}^{n} \lim_{m\rightarrow +\infty} D_{ii} v_iv_i^T\sum_{k=1}^n\gamma_k v_k +  \hat{\mathcal{R}}_2(\alpha) \\
&=\sum_i  \frac{\exp(\beta\lambda_i)-1}{\lambda_i} \gamma_iv_i +  \hat{\mathcal{R}}_2(\alpha).
\end{aligned}
\end{equation}

Furthermore, we use the first-order Taylor expansion to decompose the gradient {\small$g_{x+\hat\delta}$} of the loss~\textit{w.r.t.} the adversarial example {\small$x+\hat\delta$}.
\begin{equation}
\label{eqn:grad_taylor}
g_{x+\hat\delta} = g_x + H_x \;\hat\delta + R_2 (\hat\delta).
\end{equation}

Substituting Eq.~\eqref{supp_eqn:perturbation_5} back to Eq.~\eqref{eqn:grad_taylor}, the gradient {\small$g_{x+\hat\delta}$} can be written as
\begin{equation}
\label{eqn:grad_taylor_2}
\begin{aligned}
g_{x+\hat\delta} &= g_x + H_x  \bigg( \sum_i \frac{\exp(\beta\lambda_i)-1}{\lambda_i} \gamma_iv_i + \hat{\mathcal{R}}_2(\alpha) \bigg) + R_2 (\hat\delta)
\\
& = \sum_{i=1}^n \gamma_iv_i  + H_x \sum_i \frac{\exp(\beta\lambda_i)-1}{\lambda_i} \gamma_iv_i
+ H_x \hat{\mathcal{R}}_2(\alpha) + R_2 (\hat\delta)
\\
& = \sum_{i=1}^n \gamma_iv_i  + \sum_i \frac{\exp(\beta\lambda_i)-1}{\lambda_i} \gamma_i (H_x v_i)
+ H_x \hat{\mathcal{R}}_2(\alpha) + R_2 (\hat\delta)
\\
& = \sum_{i=1}^n \gamma_iv_i  + \sum_i \frac{\exp(\beta\lambda_i)-1}{\lambda_i} \gamma_i (\lambda_i v_i)
+ H_x \hat{\mathcal{R}}_2(\alpha) + R_2 (\hat\delta)
\\
& \approx  \sum_i \exp(\beta\lambda_i) \;\gamma_i  v_i.
\end{aligned}
\end{equation}

Hence, Theorem~\ref{theorem:adv_per_2} is proven.
\end{proof}
\section{Detailed explanation for Corollary~\ref{theorem:adv_per_norm}}

In this section, we consider $\ell_2$ attacks and $\ell_\infty$ attacks.
As two typical attacking methods, the $\ell_2$ attack and the $\ell_\infty$ attack usually regularize/normalize the adversarial strength in each step by applying {\small$g_{x+ \delta^{(t)}}^{(\ell_2)}=\frac{\partial}{\partial x} L(f(x+ \delta^{(t)}), y)/\Vert \frac{\partial}{\partial x} L(f(x+ \delta^{(t)}), y) \Vert$},
and {\small$g_{x+ \delta^{(t)}}^{(\ell_\infty)}=\text{sign} (\frac{\partial}{\partial x} L(f(x+ \delta^{(t)}), y))$}, respectively.
In fact, for the $\ell_\infty$ attack, we can roughly consider that only the gradient component {\small$o_{x}^{T}g_{x+ \delta^{(t)}}^{(\ell_\infty)} o_{x}$} disentangled from {\small$g_{x+ \delta^{(t)}}^{(\ell_\infty)}$} along {\small$\frac{\partial}{\partial x} L(f(x), y)$} is effective,
where {\small$o_{x}=\frac{\partial}{\partial x} L(f(x), y)/ \Vert \frac{\partial}{\partial x} L(f(x), y) \Vert$} is the unit vector in the direction of {\small$\frac{\partial}{\partial x} L(f(x), y)$}.
However, it is quite complex to analyze the exact attacking behavior.
Therefore, in Corollary~\ref{theorem:adv_per_norm}, we just normalize the perturbation in Theorem~\ref{theorem:adv_per_2} to roughly approximate the regularization/normalization of perturbations in  $\ell_2$ attacks and $\ell_\infty$ attacks.

\textbf{Corollary 1.}
\textit{
(Normalized perturbation of the infinite-step attack).
Based on Theorem~\ref{theorem:adv_per_2}, if we ignore residual terms {\small$R_2(\hat\delta)$},
the perturbation of the infinite-step $\ell_2$ attack generated via  {\small$g_{x+ \delta^{(t)}}^{(\ell_2)}$}, and the perturbation of the infinite-step $\ell_\infty$ attack generated via  {\small$g_{x+ \delta^{(t)}}^{(\ell_\infty)}$} can be approximated as follows.
Here, $C \in \mathbb{R}$ is roughly considered to be proportional to the step number $m$ of the adversarial attack.}
\vspace{-6pt}

\begin{small}
\begin{equation}
\label{eqn:solution_adv_per_norm}
\hat\delta^{\text{(norm)}} \approx  C \cdot \hat\delta/{\Vert \hat\delta \Vert}
=  C \cdot \sum\nolimits_{i=1}^{n}  \frac{\exp(\beta \lambda_{i})-1}{\lambda_{i}} \gamma_{i} v_{i} \Bigg/
{ \sqrt {\sum\nolimits_{i=1}^{n}  (\frac{\exp(\beta \lambda_{i})-1}{\lambda_{i}}\gamma_{i})^{2}}}.
\end{equation}\end{small}

The reason why the scalar $C$ is proportional to the step number $m$ of the adversarial attack is as follows.
According to the definition of the perturbation of the $\ell_2$ attack,
{\small$\hat\delta =\lim_{m\rightarrow +\infty} \sum\nolimits_{t=0}^{m-1}  \alpha \cdot  g_{x+ \delta^{(t)}}^{(\ell_2)}$},
and the definition of the perturbation of the $\ell_\infty$ attack,
{\small$\hat\delta =\lim_{m\rightarrow +\infty} \sum\nolimits_{t=0}^{m-1}  \alpha \cdot  g_{x+ \delta^{(t)}}^{(\ell_\infty)}$},
the adversarial strength in each step is regularized/normalized.
Hence, norms of perturbations generated via the $\ell_2$ attack and the $\ell_\infty$ attack increase along with the step number $m$.
In this way, we can roughly consider that $C$ is proportional to the step number $m$ of the adversarial attack.

\section{Proof of Assumption~\ref{assumption:sigmoid} in main paper}

In this section, we prove Assumption~\ref{assumption:sigmoid} in Section~\ref{sec:3.2} of the main paper.

\textbf{Assumption 2 in main paper}.
\textit{The analysis of binary classification based on a sigmoid function, {\small$f(x)= \frac{1}{1+exp(-z(x))}, z(x) \in \mathbb{R}$},
can also explain the multi-category classification with a softmax function, {\small$f(x) = \frac{exp(z’_{1})}{\sum_{i=1}^{c} exp(z’_{i})}, z’ \in \mathbb{R}^{c}$}, if the second-best category is much stronger than other categories.
In this case, attacks on the multi-category classification can be approximated by attacks on the binary classification between the best and the second-best categories,~\textit{i.e.,} {\small$f(x) \approx \frac{1}{1+exp(-z)}$}, subject to {\small$z=z'_1-z'_2\in \mathbb{R}$}.
$z'_1$ and $z'_{2}$ are referred to as network outputs corresponding to the best category and the second-best category, respectively.}

\begin{proof}
Given an input sample $x$ and a ReLU network $f$ trained for multi-category classification based on a softmax function,
let {\small$z'_{i}\in\mathbb{R}, 1\leq i \leq c$} denote the network output of the $i$-th confident category,~\textit{i.e., } {\small$z'_{1}>z'_{2}>\cdots>z'_{c}$}.
Then, the probability for the most confident category is given as follows.
\begin{equation}
\begin{aligned}
\label{eqn:best_cate1}
p_1&=\frac{\exp(z_1')}{\sum_{i=1}^c\exp(z_i')}\\
&=\frac{1}{\sum_{i=1}^c\exp(z_i'-z_1')}.
\end{aligned}
\end{equation}

When the second-best category is much stronger than other categories, we have $\forall i>2, \exp(z_i'-z_1')\ll\exp(z_2'-z_1')<\exp(z_1'-z_1')=1$.
In this way, Eq.~\eqref{eqn:best_cate1} can be re-written as
\begin{equation}
\label{eqn:best_cate2}
p_1=\frac{1}{\sum_{i=1}^c\exp(z_i'-z_1')}\approx \frac{1}{\exp(z_2'-z_1')+1}=\frac{1}{1+\exp(-(z_1'-z_2'))}.
\end{equation}

Let {\small$z=z_1'-z_2' \in \mathbb{R}$}, and we have {\small$f(x) = p_1 \approx \frac{1}{1+\exp(-z)}$}.
In this way, attacks on the multi-category classification can be approximated by attacks on the binary classification between the best and the second-best categories.
Hence, Assumption~\ref{assumption:sigmoid} in main paper is proven.
\end{proof}

\section{Proof of Lemma~\ref{lemma:hessian} in main paper}

In this section, we prove Lemma~\ref{lemma:hessian} in Section~\ref{sec:3.2} of the main paper.

\textbf{Lemma 1 in main paper}.
\textit{Let us focus on the cross-entropy loss {\small$L(f(x),y)$}.
If the classification is based on a softmax operation, then the Hessian matrix {\small$H_z =\frac{\partial^{2}}{\partial z \partial z^{T}}   L(f(x),y)$} is positive semi-definite.
If the classification is based on a sigmoid operation, the scalar {\small$H_z \geq g_z^{2} \geq 0$}, as long as the attacking has not finished (still {\small$z(x) \cdot y>0, y \in \{-1,+1\}$}).
Here, {\small$g_z = \frac{\partial}{\partial z} L(f(x),y)  \in \mathbb{R}$}.}

\begin{proof}
Let us first consider the classification based on a softmax operation.
Given an input sample $x$ and a ReLU network $f$, the output of the network can be written as $z(x)=f(x) \in\mathbb{R}^c$.
In this case, let {\small$p_i=\exp(z_i)/\sum_{k=1}^c\exp{(z_{k})}$} denote the probability that the network $f$ classifies the input sample $x$ as the $i$-th category, where $z_i\in\mathbb{R} $ is referred to as the network output of the $i$-th category.
Then, the cross-entropy loss can be represented as  {\small$L(f(x),y)=-\sum_{i=1}^c y_i \log(p_i)$}, where $y_i\in \{0,1\}$ denotes the label.
Here, let $i$ denote the ground-truth label for the input sample $x$,~\textit{i.e.,} $y_i=1$, and $\forall k \neq i, y_k=0$.
In this way, the gradient of the loss {\small$L(f(x),y)$}~\textit{w.r.t} the network output $z(x)\in\mathbb{R}^c$ is given as
\begin{equation}
\label{eqn:grad_z}
g_{z}=\frac{\partial L(f(x),y)}{\partial z(x)}=-\frac{y_i}{p_i}\cdot\frac{\partial p_i}{\partial z(x)}=-\frac{1}{p_i}\cdot\frac{\partial p_i}{\partial z(x)}.
\end{equation}

Let us first focus on the network output $z_i$~\textit{w.r.t.} the ground-truth category $i$.
In this scenario, we have
\begin{equation}
\begin{aligned}
\label{eqn:grad_z_i}
\frac{\partial p_i}{\partial z_i}&=\frac{\exp(z_i)(\sum_{k=1}^c\exp(z_{k}))-\exp(z_i)\exp(z_i)}{(\sum_{k=1}^c\exp(z_{k}))^2}\\
&=\frac{\exp(z_i)}{\sum_{k=1}^c\exp(z_{k})}\cdot(1-\frac{\exp(z_i)}{\sum_{k=1}^c\exp(z_{k})})\\
&=p_i(1-p_i)=p_i(y_i-p_i).
\quad // \quad y_i=1
\end{aligned}
\end{equation}

As for $z_k, k\neq i$, we have
\begin{equation}
\begin{aligned}
\label{eqn:grad_z_j}
\frac{\partial p_i}{\partial z_k}&=\frac{-\exp(z_i)\exp(z_k)}{(\sum_{k'=1}^c\exp(z_{k'}))^2}\\
&=-\frac{\exp(z_i)}{\sum_{k'=1}^c\exp(z_{k'})}\cdot\frac{\exp(z_k)}{\sum_{k'=1}^c\exp(z_{k'})}\\
&=-p_ip_k=p_i(y_k-p_k).
\quad // \quad y_k=0
\end{aligned}
\end{equation}

Combining Eq.~\eqref{eqn:grad_z}, Eq.~\eqref{eqn:grad_z_i}, and Eq.~\eqref{eqn:grad_z_j}, we have
\begin{equation}
\label{eqn:grad_z_new}
g_z=\textbf{p}-\textbf{y},
\end{equation}
where {\small$\textbf{p}=[p_1, p_2, \cdots, p_c] \in \mathbb{R}^{c}$}, and {\small$\textbf{y}=[y_1, y_2, \cdots, y_c] \in \mathbb{R}^{c}$}.

In this way, based on Eq.~\eqref{eqn:grad_z_new}, the Hessian matrix {\small$H_z \overset{\text{def}}{=}\frac{\partial^{2}}{\partial z \partial z^{T}}  L(f(x),y)$} of the loss~\textit{w.r.t} the network output $z(x)$ can be written as
\begin{equation}
\begin{aligned}
\label{eqn:hes_z_def}
H_z&=\frac{\partial^{2} L(f(x),y) }{\partial z \partial z^{T}}= \frac{\partial g_z}{\partial z(x)}\\
&=\frac{\partial (\textbf{p}-\textbf{y})}{\partial z(x)}=\frac{\partial \textbf{p}}{\partial z(x)}.
\end{aligned}
\end{equation}

According to Eq.~\eqref{eqn:grad_z_i} and Eq.~\eqref{eqn:grad_z_j}, we have {\small$\frac{\partial p_i}{\partial z_i}=p_i(1-p_i)=p_i-p_i^2$}, and {\small$\forall  k\neq i, \frac{\partial p_i}{\partial z_k}=-p_ip_k$}.
Then, the Hessian matrix $H_z$ can be re-written as
\begin{equation}
\label{eqn:hes_z}
H_z=\frac{\partial p}{\partial z(x)}=\text{diag}([p_1, p_2, \cdots, p_c])-\textbf{p}\textbf{p}^T .
\end{equation}

In order to prove the Hessian matrix $H_z$ is positive semi-define, we need to verify that all eigenvalues of the Hessian matrix $H_z$ are non-negative.
To this end, we use Gershgorin Circle theorem to estimate the bound of eigenvalues.
Specifically, Eq.~\eqref{eqn:hes_z} shows that for the $k$-th row of the Hessian matrix $H_z$, the $k$-th diagonal element of the Hessian matrix $H_z$ is {\small$p_i(1-p_i)$}, and the sum of absolute values of non-diagonal elements in the $k$-th row is {\small$\sum_{k'=1,k'\neq k}^c|p_kp_{k'}|=p_k(1-p_k)$}.
In this way, according to Gershgorin Circle theorem, each eigenvalue $\lambda$ of the Hessian matrix $H_z$ satisfies $0\le\lambda\le \max_{k} 2p_k(1-p_k)$.
In other words, all eigenvalues of $H_z$ are non-negative.
Hence, the Hessian matrix $H_z$ is proven to be positive semi-definite.

Moreover, let us focus on the classification based on a sigmoid operation.
In this case, the network output $z(x)\in\mathbb{R}$ is a scalar, and the cross-entropy loss can be represented as {\small$L(f(x),y)=-\log \frac{\exp{(z(x)\cdot y)}}{1+\exp(z(x)\cdot y)}$}, where $y\in\{-1,+1\}$.
Then, the gradient of the loss {\small$L(f(x),y)$}~\textit{w.r.t} the network output $z(x)\in\mathbb{R}$ is given as
\begin{equation}
\begin{aligned}
\label{eqn:bin-grad_z}
g_z&=\frac{\partial L(f(x),y)}{\partial z(x)}\\
&=-\frac{1+\exp(z(x)\cdot y)}{\exp(z(x)\cdot y)}\cdot\frac{\exp(z(x)\cdot y)}{(1+\exp(z(x)\cdot y))^2} \cdot y\\
&=-\frac{y}{1+\exp(z(x)\cdot y)} \in\mathbb{R}.
\end{aligned}
\end{equation}

Based on Eq.~\eqref{eqn:bin-grad_z}, {\small$H_z \overset{\text{def}}{=}\frac{\partial^{2}}{\partial z \partial z^{T}}  L(f(x),y) \in\mathbb{R}$} of the loss~\textit{w.r.t} the network output $z(x)$ can be written as
\begin{equation}
\begin{aligned}
\label{eqn:bin-hes_z}
H_z&=\frac{\partial g_z}{\partial z(x)}\\
&=-y\cdot -\frac{y\exp(z(x)\cdot y)}{(1+\exp(z(x)\cdot y))^2}\\
&=\frac{y^2\exp(z(x)\cdot y)}{(1+\exp(z(x)\cdot y))^2} \geq 0.
\end{aligned}
\end{equation}

Combining Eq.~\eqref{eqn:bin-grad_z} and Eq.~\eqref{eqn:bin-hes_z}, we have
\begin{equation}
\label{eqn:bin-Hzgz2}
\frac{H_z}{g_z^2}=\frac{y^2\exp(z(x)\cdot y)}{(1+\exp(z(x)\cdot y))^2}\cdot (-\frac{1+\exp(z(x)\cdot y)}{y})^2=\exp(z(x)\cdot y)
\end{equation}

If the attacking has not finished yet,~\textit{i.e.,} {\small$z(x)\cdot y>0$}, then we have {\small$\exp(z(x) \cdot y)>1$}, thereby {\small$H_z>g_z^2$}.
Based on Eq.~\eqref{eqn:bin-grad_z}, we obtain {\small$g_z^2=y^{2}/{(1+\exp(z\cdot y))^2}\in \mathbb{R} >0$}, thereby {\small$H_z>g_z^2>0$}.

Thus, Lemma~\ref{lemma:hessian} in Section~\ref{sec:3.2} is proven.
\end{proof}

\section{Proof of Theorem~\ref{theorem2}}
In this section, we prove Theorem~\ref{theorem2} in Section~\ref{sec:3.2} of the main paper, which explains training effects of the adversarial perturbation $\hat\delta$ in Theorem~\ref{theorem:adv_per_2} on adversarial training.

Specifically, if we use vanilla training to fine-tune the network on the original input sample $x$ for a single step, then the gradient of the loss~\textit{w.r.t.} the weight  {\small$W$} is given as {\small$g_{W}= \frac{\partial}{\partial W} L(f(x),y)$}.
In comparison, if we train the network on the adversarial example {\small$x+\hat\delta$} for a single step, then we will get the gradient {\small$g^{\text{(adv)}}_{W} =\frac{\partial}{\partial W}  L(f(x+\hat\delta),y)$}.
In this way, {\small $\Delta g_{W} = g^{\text{(adv)}}_{W}- g_{W}$} denotes additional effects of adversarial training on the gradient.
\begin{equation}
\label{eqn:delta_g_w_pre}
\begin{aligned}
\Delta g_{W} &= g^{\text{(adv)}}_{W}- g_{W}= \frac{\partial}{\partial W} L(f(x+\hat\delta),y) -\frac{\partial}{\partial W}  L(f(x),y)\\
& = x(H_h \Delta h)^T+\hat\delta (g_h+H_h\Delta h)^T,
\end{aligned}
\end{equation}
where {\small$\Delta h=W^T\hat\delta$} denotes the change of the intermediate-layer feature $h$ caused by the perturbation {\small$\hat\delta$}, where {\small$W=W_{j}^{T}\Sigma_{j-1}\cdots\Sigma_{2}W_2^{T}\Sigma_{1}W_1^{T}$}.
It is because, according to Assumption 1 in the main paper, we simplify our research into an idealized adversarial attack, whose adversarial perturbation does not significantly change gating states in gating layers.
In this way, to simplify the proof, we can roughly represent the output feature $h$ of the $j$-th layer as {\small$h\approx W^{T} x +b$}.
{\small$H_{h} = \frac{\partial^{2} }{\partial h \partial h^{T}} L(f(x),y)$} represents the Hessian matrix of the loss~\textit{w.r.t.} the output feature $h$ of the $j$-th linear layer.
{\small$g_h = \frac{\partial}{\partial h}  L(f(x),y)$} indicates the gradient of the loss~\textit{w.r.t.} the feature $h$.

\begin{proof}

According to the chain rule, the gradient of the weight $W$ can be written as {\small$g_w=(\frac{\partial L(f(x),y)}{\partial W^T})^T=(\frac{\partial L(f(x),y)}{\partial h^T}\frac{\partial h}{\partial W^T})^T$}.
Without loss of generality, let us first consider the $i$-th dimension of $h$,~\textit{i.e.} {\small$h_i=W_i^Tx\in\mathbb{R}$}, which is only related to the $i$-th row of {\small$W^T$}.
Thus, the gradient of the loss~\textit{w.r.t.} {\small$W_i^T \in\mathbb{R}^{1\times n}$} is given as
\begin{equation}
\begin{aligned}
\label{eqn:gradient_of_wi}
\frac{\partial L(f(x),y)}{\partial W_i^T}=\frac{\partial L(f(x),y)}{\partial h_i}\frac{\partial h_i}{\partial W_i^T}=\frac{\partial L(f(x),y)}{\partial h_i}x^T.
\end{aligned}
\end{equation}

In this way, combining all dimensions of $h$, we have
\begin{equation}
\begin{aligned}
\label{eqn:gradient_of_w_2}
\frac{\partial L(f(x),y)}{\partial W^T} &=[\frac{\partial L(f(x),y)}{\partial W_1^T},\frac{\partial L(f(x),y)}{\partial W_2^T},\cdots,\frac{\partial L(f(x),y)}{\partial W_D^T}]^T\\
&=\frac{\partial L(f(x),y)}{\partial h}x^T.
\end{aligned}
\end{equation}

In other words, the gradient {\small$g_w$} of the loss~\textit{w.r.t} the weight $W$ can be represented as
\begin{equation}
\begin{aligned}
\label{eqn:gradient_of_w}
g_W &=(\frac{\partial L(f(x),y)}{\partial W^T})^T
=(\frac{\partial L(f(x),y)}{\partial h}x^T)^T\\
&=x\frac{\partial L(f(x),y)}{\partial h^T}
=xg_h^T.
\end{aligned}
\end{equation}

According to Eq.~\eqref{eqn:gradient_of_w}, the gradient {\small$g^{\text{(adv)}}_{W} = \frac{\partial}{\partial W}  L(f(x+\hat\delta),y)$} can be re-written as follows, where
{\small$g_{h+\Delta h} =  \frac{\partial}{\partial h+\Delta h}  L(f(x+\hat\delta),y)$}.
\begin{equation}
\begin{aligned}
\label{eqn:gradient_of_w_adv}
g^{\text{(adv)}}_{W}=(x+\hat\delta)(g_{h+\Delta h})^T.
\end{aligned}
\end{equation}

We further use the first-order Taylor expansion to decompose the gradient {\small$g_{h+\Delta h}$}, where {\small$R^{\text{(grad)}}_2(\Delta h)$} denotes the terms no less than the second order.
\begin{equation}
\label{eqn:gradient_of_h_adv}
g_{h+\Delta h}=g_h+H_h\Delta h +R^{\text{(grad)}}_2(\Delta h).
\end{equation}

Substituting Eq.~\eqref{eqn:gradient_of_h_adv} back to Eq.~\eqref{eqn:gradient_of_w_adv}, the gradient {\small$g^{\text{(adv)}}_{W}$} can be represented as
\begin{equation}
\begin{aligned}
\label{eqn:gradient_of_w_adv_2}
g^{\text{(adv)}}_{W}=(x+\hat\delta)\bigg(g_h+H_h\Delta h +R^{\text{(grad)}}_2(\Delta h) \bigg)^T.
\end{aligned}
\end{equation}

Thus, the additional effects of adversarial training on the gradient can be written as follows.
\begin{equation}
\begin{aligned}
\label{eqn:change_gradient_w}
\Delta g_W&=g^{\text{(adv)}}_{W}-g_W
\\
& = x(H_h\Delta h)^T+\hat\delta (g_h+H_h\Delta h)^T + (x+\hat\delta) \bigg(R^{\text{(grad)}}_2(\Delta h) \bigg)^T
\\
&\approx x(H_h\Delta h)^T+\hat\delta (g_h+H_h\Delta h)^T.
\end{aligned}
\end{equation}
\end{proof}

\begin{assumption}[in Appendix]
\label{assumption:w}
Given a ReLU network $f$, let {\small$W=W_{j}^{T}\Sigma_{j-1}\cdots\Sigma_{2}W_2^{T}\Sigma_{1}W_1^{T} \in \mathbb{R}^{n \times D}$}.
Because each column of $W^T W$ is a high-dimensional vector, we can roughly consider that any pair of columns in $W^T W$ is linearly dependent.
Thus, $W^T W$ is a full rank matrix, and there exists $(W^T W)^{-1}$.
\end{assumption}

\begin{lemma}[in Appendix]
\label{lemma:H_x}
Based on Assumption~\ref{assumption:sigmoid} in the main paper, the Hessian matrix {\small$H_{h}\overset{\text{def}}{=}\frac{\partial^{2} }{\partial h \partial h^{T}} L(f(x),y)$} of the loss~\textit{w.r.t.} the feature $h$ can be represented as {\small$H_h = H_z\tilde{g}_h\tilde{g}_h^T$}, where {\small$\tilde{g}_h = \frac{\partial}{\partial h}  z(x)$} indicates the gradient of the network output {\small$z(x)$}~\textit{w.r.t.} the feature $h$, and {\small$H_z\in \mathbb{R}$}.
The Hessian matrix {\small$H_{x}\overset{\text{def}}{=}\frac{\partial^{2} }{\partial x \partial x^{T}} L(f(x),y)$} can be represented as {\small$H_x = H_z\tilde{g}_x\tilde{g}_x^T = W H_h W^T $}.
\end{lemma}

\begin{proof}
\begin{equation}
\begin{small}
\label{supp_eqn:hessian_lemma_1}
\begin{aligned}
H_h&=\frac{\partial^2 L(f(x),y)}{\partial h\partial h^T}
=\frac{\partial (\frac{\partial L(f(x),y)}{\partial z(x)}\frac{\partial z(x)}{\partial h^T})^T}{\partial h^T}\\
&=\frac{ (\frac{\partial z(x)}{\partial h^T})^T}{\partial h^T}
\cdot
\frac{\partial L(f(x),y)}{\partial z(x)}
+(\frac{\partial z(x)}{\partial h^T})^T
\cdot
\frac{\partial(\frac{\partial L(f(x),y)}{\partial z(x)})}{\partial h^T} \\
&=(\frac{\partial z(x)}{\partial h^T})^T \frac{\partial^2 L(f(x),y)}{\partial z(x)\partial z(x)} \frac{\partial z(x)}{\partial h^T}\\
&=H_z\tilde{g}_h\tilde{g}_h^T. \quad // \quad z\in \mathbb{R}, H_z\in \mathbb{R}, \text{according to Assumption~\ref{assumption:sigmoid} in the main paper}
\end{aligned}
\end{small}
\end{equation}

Similarly, the Hessian matrix $H_{x}$ can be written as
\begin{equation}
\begin{small}
\label{supp_eqn:hessian_lemma_2}
\begin{aligned}
H_x&=\frac{\partial^2 L(f(x),y)}{\partial x\partial x^T}\\
&=\frac{\partial (\frac{\partial L(f(x),y)}{\partial z(x)}\frac{\partial z(x)}{\partial x^T})^T}{\partial x^T}\\
&=\frac{ (\frac{\partial z}{\partial x^T})^T}{\partial x^T}
\cdot
\frac{\partial L(f(x),y)}{\partial z(x)}
+(\frac{\partial z(x)}{\partial x^T})^T
\cdot
\frac{\partial(\frac{\partial L(f(x),y)}{\partial z(x)})}{\partial x^T} \\
&=(\frac{\partial z(x)}{\partial x^T})^T \frac{\partial^2 L(f(x),y)}{\partial z\partial z(x)} \frac{\partial z(x)}{\partial x^T}\\
&=H_z\tilde{g}_x\tilde{g}_x^T.
\quad // \quad z\in \mathbb{R}, H_z\in \mathbb{R}, \text{according to Assumption~\ref{assumption:sigmoid} in the main paper}
\end{aligned}
\end{small}
\end{equation}

Furthermore, we use the chain rule to re-write the gradient {\small$\tilde{g}_x$} of the network ouput $z(x)$~\textit{w.r.t} the input sample $x$.
\begin{equation}
\begin{aligned}
\label{supp_eqn:gxgy}
\tilde{g}_x&=(\frac{\partial z(x)}{\partial x^T})^T=(\frac{\partial z(x)}{\partial h^T}\frac{\partial h}{\partial x^T})^T\\
&=(\tilde{g}_h^T W^T)^T=W \tilde{g}_h.
\end{aligned}
\end{equation}

In this way, substituting Eq.~\eqref{supp_eqn:gxgy} back to Eq.~\eqref{supp_eqn:hessian_lemma_2}, we get
\begin{equation}
\begin{aligned}
\label{eqn:H_x_H_h}
H_x=H_z\tilde{g}_x\tilde{g}_x^T=H_z W\tilde{g}_h(W\tilde{g}_h)^T=WH_h W^T.
\end{aligned}
\end{equation}

Thus, Lemma~\ref{lemma:H_x} in Appendix is proven.
\end{proof}

\begin{lemma}[in Appendix]
\label{lemma:H_x_g_w}
Let {\small$\tilde{g}_x = \frac{\partial}{\partial x}  z(x)$} denote the gradient of the network output $z$~\textit{w.r.t} the input sample $x$, and {\small$\mathcal{A} =\beta H_z\Vert\tilde{g}_x\Vert^2 \in \mathbb{R} $}.
Then, we have
\begin{equation}
H_x \; \Delta g_{W} =(e^{\mathcal{A}}-1)H_x xg_h^T + \frac{1}{H_z\Vert\tilde{g}_x\Vert^2}(e^{2\mathcal{A}}-e^{\mathcal{A}})H_xg_x g_h^T.
\end{equation}
\end{lemma}

\begin{proof}
To prove Lemma~\ref{lemma:H_x_g_w} in Appendix, we multiply {\small$H_x$} on both sides of Eq.~\eqref{eqn:delta_g_w_pre}.
\begin{equation}
\label{eqn:delta_g_w_expansion}
\begin{aligned}
H_x \; \Delta g_{W} &= H_x \; (g^{\text{(adv)}}_{W}- g_{W})
 =H_x \; x(H_h \Delta h)^T+ H_x \; \hat\delta (g_h+H_h\Delta h)^T.
\end{aligned}
\end{equation}

Let us first focus on the first term {\small$H_x x(H_h \Delta h)^T$} in Eq.~\eqref{eqn:delta_g_w_expansion}.
According to Eq.~\eqref{supp_eqn:perturbation_approx} and  Lemma~\ref{lemma:H_x} in Appendix, we can write $\Delta h$ as follows.
Note that since the step size is infinitesimal {\small$\alpha \to 0$}, the perturbation $\hat\delta$ is mainly dominated by the term  {\small$[I+(I+\alpha H_x)+\cdots+(I+\alpha H_x)^{m-1}]g_x$}.
Thus, we ignore the error $\hat{\mathcal{R}}_2(\alpha)$ in Eq.~\eqref{supp_eqn:perturbation_2}.
\begin{equation}
\begin{small}
\label{eqn:delta_y}
\begin{aligned}
\Delta h&=W^T \hat\delta=\alpha W^T[I+(I+\alpha H_x)+\cdots+(I+\alpha H_x)^{m-1}]g_x\\
&=\alpha W^T[I+(I+\alpha WH_hW^T)+\cdots+(I+\alpha WH_hW^T)^{m-1}]g_x
\quad // \quad \text{according to Eq.~\eqref{eqn:H_x_H_h}}
\\
&=\alpha[W^T+W^T(I+\alpha WH_hW^T)+\cdots+W^T(I+\alpha WH_hW^T)^{m-1}]g_x.\\
\end{aligned}
\end{small}
\end{equation}

To simplify $\Delta h$, we apply the mathematical induction to prove that {\small$\forall t,  1\leq t \leq m, W^T(I+\alpha WH_h W^T)^t=(I+\alpha W^TWH_h)^t W^T$}.

\emph{Base case:}
When {\small$t=1$}, {\small$W^T(I+\alpha WH_hW^T)=(W^T+\alpha W^TWH_hW^T)=(I+\alpha W^TWH_h)W^T$}.

\emph{Inductive step:}
For $t>1$, assuming {\small$W^T(I+\alpha WH_hW^T)^{t-1}=(I+\alpha W^TWH_h)^{t-1}W^T$}, we have
\begin{equation}
\label{eqn:wTwH}
\begin{aligned}
    W^T(I+\alpha WH_hW^T)^{t}&=W^T(I+\alpha WH_hW^T)^{t-1}(I+\alpha WH_hW^T)\\
    &=(I+\alpha W^TWH_h)^{t-1}W^T(I+\alpha WH_hW^T)\\
    &=(I+\alpha W^TWH_h)^{t-1}(I+\alpha W^TWH_h)W^T\\
    &=(I+\alpha W^TWH_h)^{t}W^T
\end{aligned}
\end{equation}

\emph{Conclusion:} Since both the base case and the inductive step have been proven to be true, we obtain {\small$W^T(I+\alpha WH_hW^T)^t=(I+\alpha W^TWH_h)^tW^T$}.

In this way, we combine Eq.~\eqref{eqn:delta_y} and Eq.~\eqref{eqn:wTwH}.
The change of the intermediate-layer feature $h$ caused by the perturbation $\hat\delta$ can be represented as
\begin{equation}
\label{eqn:delta_y_2}
\Delta h=\alpha[I+(I+\alpha W^TWH_h)+\cdots+(I+\alpha W^TWH_h)^{m-1}]W^Tg_x.
\end{equation}

Multiply $(I+\alpha W^TWH_h)$ on the both sides of Eq.~\eqref{eqn:delta_y_2}, and we get
\begin{equation}
\label{eqn:delta_y_sum}
(I+\alpha W^TWH_h)\Delta h=\alpha[(I+\alpha W^TWH_h)+\cdots+(I+\alpha W^TWH_h)^{m}]W^Tg_x.
\end{equation}

Then, the difference between Eq.~\eqref{eqn:delta_y_sum} and Eq.~\eqref{eqn:delta_y_2} is
\begin{equation}
\label{eqn:delta_y_sum_minus}
\begin{aligned}
(I+\alpha W^TWH_h)\Delta h - \Delta h &= \alpha[(I+\alpha W^TWH_h)^{m}-I]W^Tg_x
\\
\Rightarrow \alpha W^TWH_h \Delta h &= \alpha[(I+\alpha W^TWH_h)^{m}-I]W^Tg_x
\\
\Rightarrow W^T W H_h \Delta h &= [(I+\alpha W^TWH_h)^{m}-I]W^Tg_x.
\end{aligned}
\end{equation}

Therefore, based on Eq.~\eqref{eqn:delta_y_sum_minus}, we have
\begin{equation}
\begin{aligned}
    \label{eqn:second_term_1}
    (H_h\Delta h)^T W^T W&=(W^T W H_h\Delta h)^T\\
    &=([(I+\alpha W^T W H_h)^{m}-I]W^Tg_x)^T\\
    &=g_x^T W [(I+\alpha W^T W H_h)^{m}-I]^T\\
    &=g_h^T W^T W [(I+\alpha W^TW H_h)^{m}-I]^T\\
    &=g_h^T W^T W (I+\alpha W^TW H_h)^{m}-g_h^T W^T W.
\end{aligned}
\end{equation}

Furthermore, the term {\small$g_h^T W^T W (I+\alpha W^T W H_h)^{m}$} in Eq.~\eqref{eqn:second_term_1} can be re-written as
\begin{equation}
\begin{small}
\label{eqn:lemma_7}
\begin{aligned}
g_h^T W^T W (I+\alpha H_h W^T W)^{m}
& =g_h^T W^T W  (I+\alpha H_h W^T W ) (I+\alpha H_h W^T W)^{m-1}\\
& =g_h^T (W^T W +\alpha W^T W H_h W^T W) (I+\alpha H_h W^T W)^{m-1}\\
& =g_h^T (I +\alpha W^T W H_h)W^T W (I+\alpha H_h W^T W)^{m-1}\\
& =g_h^T (I +\alpha W^T W H_h) W^T W  (I+\alpha H_h W^T W ) (I+\alpha H_h W^T W)^{m-2}\\
& =g_h^T (I +\alpha W^T W H_h) (W^T W +\alpha W^T W H_h W^T W) (I+\alpha H_h W^T W)^{m-2}\\
& =g_h^T (I +\alpha W^T W H_h)^{2} W^T W (I+\alpha H_h W^T W)^{m-2}\\
&\cdots\\
& =g_h^T (I +\alpha W^T W H_h)^{m} W^T W.
\end{aligned}
\end{small}
\end{equation}

Based on Lemma~\ref{lemma:H_x} in Appendix, the term {\small$g_h^T (I +\alpha W^T W H_h)^{m}$} in Eq.~\eqref{eqn:lemma_7} can be simplified as
\begin{equation}
\begin{small}
\label{eqn:lemma_7_1}
\begin{aligned}
g_h^T (I +\alpha W^T W H_h)^{m}
& =g_h^T  (I +\alpha W^T W H_h) (I +\alpha W^T W H_h)^{m-1}
\\
\\
& =g_h^T  (I +\alpha W^T W H_z\tilde{g}_h \tilde{g}_h^{T}) (I +\alpha W^T W H_h)^{m-1}
\quad // \quad \text{according to Eq.~\eqref{supp_eqn:hessian_lemma_1}}
\\
\\
& = (g_h^T +\alpha \; H_z \tilde{g}_h^T W^T W \tilde{g}_h \;g_h^{T}) (I +\alpha W^T W H_h)^{m-1}
\\
\\
& = (1 +\alpha \mathcal{B})g_h^{T} (I +\alpha W^T W H_y)^{m-1}
\quad // \quad \mathcal{B}=H_z\tilde{g}_h^T W^T W\tilde{g}_h\in\mathbb{R}
\\
\\
& =(1 +\alpha \mathcal{B})g_h^{T}  (I +\alpha W^T W H_h) (I +\alpha W^T W H_h)^{m-2}
\\
\\
& =(1 +\alpha \mathcal{B}) g_h^T  (I +\alpha W^T W H_z\tilde{g}_h \tilde{g}_h^{T}) (I +\alpha W^T W H_h)^{m-2}
\\
\\
& = (1 +\alpha \mathcal{B}) (g_h^T +\alpha \; H_z \tilde{g}_h^T W^T W \tilde{g}_h \;g_h^{T}) (I +\alpha W^T W H_h)^{m-2}
\\
\\
& = (1 +\alpha \mathcal{B})^{2} g_h^{T} (I +\alpha W^T W H_y)^{m-2}
\\
& \cdots \\
& = (1 +\alpha \mathcal{B})^{m}g_h^{T}.
\end{aligned}
\end{small}
\end{equation}

In this way, combining Eq.~\eqref{eqn:lemma_7_1} and Eq.~\eqref{eqn:lemma_7}, we get
\begin{equation}
\begin{small}
\label{eqn:lemma_7_2}
\begin{aligned}
g_h^T W^T W (I+\alpha H_h W^T W)^{m}
& = g_h^T (I +\alpha W^T W H_h)^{m} W^T W \\
& = (1 +\alpha \mathcal{B})^{m}g_h^{T} W^T W.
\end{aligned}
\end{small}
\end{equation}

Substitute Eq.~\eqref{eqn:lemma_7_2} back to Eq.~\eqref{eqn:second_term_1}, and we get
\begin{equation}
\begin{aligned}
    \label{eqn:second_term_2}
    (H_h\Delta h)^T W^T W&=g_h^T W^T W (I+\alpha W^T WH_h)^{m}-g_h^T W^T W\\
    &=(1+\alpha \mathcal{B})^mg_h^T W^T W-g_h^T W^T W\\
    &=[(1+\alpha \mathcal{B})^m-1]g_h^T W^T W,
\end{aligned}
\end{equation}
where {\small$\mathcal{B}=H_z\tilde{g}_h^T W^T W\tilde{g}_h\in\mathbb{R}$}.

According to Assumption~\ref{assumption:w} in Appendix, there exists {\small$(W^T W)^{-1}$}.
Hence, multiplying {\small$(W^T W)^{-1}$} on both sides of Eq.~\eqref{eqn:second_term_2}, we get
\begin{equation}
\begin{aligned}
    \label{eqn:second_term_3}
    (H_h\Delta h)^T&=[(1+\alpha \mathcal{B})^m-1]g_h^T.
\end{aligned}
\end{equation}

Since the adversarial perturbation $\hat\delta$ is crafted via the infinite-step attack with the infinitesimal step size,~\textit{i.e.,} $m\to +\infty$, we have
\begin{equation}
\label{eqn:B_infty}
\lim_{m\to +\infty}(1+\alpha  \mathcal{B})^m = e^{\alpha m  \mathcal{B}}=e^{\beta  \mathcal{B}}.
\end{equation}

Hence, combining Eq.~\eqref{supp_eqn:gxgy} and Eq.~\eqref{eqn:B_infty}, we get
\begin{equation}
\label{eqn:A_infty}
\begin{aligned}
\lim_{m\to +\infty}(1+\alpha  \mathcal{B})^m &= e^{\beta  \mathcal{B}}
\\
&= e^{\beta  H_z\tilde{g}_h^T W^T W\tilde{g}_h}
 \\
 &=e^{\beta H_z\Vert\tilde{g}_x\Vert^2} = e^{ \mathcal{A}},
\end{aligned}
\end{equation}
where {\small$\mathcal{A} =e^{\beta H_z\Vert\tilde{g}_x\Vert^2} \in\mathbb{R}$}.

Multiply {\small$H_x x$} to both side of Eq.~\eqref{eqn:second_term_3}, and then the first term {\small$H_x x(H_h\Delta h)^T$} in Eq.~\eqref{eqn:delta_g_w_expansion} can be written as
\begin{equation}
\begin{aligned}
\label{eqn:second_term_4}
 H_x x(H_h\Delta h)^T&=\lim_{m\to +\infty} [(1+\alpha \mathcal{B})^m-1] H_x x g_h^T\\
 &= (e^{ \mathcal{A}}-1) H_x x g_h^T.
\end{aligned}
\end{equation}

Then, let us focus on the second term {\small$H_x \hat \delta (g_h+H_h\Delta h)^T$} in Eq.~\eqref{eqn:delta_g_w_expansion}.
Based on Eq.~\eqref{supp_eqn:perturbation_approx} and  Lemma~\ref{lemma:H_x} in Appendix, the second term {\small$H_x \hat \delta (g_h+H_h\Delta h)^T$} can be re-written as follows.
Note that since the step size is infinitesimal {\small$\alpha \to 0$}, the perturbation $\hat\delta$ is mainly dominated by the term  {\small$[I+(I+\alpha H_x)+\cdots+(I+\alpha H_x)^{m-1}]g_x$}.
Thus, we ignore the error $\hat{\mathcal{R}}_2(\alpha)$ in Eq.~\eqref{supp_eqn:perturbation_2}.
\begin{equation}
\begin{small}
\label{eqn:Delta x g_y^T}
\begin{aligned}
&H_x \hat \delta (g_h+H_h\Delta h)^T
\\
&=H_x\alpha[I+(I+\alpha WH_h W^T)+\cdots+(I+\alpha W H_h W^T)^{m-1}]g_x (g_h+H_h\Delta h)^T
\\
&=H_x\alpha[I+(I+\alpha WH_h W^T)+\cdots+(I+\alpha W H_h W^T)^{m-1}]W  g_h (g_h+H_h\Delta h)^T.
\end{aligned}
\end{small}
\end{equation}

For simplicity, let {\small$S =I+(I+\alpha WH_h W^T)+\cdots+(I+\alpha WH_h W^T)^{m-1}$}.
Then, multiply {\small$(I+\alpha WH_h W^T)$} to both sides of $S$, and we get
\begin{equation}
\begin{small}
\label{eqn:delta_x_g_y_3}
\begin{aligned}
(I+\alpha WH_h W^T)S
&=(I+\alpha WH_h W^T)+\cdots+(I+\alpha WH_h W^T)^{m}\\
\Rightarrow (I+\alpha WH_h W^T)S - S
&=(I+\alpha WH_h W^T)^{m}-I
\\
\Rightarrow H_x\alpha S&=(I+\alpha WH_h W^T)^{m}-I.
\quad // \quad \text{according to Eq.~\eqref{eqn:H_x_H_h}}
\end{aligned}
\end{small}
\end{equation}

Substituting Eq.~\eqref{eqn:delta_x_g_y_3} back to Eq.~\eqref{eqn:Delta x g_y^T}, we have
\begin{equation}
\label{eqn:delta_x_g_y_4}
\begin{aligned}
H_x \hat\delta  (g_h+H_h\Delta h)^T
&=[(I+\alpha WH_h W^T)^{m}-I] W g_h (g_h+H_h\Delta h)^T.
\end{aligned}
\end{equation}

To simplify Eq.~\eqref{eqn:delta_x_g_y_4}, let us first consider the term {\small$(I+\alpha WH_h W^T)^{m}-I$}.
Specifically, we apply the mathematical induction to derive the term {\small$(I+\alpha WH_h W^T)^{m}-I$}, and get {\small$\forall t, 1\leq t \leq m, (I+\alpha WH_h W^T)^{t}-I  =\frac{1}{\mathcal{B}}[(1+\alpha \mathcal{B})^t-1] WH_h W^T$}, where {\small$\mathcal{B}=H_z\tilde{g}_h^T W^T W\tilde{g}_h\in\mathbb{R}$}.

\emph{Base case:} When $t=1$,
\begin{equation}
\label{eqn:first_term_2}
\begin{aligned}
(I+\alpha WH_h W^T)^{1}-I
&=\alpha WH_h W^T  \\
& =\frac{1}{\mathcal{B}}[(1+\alpha \mathcal{B})^{1}-1] WH_h W^T.
\end{aligned}
\end{equation}

\emph{Inductive step:} For $t>1$, assuming {\small$(I+\alpha WH_h W^T)^{t-1}-I = \frac{1}{ \mathcal{B}}[(1+\alpha  \mathcal{B})^{t-1}-1]WH_h W^T$}, we get
\begin{equation}
\begin{small}
\label{eqn:t_term_2}
\begin{aligned}
(I+\alpha WH_h W^T)^{t}-I
& = (I+\alpha  WH_h W^T)^{t-1}(I+\alpha  WH_h W^T)-I
\\
\\
& = (I+\alpha  WH_h W^T)^{t-1}+(I+\alpha  WH_h W^T)^{t-1} \alpha  WH_h W^T-I
\\
\\
& = \frac{1}{\mathcal{B}}[(1+\alpha \mathcal{B})^{t-1}-1] WH_h W^T
\\
&\quad+(I+\alpha  WH_h W^T)^{t-1} \alpha  WH_h W^T.
\end{aligned}
\end{small}
\end{equation}

Since {\small$(I+\alpha WH_h W^T)^{t-1}-I = \frac{1}{ \mathcal{B}}[(1+\alpha  \mathcal{B})^{t-1}-1]WH_h W^T$}, we obtain
{\small$(I+\alpha WH_h W^T)^{t-1} = I+\frac{1}{\mathcal{B}}[(1+\alpha \mathcal{B})^{t-1}-1]WH_h W^T$}.
In this way, based on Lemma~\ref{lemma:H_x} in Appendix, Eq.~\eqref{eqn:t_term_2} can be further simplified as
\begin{equation}
\begin{small}
\label{eqn:t_term_3}
\begin{aligned}
&(I+\alpha WH_h W^T)^{t}-I\\
& =  \frac{1}{\mathcal{B}}\bigg[(1+\alpha \mathcal{B})^{t-1}-1\bigg] WH_h W^T\\
&\quad + \alpha \bigg[I+\frac{1}{\mathcal{B}}[(1+\alpha \mathcal{B})^{t-1}-1]WH_h W^T \bigg] WH_h W^T\\
&\\
& =  \frac{1}{\mathcal{B}}\bigg[(1+\alpha \mathcal{B})^{t-1}-1\bigg] WH_h W^T\\
&\quad + \alpha \bigg[WH_h W^T +\frac{1}{\mathcal{B}}[(1+\alpha \mathcal{B})^{t-1}-1]WH_h W^T WH_h W^T \bigg]\\
&\\
& = \frac{1}{\mathcal{B}}\bigg[(1+\alpha \mathcal{B})^{t-1}-1\bigg] WH_h W^T\\
&\quad + \alpha \bigg[WH_h W^T+\frac{1}{\mathcal{B}}[(1+\alpha \mathcal{B})^{t-1}-1] W H_z\tilde{g}_h\tilde{g}_h^T
W^{T}W H_z\tilde{g}_h\tilde{g}_h^T W^{T} \bigg]
\\
&\\
& = \frac{1}{\mathcal{B}}\bigg[(1+\alpha \mathcal{B})^{t-1}-1\bigg] WH_h W^T\\
&\quad + \alpha \bigg[WH_h W^T+\frac{1}{\mathcal{B}}[(1+\alpha \mathcal{B})^{t-1}-1] \,  \mathcal{B} \,W H_z\tilde{g}_h\tilde{g}_h^T W^{T} \bigg]
\quad // \quad \mathcal{B} = H_z\tilde{g}_h^T W^T W\tilde{g}_h\in\mathbb{R}
\\
&\\
& = \frac{1}{\mathcal{B}}\bigg[(1+\alpha \mathcal{B})^{t-1}-1\bigg] WH_h W^T
+ \alpha \bigg[WH_h W^T+ [(1+\alpha \mathcal{B})^{t-1}-1]  WH_h W^T \bigg]
\\
& = \frac{1}{\mathcal{B}} \bigg[(1+\alpha \mathcal{B})^{t-1}-1
+\alpha \mathcal{B} (1+\alpha \mathcal{B})^{t-1}  \bigg]WH_h W^T
\\
& =\frac{1}{\mathcal{B}}\bigg[(1+\alpha \mathcal{B})^{t}-1\bigg] WH_h W^T.
\end{aligned}
\end{small}
\end{equation}

\emph{Conclusion:} Since both the base case and the inductive step have been proven, we have
\begin{equation}
\label{eqn:t_term_final}
    (I+\alpha WH_h W^T)^{t}-I = \frac{1}{\mathcal{B}}\bigg[(1+\alpha \mathcal{B})^{t}-1\bigg] WH_h W^T,
\end{equation}
where $\mathcal{B}=H_z\tilde{g}_h^T W^T W\tilde{g}_h\in\mathbb{R}$.

Substituting Eq.~\eqref{eqn:t_term_final} back to Eq.~\eqref{eqn:delta_x_g_y_4}, we have
\begin{equation}
\label{eqn:delta_x_g_y_5}
\begin{aligned}
H_x\hat\delta (g_h+H_h\Delta h)^T
&=\frac{1}{\mathcal{B}}[(1+\alpha \mathcal{B})^{m}-1]  WH_h W^T W g_h (g_h+H_h\Delta h)^T\\
&=\frac{1}{\mathcal{B}}[(1+\alpha \mathcal{B})^{m}-1] H_x W g_h (g_h+H_h\Delta h)^T\\
&=\frac{1}{\mathcal{B}}[(1+\alpha \mathcal{B})^{m}-1] H_x g_x (g_h+H_h\Delta h)^T
\quad // \quad \text{ According to Eq.~\eqref{supp_eqn:gxgy}}\\
&=\frac{1}{\mathcal{B}}[(1+\alpha \mathcal{B})^{m}-1] H_x g_x g_h^T+\frac{1}{\mathcal{B}}[(1+\alpha \mathcal{B})^{m}-1] H_x g_x(H_h\Delta h)^T.
\end{aligned}
\end{equation}

Based on Eq.~\eqref{eqn:second_term_3}, the term {\small$H_x g_x (H_h\Delta h)^T$} can be represented as
\begin{equation}
    \label{eqn:Hh_delta_h1}
    H_x g_x (H_h\Delta h)^T=[(1+\alpha \mathcal{B})^m-1]H_x g_x g_h^T.
\end{equation}

Combining Eq.~\eqref{eqn:delta_x_g_y_5} and Eq.~\eqref{eqn:Hh_delta_h1}, we have
\begin{equation}
\label{eqn:delta_x_g_y_6}
\begin{aligned}
H_x \hat\delta (g_h+H_h\Delta h)^T
&=\frac{1}{\mathcal{B}}[(1+\alpha \mathcal{B})^{m}-1] H_x g_x g_h^T+\frac{1}{\mathcal{B}}[(1+\alpha \mathcal{B})^{m}-1] H_x g_x(H_h\Delta h)^T\\
&=\frac{1}{\mathcal{B}}[(1+\alpha \mathcal{B})^{m}-1] H_x g_x g_h^T+\frac{1}{\mathcal{B}}[(1+\alpha \mathcal{B})^{m}-1]^2 H_x g_xg_h^T\\
&=\frac{1}{\mathcal{B}}(1+\alpha \mathcal{B})^m[(1+\alpha \mathcal{B})^m-1]H_xg_xg_h^T.
\end{aligned}
\end{equation}

Based on Eq.~\eqref{eqn:A_infty}, the second term {\small$H_x \hat \delta (g_h+H_h\Delta h)^T$} in Eq.~\eqref{eqn:delta_g_w_expansion} can be written as follows, when the adversarial perturbation {\small$\hat\delta$} is generated via the infinite-step attack, {\small$m\to +\infty$}.
Here, {\small$\mathcal{A} =e^{\beta H_z\Vert\tilde{g}_x\Vert^2} \in \mathbb{R}$}, and  {\small$\mathcal{B}=H_z\tilde{g}_h^T W^T W\tilde{g}_h\in\mathbb{R}$}.
\begin{equation}
\label{eqn:delta_x_g_y_final}
\begin{aligned}
H_x\hat\delta (g_h+H_h\Delta h)^T
&=\lim_{m\to +\infty} \frac{1}{\mathcal{B}}(1+\alpha \mathcal{B})^m[(1+\alpha \mathcal{B})^m-1]H_xg_xg_h^T\\
&=\frac{1}{\mathcal{B}}(e^{2\beta \mathcal{B}}-e^{\beta \mathcal{B}})H_xg_x g_h^T\\
&=\frac{1}{H_z\Vert\tilde{g}_x\Vert^2}(e^{2\beta H_z\Vert \tilde{g}_x\Vert^2}-e^{\beta H_z\Vert \tilde{g}_x\Vert^2})H_xg_x g_h^T\\
&=\frac{1}{H_z\Vert\tilde{g}_x\Vert^2}(e^{2\mathcal{A}}-e^{\mathcal{A}})H_xg_x g_h^T.
\end{aligned}
\end{equation}

In this way, combining Eq.~\eqref{eqn:second_term_4} and Eq.~\eqref{eqn:delta_x_g_y_final},  Eq.~\eqref{eqn:delta_g_w_expansion} can be represented as
\begin{equation}
\label{eqn:delta_g_w_expansion_final }
\begin{aligned}
H_x \; \Delta g_{W}
& =H_x \; x(H_h \Delta h)^T+ H_x \; \hat\delta (g_h+H_h\Delta h)^T\\
&=(e^{\mathcal{A}}-1)H_x x g_h^T + \frac{1}{H_z\Vert\tilde{g}_x\Vert^2}(e^{2\mathcal{A}}-e^{\mathcal{A}})H_xg_x g_h^T.
\end{aligned}
\end{equation}

Thus, Lemma~\ref{lemma:H_x_g_w} in Appendix is proven.
\end{proof}
\subsection{Proof of Theorem~\ref{theorem2}}
\textbf{Theorem 3.}
\textit{Based on Assumptions~\ref{assumption:relu} and~\ref{assumption:sigmoid}, the effect of the adversarial perturbation {\small$\hat\delta$} on the change of the gradient {\small$\tilde{g}_x = \frac{\partial z(x)}{\partial x}$} is formulated as follows.
{\small$\Delta \tilde{g}_x = - \eta \Delta g_{W} \tilde{g}_h$} represents the additional effects of adversarial training on changing {\small$\tilde{g}_x$}, because
adversarial training makes an additional change {\small$- \eta \Delta g_{W}$} on {\small${W}$}\footnote[1]{It is because adversarial training changes {\small${W}$} by {\small$- \eta g_{W}^{\text{(adv)}}$}, and vanilla training changes {\small${W}$} by  {\small$- \eta g_{W}$}, {\small$\eta>0$}.}.
In this way, {\small$\tilde{g}_x ^{T} \Delta \tilde{g}_x$} measures the significance of such additional changes along the direction of the gradient {\small$\tilde{g}_x$}.}

\begin{small}
\begin{equation}
\label{eqn:adv_train}
\tilde{g}_x ^{T} \Delta \tilde{g}_x
=-\eta\tilde{g}_x ^{T} \Delta g_{W} \tilde{g}_h
= (e^{\mathcal{A}} -1) \tilde{g}_x ^{T}  \Delta \tilde{g}_x^{\text{(ori)}}
- \frac{ \eta g_z^{2} \;\Vert \tilde{g}_h \Vert ^{2} }{H_z}  (e^{2\mathcal{A}} -e^{\mathcal{A}}),
\end{equation}\end{small}
\textit{where {\small$\tilde{g}_h = \frac{\partial z(x)}{\partial h}$}, {\small$\mathcal{A} = \beta H_z \Vert \tilde{g}_x \Vert ^{2} \in \mathbb{R}$}, and $\eta$ denotes the learning rate to update the weight.
Considering the footnote\footnotemark[1], {\small$\Delta \tilde{g}_x^{\text{(ori)}} =- \eta g_{W} \tilde{g}_h$} measures the effects of vanilla training on changing {\small$\tilde{g}_x$} in the current back-propagation.
}

\begin{proof}

Based on Lemma~\ref{lemma:H_x} in Appendix and Lemma~\ref{lemma:H_x_g_w} in Appendix, we have
\begin{equation}
\begin{small}
\begin{aligned}
\label{eqn:H_x_change_gradient_w}
H_x\Delta g_W &=(e^{\mathcal{A}}-1)H_x x g_h^T+\frac{1}{H_z\Vert\tilde{g}_x\Vert^2}(e^{2\mathcal{A}}-e^{\mathcal{A}})H_xg_x g_h^T
\\
\Rightarrow \; H_z\tilde{g}_x\tilde{g}_x^T\Delta g_W&=(e^{\mathcal{A}}-1)H_z\tilde{g}_x\tilde{g}_x^T xg_h^T
+\frac{1}{H_z\Vert\tilde{g}_x\Vert^2}(e^{2 \mathcal{A}}-e^{\mathcal{A}})H_z\tilde{g}_x\tilde{g}_x^Tg_x g_h^T
\\
\Rightarrow \;  \tilde{g}_x\tilde{g}_x^T\Delta g_W&=(e^{\mathcal{A}}-1)\tilde{g}_x\tilde{g}_x^T xg_h^T
+\frac{1}{H_z\Vert\tilde{g}_x\Vert^2}(e^{2\mathcal{A}}-e^{\mathcal{A}})\tilde{g}_x\tilde{g}_x^Tg_x g_h^T.
\quad // \quad H_z\in \mathbb{R}
\end{aligned}
\end{small}
\end{equation}

Multiply {\small$\tilde{g}_x^{T}$} and {\small$\tilde{g}_h$} on both sides of Eq.~\eqref{eqn:H_x_change_gradient_w}, and we get
\begin{equation}
\begin{small}
\label{eqn:theorem_2}
\begin{aligned}
\tilde{g}_x^T\tilde{g}_x\tilde{g}_x^T\Delta g_W\tilde{g}_h
&=(e^{\mathcal{A}}-1)\tilde{g}_x^T\tilde{g}_x\tilde{g}_x^T xg_h^T\tilde{g}_h
+\frac{1}{H_z\Vert\tilde{g}_x\Vert^2}(e^{2\mathcal{A}}-e^{\mathcal{A}})\tilde{g}_x^T\tilde{g}_x\tilde{g}_x^Tg_x g_h^T\tilde{g}_h
\\
\Rightarrow \tilde{g}_x^T\tilde{g}_x\tilde{g}_x^T\Delta g_W\tilde{g}_h
&=(e^{\mathcal{A}}-1)\tilde{g}_x^T\tilde{g}_x\tilde{g}_x^T g_W\tilde{g}_h
+\frac{g_z^2}{H_z\Vert\tilde{g}_x\Vert^2}(e^{2\mathcal{A}}-e^{\mathcal{A}})\tilde{g}_x^T\tilde{g}_x\tilde{g}_x^T\tilde{g}_x \tilde{g}_h^T\tilde{g}_h
\\
\Rightarrow \tilde{g}_x^T\Delta g_W\tilde{g}_h
&=(e^{\mathcal{A}}-1)\tilde{g}_x^T g_W\tilde{g}_h
+\frac{g_z^2}{H_z}(e^{2\mathcal{A}}-e^{\mathcal{A}}) \tilde{g}_h^T\tilde{g}_h
\\
\Rightarrow \tilde{g}_x^T\Delta g_W\tilde{g}_h
&=(e^{\mathcal{A}}-1)\tilde{g}_x^T g_W\tilde{g}_h
+\frac{g_z^2 \Vert \tilde{g}_h\Vert^2}{H_z}(e^{2\mathcal{A}}-e^{\mathcal{A}})
\end{aligned}
\end{small}
\end{equation}

Let {\small$\Delta \tilde{g}_x = - \eta \Delta g_{W} \tilde{g}_h$} represent the additional effects of adversarial training on changing {\small$\tilde{g}_x$}, because
adversarial training makes an additional change {\small$- \eta \Delta g_{W}$} on {\small${W}$}\footnotemark[1].
Let {\small$\Delta \tilde{g}_x^{\text{(ori)}} =- \eta g_{W} \tilde{g}_h$} reflect the effects of vanilla training on changing {\small$\tilde{g}_x$} in the current back-propagation, considering the footnote\footnotemark[1].
In this way, Eq.~\eqref{eqn:theorem_2} can be re-written as
\begin{equation}
\label{eqn:theorem_2_final}
\begin{aligned}
 \tilde{g}_x^T (- \eta) \Delta g_W\tilde{g}_h
&=(e^{\mathcal{A}}-1)\tilde{g}_x^T (- \eta) g_W\tilde{g}_h
- \frac{ \eta g_z^{2} \;\Vert \tilde{g}_h \Vert ^{2} }{H_z}  (e^{2\mathcal{A}} -e^{\mathcal{A}})
\\
\Rightarrow \tilde{g}_x ^{T} \Delta \tilde{g}_x
&=(e^{\mathcal{A}} -1) \tilde{g}_x ^{T}  \Delta \tilde{g}_x^{\text{(ori)}}
- \frac{ \eta g_z^{2} \;\Vert \tilde{g}_h \Vert ^{2} }{H_z}  (e^{2\mathcal{A}} -e^{\mathcal{A}}).
\end{aligned}
\end{equation}

Thus, Theorem~\ref{theorem2} is proven.
\end{proof}

\section{Proof of Theorem~\ref{theorem3}}
In this section, we prove Theorem~\ref{theorem3} in Section~\ref{sec:3.2} of the main paper, which explains training effects of the adversarial perturbation $\hat\delta$ in Theorem~\ref{theorem:adv_per_2} on adversarial training.

\textbf{Theorem 4.}
\textit{Based on Assumptions~\ref{assumption:relu} and~\ref{assumption:sigmoid}, we derived the following equation~\textit{w.r.t.} adversarial training based on perturbations {\small$\hat\delta$} in Theorem~\ref{theorem:adv_per_2}.
Considering the footnote\footnotemark[1], {\small$\Delta \tilde{g}_x^{\text{(adv)}} =-\eta g_{W}^{\text{(adv)}} \tilde{g}_h$} reflects effects of adversarial training on changing the gradient {\small$\tilde{g}_x$}.
In this way, {\small$\tilde{g}_x ^{T}  \Delta \tilde{g}_x^{\text{(adv)}}$} represents the significance of such effects along the direction of the gradient {\small$\tilde{g}_x$}.}

\begin{small}
\begin{equation}
\label{eqn:adv_train_1}
\tilde{g}_x ^{T}  \Delta \tilde{g}_x^{\text{(adv)}}
= -\eta\tilde{g}_x ^{T} g_{W}^{\text{(adv)}} \tilde{g}_h
= e^{ \mathcal{A}} \tilde{g}_x ^{T} \Delta \tilde{g}_x^{\text{(ori)}}
- \frac{ \eta g_z^{2}  (e^{2 \mathcal{A}} -e^{ \mathcal{A}})}{H_z}  \Vert \tilde{g}_h \Vert ^{2}.
\end{equation}
\end{small}


\begin{proof}
Based on Eq.~\eqref{eqn:delta_g_w_pre}, {\small$\Delta g_{W} = g^{\text{(adv)}}_{W}- g_{W}$}, we add {\small$\tilde{g}_x^T g_W\tilde{g}_h$} on both sides of Eq.~\eqref{eqn:theorem_2}.
\begin{equation}
\begin{small}
\label{eqn:theorem_3}
\begin{aligned}
 \tilde{g}_x^T\Delta g_W\tilde{g}_h + \tilde{g}_x^T g_W\tilde{g}_h
&=(e^{\mathcal{A}}-1)\tilde{g}_x^T g_W\tilde{g}_h + \tilde{g}_x^T g_W\tilde{g}_h
+\frac{g_z^2 \Vert \tilde{g}_h\Vert^2}{H_z}(e^{2\mathcal{A}}-e^{\mathcal{A}})
\\
\Rightarrow  \tilde{g}_x^T (\Delta g_W + g_W) \tilde{g}_h
 &=e^{\mathcal{A}}\tilde{g}_x^T g_W\tilde{g}_h
+\frac{g_z^2 \Vert \tilde{g}_h\Vert^2}{H_z}(e^{2\mathcal{A}}-e^{\mathcal{A}})
\\
\Rightarrow  \tilde{g}_x^T  g^{\text{(adv)}}_{W} \tilde{g}_h
 &=e^{\mathcal{A}}\tilde{g}_x^T g_W\tilde{g}_h
+\frac{g_z^2 \Vert \tilde{g}_h\Vert^2}{H_z}(e^{2\mathcal{A}}-e^{\mathcal{A}}).
\end{aligned}
\end{small}
\end{equation}

Let {\small$\Delta \tilde{g}_x^{\text{(adv)}} =-\eta g_{W}^{\text{(adv)}} \tilde{g}_h$} represent effects of adversarial training on changing the gradient {\small$\tilde{g}_x$}.
Then, Eq.~\eqref{eqn:theorem_3} can be simplified as
\begin{equation}
\label{eqn:theorem_2_final}
\begin{aligned}
 \tilde{g}_x^T (- \eta) g^{\text{(adv)}}_{W}\tilde{g}_h
&=e^{\mathcal{A}}\tilde{g}_x^T (- \eta) g_W\tilde{g}_h
- \frac{ \eta g_z^{2} \;\Vert \tilde{g}_h \Vert ^{2} }{H_z}  (e^{2\mathcal{A}} -e^{\mathcal{A}})
\\
\Rightarrow \tilde{g}_x ^{T} \Delta \tilde{g}_x^{\text{(adv)}}
&=e^{\mathcal{A}}  \tilde{g}_x ^{T}  \Delta \tilde{g}_x^{\text{(ori)}}
- \frac{ \eta g_z^{2} \;\Vert \tilde{g}_h \Vert ^{2} }{H_z}  (e^{2\mathcal{A}} -e^{\mathcal{A}}).
\end{aligned}
\end{equation}

Thus, Theorem~\ref{theorem3} is proven.
\end{proof}


\section{Proof of Theorem~\ref{theorem4}}

In this section, we prove Theorem~\ref{theorem4} in Section~\ref{sec:3.2}  of the main paper, which approximately explains adversarial training based on perturbations of the $\ell_2$ attack and the $\ell_\infty$ attack.

Specifically, if we use vanilla training to fine-tune the network on the original input sample $x$ for a single step, then the gradient of the loss~\textit{w.r.t.} the weight {\small$W$} is given as {\small$g_{W}= \frac{\partial}{\partial W} L(f(x),y)$}.
In comparison, if we train the network on the adversarial example {\small$x+\hat\delta^{\text{(norm)}}$} for a single step, then we will get the gradient {\small$g^{\text{(adv,norm)}}_{W} =\frac{\partial}{\partial W}  L(f(x+\hat\delta^{\text{(norm)}}),y)$}.
In this way, {\small $\Delta g_{W}^{\text{(norm)}} = g^{\text{(adv,norm)}}_{W}- g_{W}$} represents additional effects on the gradient brought by adversarial training, when we use the normalized perturbation {\small$\hat\delta^{\text{(norm)}}$} in Corollary~\ref{theorem:adv_per_norm} (related to the $\ell_2$ attack and the $\ell_\infty$ attack).
\begin{equation}
\label{eqn:delta_g_w_norm_pre}
\begin{small}
\begin{aligned}
\Delta g_{W}^{\text{(norm)}} &= g^{\text{(adv,norm)}}_{W}- g_{W}= \frac{\partial}{\partial W} L(f(x+\hat\delta^{\text{(norm)}}),y) - \frac{\partial}{\partial W}  L(f(x),y).\\
& = x(H_h \Delta h^\text{(norm)})^T+\hat\delta^\text{(norm)} (g_h+H_h\Delta h^\text{(norm)})^T\\
& =  x(\frac{C}{\Vert \hat\delta \Vert} H_h  \Delta h)^T+ \frac{C \cdot \hat\delta}{\Vert \hat\delta \Vert}  (g_h+ \frac{C}{\Vert \hat\delta \Vert} H_h\Delta h)^T,
\end{aligned}
\end{small}
\end{equation}
where {\small$\Delta h^\text{(norm)}=W^T\hat\delta^{\text{(norm)}} = \frac{C}{\Vert \hat\delta \Vert} W^T\hat\delta = \frac{C}{\Vert \hat\delta \Vert}  \Delta h$} denotes the change of the intermediate-layer feature $h$ caused by the perturbation {\small$\hat\delta^{\text{(norm)}}$}.
Here, {\small$W=W_{j}^{T}\Sigma_{j-1}\cdots\Sigma_{2}W_2^{T}\Sigma_{1}W_1^{T}$}.
It is because, according to Assumption 1 in the main paper, we simplify our research into an idealized adversarial attack, whose adversarial perturbation does not significantly change gating states in gating layers.
In this way, to simplify the proof, we can roughly represent the output feature $h$ of the $j$-th layer as {\small$h\approx W^{T} x +b$}.

\begin{proof}

According to Eq.~\eqref{eqn:gradient_of_w}, {\small$g_W =xg_h^T$}, the gradient {\small$g^{\text{(adv,norm)}}_{W} = \frac{\partial}{\partial W}  L(f(x+\hat\delta^{\text{(norm)}}),y)$} can be re-written as follows, where
{\small$g_{h+\Delta h^\text{(norm)}} =  \frac{\partial}{\partial h+\Delta h^\text{(norm)}}  L(f(x+\hat\delta^{\text{(norm)}}),y)$}.
\begin{equation}
\begin{aligned}
\label{eqn:gradient_of_w_adv_norm}
g^{\text{(adv,norm)}}_{W}=(x+\hat\delta^\text{(norm)})(g_{h+\Delta h^\text{(norm)}})^T.
\end{aligned}
\end{equation}

We further use the first-order Taylor expansion to decompose the gradient {\small$g_{h+\Delta h^\text{(norm)}}$}, where {\small$R^{\text{(grad)}}_2(\Delta h^\text{(norm)})$} denotes the terms no less than the second order.
\begin{equation}
\label{eqn:gradient_of_h_adv_norm}
g_{h+\Delta h^\text{(norm)}} =g_h+H_h \Delta h^\text{(norm)} + R^{\text{(grad)}}_2(\Delta h^\text{(norm)}).
\end{equation}

Substituting Eq.~\eqref{eqn:gradient_of_h_adv_norm} back to Eq.~\eqref{eqn:gradient_of_w_adv_norm}, the gradient {\small$g^{\text{(adv,norm)}}_{W}$} can be represented as
\begin{equation}
\begin{aligned}
\label{eqn:gradient_of_w_adv_2_norm}
g^{\text{(adv,norm)}}_{W}=(x+\hat\delta^\text{(norm)}) \bigg(g_h+H_h\Delta h^\text{(norm)} + R^{\text{(grad)}}_2(\Delta h^\text{(norm)}) \bigg)^T.
\end{aligned}
\end{equation}

Thus, the additional effects of adversarial training on the gradient can be written as follows.
\begin{equation}
\begin{small}
\begin{aligned}
\label{eqn:change_gradient_w_norm}
\Delta g_{W}^{\text{(norm)}} &=g^{\text{(adv,norm)}}_{W}-g_W
\\
&= x(H_h\Delta h^\text{(norm)})^T+\hat\delta^\text{(norm)} (g_h+H_h\Delta h^\text{(norm)})^T +
(x+\hat\delta^\text{(norm)}) \bigg( R^{\text{(grad)}}_2(\Delta h^\text{(norm)}) \bigg)^T
\\
&\approx x(H_h\Delta h^\text{(norm)})^T+\hat\delta^\text{(norm)} (g_h+H_h\Delta h^\text{(norm)})^T
\\
&=x(\frac{C}{\Vert \hat\delta \Vert} H_h  \Delta h)^T
+ \frac{C \cdot \hat\delta}{\Vert \hat\delta \Vert}  (g_h+ \frac{C}{\Vert \hat\delta \Vert} H_h\Delta h)^T
\end{aligned}
\end{small}
\end{equation}
\end{proof}

\begin{lemma}[in Appendix]
\label{lemma:H_x_g_w_norm}
Let {\small$\tilde{g}_x = \frac{\partial}{\partial x}  z(x)$} denote the gradient of the network output $z$~\textit{w.r.t} the input sample $x$, and {\small$\mathcal{A} =\beta H_z\Vert\tilde{g}_x\Vert^2 \in \mathbb{R} $}.
Then, we have

\begin{small}
\begin{equation}
H_x \; \Delta g_{W}^{\text{(norm)}} =\frac{C}{\Vert \hat\delta \Vert} (e^{\mathcal{A}}-1)H_x xg_h^T
+ \frac{C}{ \Vert \hat\delta \Vert H_z\Vert\tilde{g}_x\Vert^2} (e^{\mathcal{A}}-1) \big[ 1+ \frac{C}{\Vert \hat\delta \Vert} (e^{\mathcal{A}}-1) \big] H_xg_x g_h^T.
\end{equation}
\end{small}
\end{lemma}

\begin{proof}
To prove Lemma~\ref{lemma:H_x_g_w_norm} in Appendix, we multiply {\small$H_x$} on both sides of Eq.~\eqref{eqn:delta_g_w_norm_pre}
\begin{equation}
\label{eqn:delta_g_w_expansion_norm}
\begin{small}
\begin{aligned}
H_x \; \Delta g_{W}^{\text{(norm)}} &= H_x \; (g^{\text{(adv,norm)}}_{W} - g_{W})\\
 &=H_x \; x(\frac{C}{\Vert \hat\delta \Vert} H_h  \Delta h)^T
+ H_x \; \frac{C \cdot \hat\delta}{\Vert \hat\delta \Vert}  (g_h+ \frac{C}{\Vert \hat\delta \Vert} H_h\Delta h)^T
\end{aligned}
\end{small}
\end{equation}

Let us first focus on the first term {\small$H_x \; x(\frac{C}{\Vert \hat\delta \Vert} H_h  \Delta h)^T$} in Eq.~\eqref{eqn:delta_g_w_expansion_norm}.
Based on Eq.~\eqref{eqn:second_term_4}, {\small$H_x x(H_h \Delta h)^T= (e^{ \mathcal{A}}-1) H_x x g_h^T$},
we have
\begin{equation}
\begin{aligned}
\label{eqn:first_term_norm}
H_x x(\frac{C}{\Vert \hat\delta \Vert} H_h  \Delta h)^T
 &= \frac{C}{\Vert \hat\delta \Vert} (e^{ \mathcal{A}}-1) H_x x g_h^T.
\end{aligned}
\end{equation}

Then, let us focus on the second term {\small$ \frac{C}{\Vert \hat\delta \Vert} H_x \hat \delta (g_h+\frac{C}{\Vert \hat\delta \Vert} H_h\Delta h)^T$} in Eq.~\eqref{eqn:delta_g_w_expansion_norm}.
Based on Eq.~\eqref{supp_eqn:perturbation_approx} and  Lemma~\ref{lemma:H_x} in Appendix, the second term {\small$ \frac{C}{\Vert \hat\delta \Vert} H_x \hat \delta (g_h+\frac{C}{\Vert \hat\delta \Vert} H_h\Delta h)^T$} can be re-written as follows.
Note that since the step size is infinitesimal {\small$\alpha \to 0$}, the perturbation $\hat\delta$ is mainly dominated by the term  {\small$[I+(I+\alpha H_x)+\cdots+(I+\alpha H_x)^{m-1}]g_x$}.
Thus, we ignore the error $\hat{\mathcal{R}}_2(\alpha)$ in Eq.~\eqref{supp_eqn:perturbation_2}.
\begin{equation}
\begin{small}
\label{eqn:Delta x g_y^T_norm}
\begin{aligned}
&\frac{C}{\Vert \hat\delta \Vert} H_x \hat \delta (g_h+\frac{C}{\Vert \hat\delta \Vert} H_h\Delta h)^T
\\
&= \frac{C}{\Vert \hat\delta \Vert} H_x\alpha[I+(I+\alpha WH_h W^T)+\cdots+(I+\alpha W H_h W^T)^{m-1}]g_x (g_h+ \frac{C}{\Vert \hat\delta \Vert} H_h\Delta h)^T
\\
&=\frac{C}{\Vert \hat\delta \Vert} H_x\alpha[I+(I+\alpha WH_h W^T)+\cdots+(I+\alpha W H_h W^T)^{m-1}]W  g_h (g_h+ \frac{C}{\Vert \hat\delta \Vert}H_h\Delta h)^T.
\end{aligned}
\end{small}
\end{equation}

For simplicity, let {\small$S =I+(I+\alpha WH_h W^T)+\cdots+(I+\alpha WH_h W^T)^{m-1}$}.
According to Eq.~\eqref{eqn:delta_x_g_y_3}, we have proven {\small$H_x\alpha S =(I+\alpha WH_h W^T)^{m}-I$}.
In this way, Eq.~\eqref{eqn:Delta x g_y^T_norm} can be further simplified as
\begin{equation}
\label{eqn:delta_x_g_y_norm}
\begin{small}
\begin{aligned}
\frac{C}{\Vert \hat\delta \Vert} H_x \hat \delta (g_h+\frac{C}{\Vert \hat\delta \Vert} H_h\Delta h)^T
&= \frac{C}{\Vert \hat\delta \Vert} \big[(I+\alpha WH_h W^T)^{m}-I\big] W g_h (g_h+ \frac{C}{\Vert \hat\delta \Vert} H_h\Delta h)^T.
\end{aligned}
\end{small}
\end{equation}

Moreover, we have proven {\small$(I+\alpha WH_h W^T)^{m}-I = \frac{1}{\mathcal{B}}[(1+\alpha \mathcal{B})^t-1] WH_h W^T$} in Eq.~\eqref{eqn:t_term_final}, where {\small$\mathcal{B}=H_z\tilde{g}_h^T W^T W\tilde{g}_h\in\mathbb{R}$}.
In this way, we get
\begin{equation}
\label{eqn:delta_x_g_y_norm_1}
\begin{small}
\begin{aligned}
&\frac{C}{\Vert \hat\delta \Vert} H_x \hat \delta (g_h+\frac{C}{\Vert \hat\delta \Vert} H_h\Delta h)^T\\
&=\frac{C}{\Vert \hat\delta \Vert \cdot \mathcal{B}} \bigg[(1+\alpha \mathcal{B})^{m}-1\bigg]  WH_h W^T W g_h (g_h+ \frac{C}{\Vert \hat\delta \Vert} H_h\Delta h)^T\\
&=\frac{C}{\Vert \hat\delta \Vert \cdot \mathcal{B}}\bigg[(1+\alpha \mathcal{B})^{m}-1\bigg] H_x W g_h (g_h+ \frac{C}{\Vert \hat\delta \Vert} H_h\Delta h)^T\\
&=\frac{C}{\Vert \hat\delta \Vert \cdot \mathcal{B}}\bigg[(1+\alpha \mathcal{B})^{m}-1\bigg] H_x g_x (g_h+ \frac{C}{\Vert \hat\delta \Vert} H_h\Delta h)^T
 \quad// \quad \text{ According to Eq.~\eqref{supp_eqn:gxgy}}\\
&=\frac{C}{\Vert \hat\delta \Vert \cdot \mathcal{B}}\bigg[(1+\alpha \mathcal{B})^{m}-1\bigg] H_x g_x g_h^T\\
&\quad+\frac{C}{\Vert \hat\delta \Vert \cdot \mathcal{B}}\bigg[(1+\alpha \mathcal{B})^{m}-1\bigg] H_x g_x( \frac{C}{\Vert \hat\delta \Vert} H_h\Delta h)^T.
\end{aligned}
\end{small}
\end{equation}

Based on Eq.~\eqref{eqn:second_term_3}, the term {\small$H_x g_x ( \frac{C}{\Vert \hat\delta \Vert}  H_h\Delta h)^T$} can be represented as
\begin{equation}
\begin{small}
 \label{eqn:Hh_delta_h1_norm}
 \begin{aligned}
    H_x g_x (\frac{C}{\Vert \hat\delta \Vert}  H_h\Delta h)^T= \frac{C}{\Vert \hat\delta \Vert}  \bigg[(1+\alpha \mathcal{B})^m-1\bigg]H_x g_x g_h^T.
  \end{aligned}  \end{small}
\end{equation}

Combining Eq.~\eqref{eqn:Hh_delta_h1_norm} and Eq.~\eqref{eqn:delta_x_g_y_norm_1}, we have
\begin{equation}
\label{eqn:delta_x_g_y_norm_2}
\begin{small}
\begin{aligned}
\frac{C}{\Vert \hat\delta \Vert} H_x \hat \delta (g_h+\frac{C}{\Vert \hat\delta \Vert} H_h\Delta h)^T
&=\frac{C}{\Vert \hat\delta \Vert \cdot \mathcal{B}}\bigg[(1+\alpha \mathcal{B})^{m}-1\bigg] H_x g_x g_h^T\\
&\quad+\frac{C}{\Vert \hat\delta \Vert \cdot \mathcal{B}}\bigg[(1+\alpha \mathcal{B})^{m}-1\bigg] H_x g_x( \frac{C}{\Vert \hat\delta \Vert} H_h\Delta h)^T
\\
\\
&=\frac{C}{\Vert \hat\delta \Vert \cdot \mathcal{B}} \bigg[(1+\alpha \mathcal{B})^{m}-1\bigg] H_x g_x g_h^T\\
&\quad+\frac{C^2}{\Vert \hat\delta \Vert^2 \cdot \mathcal{B}} \bigg[(1+\alpha \mathcal{B})^{m}-1\bigg]^2 H_x g_xg_h^T
\\
\\
&=\frac{C}{\Vert \hat\delta \Vert \cdot \mathcal{B}} \bigg[(1+\alpha \mathcal{B})^{m}-1\bigg]
\bigg[1+ \frac{C}{\Vert \hat\delta \Vert }   \big[(1+\alpha \mathcal{B})^{m}-1\big]    \bigg]H_xg_xg_h^T.
\end{aligned}
\end{small}
\end{equation}

It is because in Eq.~\eqref{eqn:A_infty}, we have proven {\small$\lim_{m\to +\infty}(1+\alpha  \mathcal{B})^m = e^{ \mathcal{A}}$}, where {\small$\mathcal{A} =e^{\beta H_z\Vert\tilde{g}_x\Vert^2} \in \mathbb{R}$}.
Then, the second term {\small$ \frac{C}{\Vert \hat\delta \Vert} H_x \hat \delta (g_h+\frac{C}{\Vert \hat\delta \Vert} H_h\Delta h)^T$} in Eq.~\eqref{eqn:delta_g_w_expansion_norm} can be further written as follows, when the adversarial perturbation {\small$\hat\delta$} is generated via the infinite-step attack, {\small$m\to +\infty$}.
\begin{equation}
\label{eqn:delta_x_g_y_final_norm}
\begin{small}
\begin{aligned}
&\frac{C}{\Vert \hat\delta \Vert} H_x \hat \delta (g_h+\frac{C}{\Vert \hat\delta \Vert} H_h\Delta h)^T\\
&=\lim_{m\to +\infty} \frac{C}{\Vert \hat\delta \Vert \cdot \mathcal{B}} \bigg[(1+\alpha \mathcal{B})^{m}-1\bigg]
\bigg[1+ \frac{C}{\Vert \hat\delta \Vert }   \big[(1+\alpha \mathcal{B})^{m}-1\big]    \bigg]H_xg_xg_h^T
\\
&=\frac{C}{\Vert \hat\delta \Vert \cdot \mathcal{B}} (e^{ \mathcal{A}}-1)
\bigg[1+ \frac{C}{\Vert \hat\delta \Vert }   (e^{ \mathcal{A}}-1)    \bigg]H_xg_xg_h^T
\\
&=\frac{C}{\Vert \hat\delta \Vert  H_z \Vert\tilde{g}_x\Vert^2} (e^{ \mathcal{A}}-1)
\bigg[1+ \frac{C}{\Vert \hat\delta \Vert }   (e^{ \mathcal{A}}-1)    \bigg]H_xg_xg_h^T.
\end{aligned}\end{small}
\end{equation}

In this way, combining Eq.~\eqref{eqn:first_term_norm} and Eq.~\eqref{eqn:delta_x_g_y_final_norm},  Eq.~\eqref{eqn:delta_g_w_expansion_norm} can be represented as
\begin{equation}
\begin{small}
\label{eqn:delta_g_w_expansion_final_norm}
\begin{aligned}
H_x \; \Delta g_{W}^{\text{(norm)}}
 &=H_x \; x(\frac{C}{\Vert \hat\delta \Vert} H_h  \Delta h)^T
+ H_x \; \frac{C \cdot \hat\delta}{\Vert \hat\delta \Vert}  (g_h+ \frac{C}{\Vert \hat\delta \Vert} H_h\Delta h)^T
\\
&=\frac{C}{\Vert \hat\delta \Vert} (e^{ \mathcal{A}}-1) H_x x g_h^T
+\frac{C}{\Vert \hat\delta \Vert  H_z \Vert\tilde{g}_x\Vert^2} (e^{ \mathcal{A}}-1)
\bigg[1+ \frac{C}{\Vert \hat\delta \Vert }   (e^{ \mathcal{A}}-1)    \bigg]H_xg_xg_h^T.
\end{aligned}\end{small}
\end{equation}

Thus, Lemma~\ref{lemma:H_x_g_w_norm} in Appendix is proven.
\end{proof}

\subsection{Proof of Theorem~\ref{theorem4}}
\begin{theorem}
Based on Assumptions~\ref{assumption:relu} and~\ref{assumption:sigmoid}, we derived the following equation~\textit{w.r.t.} adversarial training based on normalized perturbations {\small$\hat\delta^{\text{(norm)}}$}  in Corollary~\ref{theorem:adv_per_norm}.
Considering the footnote\footnotemark[1], {\small$\Delta \tilde{g}_x^{\text{(norm)}}=- \eta \Delta g_{W}^{\text{(norm)}} \tilde{g}_h = - \eta( g_{W}^{\text{(adv, norm)}}- g_{W}) \tilde{g}_h$}
represents additional effects of adversarial training on changing {\small$\tilde{g}_x$}.
In this way, {\small $\tilde{g}_x ^{T} \Delta \tilde{g}_x^{\text{(norm)}} =- \eta \tilde{g}_x ^{T}\Delta g_{W}^{\text{(norm)}} \tilde{g}_h$} reflects the significance of such additional effects along the direction of the gradient {\small$ \tilde{g}_x$}.

\begin{small}
\begin{equation}
\label{eqn:norm_adv_train}
\tilde{g}_x ^{T} \Delta \tilde{g}_x^{\text{(norm)}}
= C \cdot \big(\frac{e^{\mathcal{A}}}{\Vert \hat\delta \Vert} -\frac{1}{\Vert \hat\delta \Vert}\big) \tilde{g}_x ^{T}  \Delta \tilde{g}_x^{\text{(ori)}}
- C \cdot \frac{   \eta g_z^{2}\; \Vert \tilde{g}_h \Vert ^{2} }{H_z }
\bigg(\frac{e^{\mathcal{A}}}{\Vert \hat\delta \Vert} -\frac{1}{\Vert \hat\delta \Vert} + C \cdot (\frac{e^{\mathcal{A}}}{\Vert \hat\delta \Vert} -\frac{1}{\Vert \hat\delta \Vert})^{2} \bigg).
\end{equation}
\end{small}
\end{theorem}

\begin{proof}

Based on Lemma~\ref{lemma:H_x} in Appendix and Lemma~\ref{lemma:H_x_g_w_norm} in Appendix, we have
\begin{equation}
\begin{small}
\begin{aligned}
\label{eqn:H_x_change_gradient_w_norm}
H_x \; \Delta g_{W}^{\text{(norm)}}
 &=\frac{C}{\Vert \hat\delta \Vert} (e^{ \mathcal{A}}-1) H_x x g_h^T
+\frac{C}{\Vert \hat\delta \Vert  H_z \Vert\tilde{g}_x\Vert^2} (e^{ \mathcal{A}}-1)
\bigg[1+ \frac{C}{\Vert \hat\delta \Vert }   (e^{ \mathcal{A}}-1)    \bigg]H_xg_xg_h^T
\\
\\
\Rightarrow \; H_z\tilde{g}_x\tilde{g}_x^T \Delta g_{W}^{\text{(norm)}}
&=\frac{C}{\Vert \hat\delta \Vert} (e^{\mathcal{A}}-1)H_z\tilde{g}_x\tilde{g}_x^T xg_h^T\\
&\quad+\frac{C}{\Vert \hat\delta \Vert  H_z \Vert\tilde{g}_x\Vert^2} (e^{ \mathcal{A}}-1)
\bigg[1+ \frac{C}{\Vert \hat\delta \Vert }   (e^{ \mathcal{A}}-1)    \bigg] H_z\tilde{g}_x\tilde{g}_x^Tg_x g_h^T
\\
\\
\Rightarrow \;  \tilde{g}_x\tilde{g}_x^T \Delta g_{W}^{\text{(norm)}}
&=\frac{C}{\Vert \hat\delta \Vert} (e^{\mathcal{A}}-1)\tilde{g}_x\tilde{g}_x^T xg_h^T
\\
&\quad+\frac{C}{\Vert \hat\delta \Vert  H_z \Vert\tilde{g}_x\Vert^2} (e^{ \mathcal{A}}-1)
\bigg[1+ \frac{C}{\Vert \hat\delta \Vert }   (e^{ \mathcal{A}}-1)    \bigg] \tilde{g}_x\tilde{g}_x^Tg_x g_h^T.
\quad // \quad H_z\in \mathbb{R}
\end{aligned}
\end{small}
\end{equation}

Multiply {\small$\tilde{g}_x^{T}$} and {\small$\tilde{g}_h$} on both sides of Eq.~\eqref{eqn:H_x_change_gradient_w_norm}, and we get
\begin{equation}
\begin{small}
\label{eqn:theorem_norm}
\begin{aligned}
\tilde{g}_x^T\tilde{g}_x\tilde{g}_x^T \Delta g_{W}^{\text{(norm)}} \tilde{g}_h
&=\frac{C}{\Vert \hat\delta \Vert} (e^{\mathcal{A}}-1)\tilde{g}_x^T\tilde{g}_x\tilde{g}_x^T xg_h^T\tilde{g}_h\\
&\quad+\frac{C}{\Vert \hat\delta \Vert  H_z \Vert\tilde{g}_x\Vert^2} (e^{ \mathcal{A}}-1)
\bigg[1+ \frac{C}{\Vert \hat\delta \Vert }   (e^{ \mathcal{A}}-1)    \bigg]\tilde{g}_x^T\tilde{g}_x\tilde{g}_x^Tg_x g_h^T\tilde{g}_h
\\
\\
\Rightarrow \tilde{g}_x^T\tilde{g}_x\tilde{g}_x^T \Delta g_{W}^{\text{(norm)}} \tilde{g}_h
&=\frac{C}{\Vert \hat\delta \Vert}(e^{\mathcal{A}}-1)\tilde{g}_x^T\tilde{g}_x\tilde{g}_x^T g_W\tilde{g}_h\\
&\quad+\frac{C g_z^2}{\Vert \hat\delta \Vert  H_z \Vert\tilde{g}_x\Vert^2} (e^{ \mathcal{A}}-1)
\bigg[1+ \frac{C}{\Vert \hat\delta \Vert }   (e^{ \mathcal{A}}-1)    \bigg] \tilde{g}_x^T\tilde{g}_x\tilde{g}_x^T\tilde{g}_x \tilde{g}_h^T\tilde{g}_h
\\
\\
\Rightarrow \tilde{g}_x^T   \Delta g_{W}^{\text{(norm)}} \tilde{g}_h
&=\frac{C}{\Vert \hat\delta \Vert} (e^{\mathcal{A}}-1)\tilde{g}_x^T g_W\tilde{g}_h\\
&\quad+\frac{C g_z^2}{\Vert \hat\delta \Vert  H_z } (e^{ \mathcal{A}}-1)
\bigg[1+ \frac{C}{\Vert \hat\delta \Vert }   (e^{ \mathcal{A}}-1)    \bigg] \tilde{g}_h^T\tilde{g}_h
\\
\\
\Rightarrow \tilde{g}_x^T \Delta g_{W}^{\text{(norm)}} \tilde{g}_h
&=\frac{C}{\Vert \hat\delta \Vert}  (e^{\mathcal{A}}-1)\tilde{g}_x^T g_W\tilde{g}_h
+\frac{C g_z^2  \Vert\tilde{g}_h\Vert^2}{\Vert \hat\delta \Vert  H_z } (e^{ \mathcal{A}}-1)
\bigg[1+ \frac{C}{\Vert \hat\delta \Vert }   (e^{ \mathcal{A}}-1)    \bigg]
\\
\\
\Rightarrow \tilde{g}_x^T \Delta g_{W}^{\text{(norm)}} \tilde{g}_h
&=C \cdot  \big(\frac{e^{\mathcal{A}}}{\Vert \hat\delta \Vert} -\frac{1}{\Vert \hat\delta \Vert}\big) \tilde{g}_x^T g_W\tilde{g}_h
\\
&\quad+C \cdot \frac{ g_z^2  \Vert\tilde{g}_h\Vert^2}{H_z } \bigg(\frac{e^{\mathcal{A}}}{\Vert \hat\delta \Vert} -\frac{1}{\Vert \hat\delta \Vert} + C \cdot (\frac{e^{\mathcal{A}}}{\Vert \hat\delta \Vert} -\frac{1}{\Vert \hat\delta \Vert})^{2} \bigg).
\end{aligned}
\end{small}
\end{equation}

Let {\small$\Delta \tilde{g}_x^{\text{(norm)}}=- \eta \Delta g_{W}^{\text{(norm)}} \tilde{g}_h$} represent the additional effects of adversarial training on changing {\small$\tilde{g}_x$}, considering the footnote\footnotemark[1].
In this way, Eq.~\eqref{eqn:theorem_norm} can be re-written as
\begin{equation}
\begin{small}
\label{eqn:theorem_norm_final}
\begin{aligned}
\tilde{g}_x^T(- \eta)  \Delta g_{W}^{\text{(norm)}} \tilde{g}_h
&=C \cdot  \big(\frac{e^{\mathcal{A}}}{\Vert \hat\delta \Vert} -\frac{1}{\Vert \hat\delta \Vert}\big) \tilde{g}_x^T (- \eta) g_W\tilde{g}_h
- \frac{ \eta g_z^{2} \;\Vert \tilde{g}_h \Vert ^{2} }{H_z}
\bigg(\frac{e^{\mathcal{A}}}{\Vert \hat\delta \Vert} -\frac{1}{\Vert \hat\delta \Vert} + C \cdot (\frac{e^{\mathcal{A}}}{\Vert \hat\delta \Vert} -\frac{1}{\Vert \hat\delta \Vert})^{2} \bigg)
\\
\Rightarrow \tilde{g}_x ^{T} \Delta \tilde{g}_x^{\text{(norm)}}
&=C \cdot  \big(\frac{e^{\mathcal{A}}}{\Vert \hat\delta \Vert} -\frac{1}{\Vert \hat\delta \Vert}\big) \tilde{g}_x^T \Delta \tilde{g}_x^{\text{(ori)}}
- \frac{ \eta g_z^{2} \;\Vert \tilde{g}_h \Vert ^{2} }{H_z}
\bigg(\frac{e^{\mathcal{A}}}{\Vert \hat\delta \Vert} -\frac{1}{\Vert \hat\delta \Vert} + C \cdot (\frac{e^{\mathcal{A}}}{\Vert \hat\delta \Vert} -\frac{1}{\Vert \hat\delta \Vert})^{2} \bigg).
\end{aligned}\end{small}
\end{equation}

Thus, Theorem~\ref{theorem4} is proven.
\end{proof}

\subsection{Proof for the strength of the training effect {\small$\tilde{g}_x ^{T} \Delta \tilde{g}_x^{\text{(norm)}}$} in Theorem~\ref{theorem4}}

Given a relatively strong attack, Theorem~\ref{theorem:adv_per_2} shows {\small$\Vert \hat\delta \Vert \to \exp(\beta  \Vert \tilde{g}_x \Vert ^{2} g_z^{2} )/{ \Vert {g}_x \Vert}$}.
In this way, we can ignore the term {\small$1/ \Vert \hat\delta \Vert \to 0$} in Eq.~\eqref{main_eqn:norm_adv_train}, and prove that the strength of
the training effect {\small$\tilde{g}_x ^{T} \Delta \tilde{g}_x^{\text{(norm)}}$} is mainly determined by the term {\small$\exp(\mathcal{A})/\Vert \hat\delta \Vert \approx \Vert {g}_x \Vert \cdot \exp( \beta  \Vert \tilde{g}_x \Vert ^{2} (H_z-g_z^{2}))$}.
The proof is as follows.

\begin{proof}
%

Given a relatively strong attack, we can ignore the term {\small$1/ \Vert \hat\delta \Vert \to 0$} in Eq.~\eqref{main_eqn:norm_adv_train}, because a a relatively strong adversarial strength $\beta$ usually makes {\small$\Vert \hat\delta \Vert \to \exp(\beta  \Vert \tilde{g}_x \Vert ^{2} g_z^{2} )/{ \Vert {g}_x \Vert}$} with an exponential strength.
In this way, Eq.~\eqref{main_eqn:norm_adv_train} can be re-written as
\begin{equation}
\begin{small}
\begin{aligned}
\label{eqn:strong_attack_norm_adv_train}
\tilde{g}_x^T\Delta\tilde{g}_x^{\text{(norm)}}&=C \cdot  \big(\frac{e^{\mathcal{A}}}{\Vert \hat\delta \Vert} -\frac{1}{\Vert \hat\delta \Vert}\big) \tilde{g}_x^T \Delta \tilde{g}_x^{\text{(ori)}}
- \frac{ \eta g_z^{2} \;\Vert \tilde{g}_h \Vert ^{2} }{H_z}
\bigg(\frac{e^{\mathcal{A}}}{\Vert \hat\delta \Vert} -\frac{1}{\Vert \hat\delta \Vert} + C \cdot (\frac{e^{\mathcal{A}}}{\Vert \hat\delta \Vert} -\frac{1}{\Vert \hat\delta \Vert})^{2} \bigg)\\
&\approx C \cdot \frac{e^{\mathcal{A}}}{\Vert \hat\delta \Vert} \tilde{g}_x^T \Delta \tilde{g}_x^{\text{(ori)}}
- \frac{ \eta g_z^{2} \;\Vert \tilde{g}_h \Vert ^{2} }{H_z}
\bigg(\frac{e^{\mathcal{A}}}{\Vert \hat\delta \Vert} + C \cdot (\frac{e^{\mathcal{A}}}{\Vert \hat\delta \Vert})^{2} \bigg)\\
&=\frac{e^{\mathcal{A}}}{\Vert\hat{\delta}\Vert}\left[ C\cdot \tilde{g}_x^T\Delta\tilde{g}_x^{\text{ori}} -
\frac{\eta g_z^2\Vert\tilde{g}_h\Vert^2}{H_z}\left(1 + C\cdot \frac{e^\mathcal{A}}{\Vert\hat{\delta}\Vert}
\right)
\right].
\end{aligned}\end{small}
\end{equation}

Thus, {\small$\tilde{g}_x^T\Delta\tilde{g}_x^{\text{(norm)}}$} is determined by the term {\small$\frac{e^{\mathcal{A}}}{\Vert\hat{\delta}\Vert}$}.
Since the attack is relatively strong, we have {\small$\Vert\hat{\delta}\Vert\approx \exp(\beta\Vert\tilde{g}_x\Vert^2g_z^2)/\Vert g_x\Vert$}. In this case, the term {\small$\frac{e^{\mathcal{A}}}{\Vert\hat{\delta}\Vert}$} can be  represented as
\begin{equation}
\begin{aligned}
\label{eqn:determine_term}
\frac{e^{\mathcal{A}}}{\Vert\hat{\delta}\Vert}&\approx\frac{\Vert g_x\Vert\exp(\mathcal{A})}{\exp(\beta\Vert\tilde{g}_x\Vert^2g_z^2)}\\
&=\Vert g_x\Vert\exp \bigg[\beta\Vert\tilde{g}_x\Vert^2(H_z-g_z^2)\bigg].
\end{aligned}
\end{equation}

Hence, we can consider the strength of the training effect {\small$\tilde{g}_x ^{T} \Delta \tilde{g}_x^{\text{(norm)}}$} is mainly determined by the term {\small$\exp(\mathcal{A})/\Vert \hat\delta \Vert \approx \Vert {g}_x \Vert \cdot \exp( \beta  \Vert \tilde{g}_x \Vert ^{2} (H_z-g_z^{2}))$}.
\end{proof}


\section{Experimental verification 2 of Theorem~\ref{theorem:adv_per_2}}
\label{sec:exp_verify_theorem2}

To verify the correctness of Theorem~\ref{theorem:adv_per_2}, we conducted experiments to generate adversarial perturbations on four types of ReLU networks, and examined whether the derived analytic solution well fit the real perturbation measured in practice.
Specifically, we calculated the metric {\small$\kappa = {\mathbb{E}_{x}[ \| \delta^{*} - \hat\delta \|}/{ \|\delta^{*}\|}]$} to  evaluate the error between the derived analytic solution $\hat\delta$ in Theorem~\ref{theorem:adv_per_2} and the real perturbation $\delta^{*}$ measured in experiments.
To this end, we learned four types of ReLU networks, including MLPs, CNNs, MLPs with skip connections (namely ResMLP), and CNNs with skip connections (namely ResCNN), on the MNIST dataset~\cite{lecun1998gradient} via adversarial training.
Here, we followed settings in~\cite{ren2022towards} to construct five different MLPs, which consisted of $1,2,3,4,5$ fully-connected (FC) layers, respectively.
Each FC layer contained  $200$ neurons.
We also built five different CNNs, which consisted of $1,2,3,4,5$ convolutional layers, respectively, with a FC layer on the top.
Each convolutional layer contained $32$ filters.
Additionally, we added a skip-connection to each block of a FC layer and a ReLU layer in the above MLPs to construct different ResMLPs.
We also added a skip connection to each block consisting of a convolutional layer and a ReLU layer in the above CNNs to build different ResCNNs.

To generate adversarial perturbations, we constructed four baseline attacks.
In the first baseline, we set the loss function to the MSE loss, and controlled the gating states of each ReLU layer in each step of the adversarial attack to be the same as those corresponding to the original input sample $x$.
In this way, this baseline attack ignored the higher term {\small$\hat{\mathcal{R}}_2(\alpha)$} in Theorem~\ref{theorem:adv_per_2}, and neglected changes of gating states in Assumption 1, thereby being termed~\textbf{attack w/o {\small$\hat{\mathcal{R}}_2(\alpha)$} w/o {\small$\Delta \Sigma$}}.
For the second baseline attack, we did not fix the gating states of each ReLU layer, thereby being termed~\textbf{attack w/o {\small$\hat{\mathcal{R}}_2(\alpha)$}}
For the third baseline attack, we controlled the gating states of ReLU layer, and set the loss function to the cross-entropy loss without ignoring the higher term {\small$\hat{\mathcal{R}}_2(\alpha)$}, thereby being named as~\textbf{attack w/o {\small$\Delta \Sigma$}}.
For the fourth baseline attack, we both set the loss function to the cross-entropy loss and did not fix the gating states, thereby being named as~\textbf{attack}.
Table~\ref{tab:adv_per} reports errors $\kappa$  computed in four different experimental settings, which successfully verified Theorem~\ref{theorem:adv_per_2}.

\begin{table}[h!]
\begin{center}
\caption{The error $\kappa$ between the derived analytic solution $\hat\delta$ in Theorem~\ref{theorem:adv_per_2} and the real perturbation based on different ReLU networks. The error $\kappa$ based on each network is small, which successfully verifies Theorem~\ref{theorem:adv_per_2}.}
\vspace{-10pt}
\label{tab:adv_per}
\resizebox{\linewidth}{!}{\
\begin{tabular}{c|ccccc|ccc}
\toprule
\makecell[c]{Attacking\\ methods}&\makecell[c]{1-layer \\MLP}& \makecell[c]{2-layer\\ MLP}&\makecell[c]{3-layer \\MLP}& \makecell[c]{4-layer\\ MLP}&\makecell[c]{5-layer\\ MLP}&\makecell[c]{3-layer \\ResMLP}& \makecell[c]{4-layer\\ ResMLP}&\makecell[c]{5-layer\\ ResMLP}\\
\midrule

Attack w/o {\small$\hat{\mathcal{R}}_2(\alpha)$} w/o {\small$\Delta \Sigma$}
& 4.8$\times 10^{-4}$ & 6.5$\times 10^{-5}$ &  8.0$\times 10^{-6}$&  1.8$\times 10^{-6}$ & 2.9$\times 10^{-7}$
& 6.9$\times 10^{-5}$ & 8.0$\times 10^{-5}$ &  7.4$\times 10^{-5}$  \\
\midrule
Attack w/o {\small$\hat{\mathcal{R}}_2(\alpha)$}
&4.8$\times 10^{-4}$ & 5.4$\times 10^{-2}$ &  4.7$\times 10^{-2}$&  4.1$\times 10^{-2}$ & 2.4$\times 10^{-2}$
& 6.2$\times 10^{-2}$ & 1.2$\times 10^{-1}$ &  8.9$\times 10^{-2}$ \\
\midrule

Attack w/o {\small$\Delta \Sigma$}
& 4.9$\times 10^{-4}$ & 1.4$\times 10^{-5}$ &  1.0$\times 10^{-6}$&  2.2$\times 10^{-7}$ & 3.6$\times 10^{-8}$
& 7.7$\times 10^{-6}$ & 2.4$\times 10^{-5}$ &  1.0$\times 10^{-5}$  \\
\midrule

Attack
& 4.9$\times 10^{-4}$ & 2.4$\times 10^{-2}$ &  2.6$\times 10^{-2}$&  2.7$\times 10^{-2}$ & 8.9$\times 10^{-3}$
& 3.5$\times 10^{-2}$ & 4.9$\times 10^{-2}$ &  5.5$\times 10^{-2}$   \\

\midrule
\midrule
\makecell[c]{Attacking\\ methods}&\makecell[c]{1-layer \\CNN}& \makecell[c]{2-layer\\ CNN}&\makecell[c]{3-layer \\CNN}& \makecell[c]{4-layer\\ CNN}&\makecell[c]{5-layer\\ CNN}&\makecell[c]{3-layer \\ResCNN}& \makecell[c]{4-layer\\ ResCNN}&\makecell[c]{5-layer\\ ResCNN}\\
\midrule

Attack w/o {\small$\hat{\mathcal{R}}_2(\alpha)$} w/o {\small$\Delta \Sigma$}
& 6.7$\times 10^{-8}$ & 2.4$\times 10^{-7}$ &  1.7$\times 10^{-8}$&  8.5$\times 10^{-9}$ & 2.3$\times 10^{-9}$
& 6.4$\times 10^{-7}$ & 7.9$\times 10^{-7}$ &  2.0$\times 10^{-6}$  \\
\midrule

Attack w/o {\small$\hat{\mathcal{R}}_2(\alpha)$}
&1.1$\times 10^{-2}$ & 6.4$\times 10^{-3}$ &  6.3$\times 10^{-3}$&  6.5$\times 10^{-3}$ & 1.9$\times 10^{-3}$
& 2.5$\times 10^{-2}$ & 2.7$\times 10^{-2}$ &  2.5$\times 10^{-2}$ \\
\midrule

Attack w/o {\small$\Delta \Sigma$}
& 8.4$\times 10^{-9}$ & 2.9$\times 10^{-8}$ &  2.1$\times 10^{-9}$&  3.7$\times 10^{-10}$ & 2.9$\times 10^{-10}$
& 7.9$\times 10^{-8}$ & 9.8$\times 10^{-8}$ &  2.4$\times 10^{-7}$  \\
\midrule

Attack
& 9.7$\times 10^{-3}$ & 5.1$\times 10^{-3}$ &  4.4$\times 10^{-3}$&  4.4$\times 10^{-3}$ & 1.6$\times 10^{-3}$
& 1.3$\times 10^{-2}$ & 2.3$\times 10^{-2}$ &  1.2$\times 10^{-2}$  \\

\bottomrule

\end{tabular}}
\vspace{-8pt}
\end{center}
\end{table}
\newpage

\section{Experimental verification 1 of Theorem~\ref{theorem2}}

To verify the correctness of Theorem~\ref{theorem2}, we conducted experiments to examine whether the derived training effect well fit the real effect, based on sixteen adversarially trained ReLU networks in Section~\ref{sec:exp_verify_theorem2}.
Specifically, we calculated the metric {\small$\kappa = {\mathbb{E}_{x}[ \| \phi^{*} - \hat\phi \|}/{ \|\phi^{*}\|}]$} to evaluate the fitting between the theoretical derivation {\small$\hat\phi$} computed using the right side of Eq.~\eqref{main_eqn:adv_train} and {\small$\phi^{*}=\tilde{g}_x ^{T} \Delta \tilde{g}_x$} measured in experiments,
where {\small$\phi^{*}$} was computed using real measurements of {\small$\tilde{g}_x, \eta, g_{W}^{\text{(adv)}}, g_{W}$}, and {\small$\tilde{g}_h$} on each ReLU network.

To generate adversarial perturbations, we set the loss function to the MSE loss.
Considering Theorem~\ref{theorem2} was based on the assumption of consistent gating states in Assumption 1, we measured an additional effect {\small$\phi'$} in experiments by manually forcing gating states of each ReLU layer in the process of generating adversarial perturbations to be the same as gating states for the input sample without being perturbed.
To this end, we calculated a new error {\small$\kappa' = {\mathbb{E}_{x}[ \| \phi' - \hat\phi \|}/{ \|\phi'\|}]$}.
Such a setting well fit Assumption 1.
Table~\ref{tab:adv_train} reports errors $\kappa$ and $\kappa’$ computed in two different experimental settings, which both verified the correctness of Theorem~\ref{theorem2}.
Particularly, the change of gating states was unpredictable, which brought significant instability in the computation of  {\small$\phi^{*}$} on a few adversarial examples,~\textit{e.g.,} causing dividing $0$.
Thus, we used $90\%$ samples corresponding to the smallest errors between {\small$\hat\phi$} and {\small$\phi^{*}$} to calculate the metric {\small$\kappa$}.
Experimental results show that the derived training effect {\small$\phi^{*}$} still well explained real effects on most adversarial examples.

\begin{table}[h!]
\begin{center}
\caption{Experimental verification of Theorem~\ref{theorem2} on different adversarially trained ReLU networks.
Both the error $\kappa$ and the error $\kappa'$ are small, which verifies Theorem~\ref{theorem2}.}
\vspace{-10pt}
\label{tab:adv_train}
\resizebox{\linewidth}{!}{\
\begin{tabular}{c|ccccc|ccc}
\toprule
&\makecell[c]{1-layer MLP}& \makecell[c]{2-layer MLP}&\makecell[c]{3-layer MLP}& \makecell[c]{4-layer MLP}&\makecell[c]{5-layer MLP}&\makecell[c]{3-layer ResMLP}& \makecell[c]{4-layer ResMLP}&\makecell[c]{5-layer ResMLP}\\
\midrule

$\kappa$ & 6.8 $\times 10^{-4}$ & 2.7 $\times 10^{-2}$ & 9.1 $\times 10^{-2}$  & 2.7 $\times 10^{-1}$ & 4.1 $\times 10^{-2}$  & 3.4 $\times 10^{-2}$ & 1.7 $\times 10^{-1}$ & 3.9 $\times 10^{-2}$ \\
$\kappa'$ & 6.8 $\times 10^{-4}$ & 6.9 $\times 10^{-5}$  &  7.9 $\times 10^{-6}$ & 1.7 $\times 10^{-6}$ & 2.7 $\times 10^{-7}$ & 7.5 $\times 10^{-5}$  & 8.4 $\times 10^{-5}$  & 7.9 $\times 10^{-5}$\\

\midrule

&\makecell[c]{1-layer CNN}& \makecell[c]{2-layer CNN}&\makecell[c]{3-layer CNN}& \makecell[c]{4-layer CNN}&\makecell[c]{5-layer CNN}&\makecell[c]{3-layer ResCNN}& \makecell[c]{4-layer ResCNN}&\makecell[c]{5-layer ResCNN}\\
\midrule

$\kappa$ &  9.5$\times 10^{-3}$ & 2.0$\times 10^{-2}$ & 2.6$\times 10^{-1}$  & 1.8$\times 10^{-1}$ &  4.8$\times 10^{-2}$ & 4.8$\times 10^{-2}$  &  2.6$\times 10^{-1}$ & 4.6$\times 10^{-2}$\\
$\kappa'$ & 6.5$\times 10^{-8}$ & 2.3$\times 10^{-7}$  &  1.7$\times 10^{-8}$ &  2.7$\times 10^{-9}$ &  2.3$\times 10^{-9}$ & 6.3$\times 10^{-7}$ & 7.8$\times 10^{-7}$ & 2.0$\times 10^{-6}$  \\
\bottomrule
\end{tabular}}
\vspace{-5pt}
\end{center}
\end{table}


\section{Detailed Discussions}

\textbullet\;
In \textit{Experimental verification 1 of Theorem~\ref{theorem:adv_per_2}}, we used SGD with learning rate $0.01$ for $50$ epochs, and set the batch size to $128$ to train
VGG-11~\cite{simonyan2014very}, AlexNet~\cite{krizhevsky2012imagenet}, and ResNet-18~\cite{he2016deep} on MNIST dataset, respectively.
Based on these networks, we crafted adversarial perturbations $\hat\delta$ in Theorem~\ref{theorem:adv_per_2} by the gradient {\small$g_{x+ \hat\delta^{(t)}} =\frac{\partial}{\partial x} L(f(x+ \hat\delta^{(t)}), y)$} for $500$ steps with the step size {\small$\alpha = \frac{1}{100} \epsilon = 0.02$}.
We further generated adversarial perturbations of the $\ell_2$ attack by {\small$g_{x+ \delta^{(t)}}^{(\ell_2)}=\frac{\partial}{\partial x} L(f(x+ \delta^{(t)}), y)/\Vert \frac{\partial}{\partial x} L(f(x+ \delta^{(t)}), y) \Vert$} for $200$ steps with the step size {\small$\alpha = \frac{1}{100} \epsilon = 0.02$}.
Besides, we also crafted adversarial perturbations of the $\ell_\infty$ attack by applying {\small$g_{x+ \delta^{(t)}}^{(\ell_\infty)}=\text{sign} (\frac{\partial}{\partial x} L(f(x+ \delta^{(t)}), y))$} for $20$ steps with the step size {\small$\alpha = \frac{1}{100} \epsilon = 0.02$}.
Note that the goal of this experiment was to verify whether the norm of the gradient {\small$\Vert g_{x+\hat\delta} \Vert$}, and the norm of the adversarial perturbation {\small$\Vert\hat\delta\Vert$} increased with the step number $m$ in an approximately exponential manner.
Hence, we ignored the constraint {\small$\Vert \hat\delta \Vert_{p} < \epsilon$} of adversarial perturbations, in order to prevent the analysis from being affected by the constraint {\small$\Vert \hat\delta \Vert_{p} < \epsilon$}.

\textbullet\;
In \textit{Experimental verification 2 of Theorem~\ref{theorem2}}, we trained VGG-11 and AlexNet on MNIST dataset against a PGD adversary with $20$ steps of the step size {\small$\frac{1}{10} \epsilon = 0.2$}.
We learned the above networks using SGD  with learning rate $0.01$ for $50$ epochs.
The adversarial perturbation $\hat\delta$ for evaluation was generated via the gradient {\small$g_{x+ \hat\delta^{(t)}} =\frac{\partial}{\partial x} L(f(x+ \hat\delta^{(t)}), y)$} for $100$ steps with the step size {\small$\alpha = \frac{1}{100} \epsilon = 0.02$}.
Here, we still neglected the constraint of adversarial perturbations.

\textbullet\;
In \textit{Experimental verification 3 of Theorem~\ref{theorem2}}, we trained VGG-11 and AlexNet on MNIST dataset against a PGD adversary with $20$ steps of the step size {\small$\frac{1}{10} \epsilon = 0.2$}.
We learned the above networks using SGD  with learning rate $0.01$ for $50$ epochs.
Adversarial perturbation $\hat\delta$ for evaluation were generated via the gradient {\small$g_{x+ \hat\delta^{(t)}} =\frac{\partial}{\partial x} L(f(x+ \hat\delta^{(t)}), y)$} for $100, 150$ and $200$ steps with the step size {\small$\alpha = \frac{1}{100} \epsilon = 0.02$}, respectively.
Here, we ignored the constraint of adversarial perturbations.

\textbullet\;
In \textit{Experimental verification of Theorem~\ref{theorem:oscillation}}, we used SGD with learning rate $0.01$ for $50$ epochs, and set the batch size to $128$ to train
VGG-11 and AlexNet~\cite{krizhevsky2012imagenet} on MNIST dataset, respectively.
Given an input sample $x$, we generated adversarial example $x+\delta$ of the $\ell_2$ attack by {\small$g_{x+ \delta^{(t)}}^{(\ell_2)}=\frac{\partial}{\partial x} L(f(x+ \delta^{(t)}), y)/\Vert \frac{\partial}{\partial x} L(f(x+ \delta^{(t)}), y) \Vert$} for $20$ steps with the step size {\small$\alpha = \frac{1}{10} \epsilon = 0.2$}.
To verify Theorem 6, we used the original input sample $x$ and the corresponding adversarial example {\small$x+\delta$} to update the weight {\small$W_{j}$} in the $j$-th layer by the same length {\small$\Vert \Delta W_{j}\Vert = \Vert \Delta W_{j}^{(\text{adv})} \Vert = 0.001$}.


\end{document}